%% file: sample.tex
\documentclass[twoside,11pt]{article}

\usepackage[preprint]{jmlr2e}
\usepackage{amsmath,bm}
\usepackage{accents}
\usepackage{algorithm}
\usepackage{algpseudocode}
\usepackage{xcolor}

\usepackage{todonotes}
\usepackage{color}

\newcommand{\commentout}[1]{}

\newcommand{\mk}[1]{{\color{blue}[MK: #1]}}

\jmlrheading{}{}{}{}{}{meila00a}{Patan\`e, Blaas, Laurenti, Cardelli, Roberts and Kwiatkowska}

\ShortHeadings{Adversarial Robustness of GPs}{Patan\`e, Blaas, Laurenti, Cardelli, Roberts and Kwiatkowska}
\firstpageno{1}

\input{sections/notation}

\begin{document}

\title{Adversarial Robustness Guarantees for Gaussian Processes}

\author{\name Andrea Patan\`e \email andrea.patane@cs.ox.ac.uk \\
       \addr Department of Computer Science\\
       University of Oxford
       \AND
       \name Arno Blaas \email arno@robots.ox.ac.uk \\
       \addr Department of Engineering Science\\
       University of Oxford
              \AND
       \name Luca Laurenti\footnote{Current address: Delft Center for System and Control (DCSC), TU Delft.} \email luca.laurenti@cs.ox.ac.uk \\
       \addr Department of Computer Science\\
       University of Oxford
               \AND
       \name Luca Cardelli \email luca.cardelli@cs.ox.ac.uk \\
       \addr Department of Computer Science\\
       University of Oxford
              \AND
       \name Stephen Roberts \email  sjrob@robots.ox.ac.uk \\
       \addr Department of Engineering Science\\
       University of Oxford
       \AND
       \name Marta Kwiatkowska \email marta.kwiatkowska@cs.ox.ac.uk   \\
       \addr Department of Computer Science\\
       University of Oxford
       
       }

\editor{ }

\maketitle

\begin{abstract}


Gaussian processes (GPs) enable principled computation of model uncertainty, making them attractive for safety-critical applications. 
Such scenarios demand that GP decisions are not only accurate, but also robust to perturbations.
In this paper we present a framework to analyse adversarial robustness of GPs, defined as invariance of the model's decision to bounded perturbations.
Given a compact subset of the input space $T\subseteq \mathbb{R}^d$, a point $x^*$ and a GP, we provide provable guarantees of adversarial robustness of the GP by computing lower and upper bounds on its prediction range in $T$.
We develop a branch-and-bound scheme to refine the bounds and show, for any $\epsilon > 0$, that our algorithm is guaranteed to 
converge to values $\epsilon$-close to the actual values in finitely many iterations.
The algorithm is anytime and can handle 
both regression and classification tasks, with analytical formulation for most kernels used in practice. 
We evaluate our methods on a collection of synthetic and standard benchmark datasets, including SPAM, MNIST and FashionMNIST. We study the effect of approximate inference techniques on robustness and demonstrate how our method can be used for interpretability. 
Our empirical results suggest that the adversarial robustness of GPs increases with accurate posterior estimation.
\end{abstract}


\section{Introduction}
    \input{sections/introduction}
\section{Bayesian Learning with Gaussian Processes}\label{sec:background}
\input{sections/background.tex}

\section{Bounding Posterior Mean and Variance Function}\label{sec:mean_and_var_opts}
\input{sections/constant_computation}

\section{Bounds on Adversarial Robustness}\label{sec:adv_rob_method}
    \input{sections/methods.tex}

\section{Branch-and-Bound Algorithm}\label{sec:bnb_alg}
    \input{sections/bnb.tex}
\section{Experimental Results} \label{sec:experiments_adversarial}
\input{sections/experiments.tex}

\section{Conclusion}

\input{sections/conclusion.tex}


\acks{SR is grateful for support from the UK Royal Academy of Engineering and the Oxford-Man Institute. AB thanks the Konrad-Adenauer-Stiftung and the Oxford-Man Institute for their support. MK, LL and AP received funding from the European Research Council (ERC)
under the European Union’s Horizon 2020 research and innovation programme
(FUN2MODEL, grant agreement No.~834115).
AP and MK acknowledge partial funding from the European
Union's Horizon 2020 research and innovation program under the Marie Sklodowska-Curie grant agreement No 722022 ``AffecTech''. LC is supported by a Royal Society Research Professorship.}


\newpage

\appendix
\input{sections/appendix}

\vskip 0.2in
\bibliography{sample.bib}

\end{document}

%% file: sections/notation.tex
\newtheorem{mydef}{Definition}
\newtheorem{myprobl}{Problem}

\newcommand{\GP}{\bm{f}}

\newcommand{\dataset}{\mathcal{D}}
\newcommand{\defnotation}{:=}
\newcommand{\inputspace}{\mathbb{R}^d}
\newcommand{\outputspace}{\mathbb{R}^m}

\newcommand{\x}[1]{x^{(#1)}}
\newcommand{\y}[1]{y^{(#1)}}

\newcommand{\tss}[0]{{N}}

\newcommand{\ali}[0]{{a_L^{(i)}}}
\newcommand{\bli}[0]{{b_L^{(i)}}}
\newcommand{\aui}[0]{{a_U^{(i)}}}
\newcommand{\bui}[0]{{b_U^{(i)}}}

\newcommand{\covvec}{\mathbf{r}}

\newcommand{\piInf}[1]{\pi_{\min}(#1)}
\newcommand{\piSup}[1]{\pi_{\max}(#1)}
\newcommand{\piInfL}[1]{\pi_{\min}^L(#1)}
\newcommand{\piInfU}[1]{\pi_{\min}^U(#1)}
\newcommand{\piSupL}[1]{\pi_{\max}^L(#1)}
\newcommand{\piSupU}[1]{\pi_{\max}^U(#1)}

\DeclareMathOperator\erf{erf}
\DeclareMathOperator*{\argmax}{arg\,max}
\DeclareMathOperator*{\argmin}{arg\,min}

\newcommand{\LL}[1]{{\color{red}[LL: #1]}}

\newcommand\munderbar[1]{%
  \underaccent{\bar}{#1}}

%% file: sections/introduction.tex
Adversarial examples are input points intentionally crafted to trick a machine learning model into a misclassification.
Imperceptible perturbations that can fool deep learning models in computer vision have been popularised by \cite{szegedy2013intriguing} and, in the context of security, account for the growth in adversarial machine learning techniques, see review in  \citep{biggio2018wild}.
Since test accuracy fails to account for the behaviour of a model in adversarial settings, algorithmic techniques for quantifying \emph{adversarial robustness} of machine learning models are needed to aid their deployment in safety-critical scenarios. As a consequence, a number of methods that provide exact or approximate guarantees on the model output have been developed for neural networks, e.g., \citep{huang2017safety,katz2017reluplex,zhang2018efficient}.


Gaussian process (GP) models~\citep{williams2006gaussian} provide a flexible probabilistic framework for performing inference over functions, which integrates information from prior and data into a predictive posterior distribution that informs the optimal model decision.  
GPs are particularly attractive in view of their favourable analytical properties and support for Bayesian inference.  
One advantage of GPs compared to
neural network models is that they support the computation of uncertainty over model predictions, which can then be propagated through the decision-making pipelines.  
Various notions of robustness have been investigated for Gaussian process models, such as robustness against outliers \citep{kim2008outlier} or against labelling errors \citep{hernandez2011robust}. However, to the best of our knowledge,  studies of adversarial robustness of GPs have been limited to statistical (i.e., input distribution dependent) \citep{abdelaziz2017data} and heuristic analyses \citep{grosse2018limitations,bradshaw2017adversarial}.

In this work, we develop a novel algorithmic framework to quantify the adversarial robustness of optimal predictions of Gaussian process models trained on a dataset $\mathcal{D}$. To this end, we adapt the notion of adversarial robustness commonly employed for neural networks models to the GP setting, 
defined as the invariance of the decision in a small neighbourhood of a test point~\citep{huang2017safety}, and thus study the worst-case effect of bounded perturbations of the input on the GP optimal decision. 
We represent bounded perturbations by a compact subset of the input space $T\subseteq \mathbb{R}^d$ enclosing a test point $x^* \in \inputspace$, 
and consider the prediction range of the GP over $T$.
Similarly to~\citep{ruan2018reachability}, we observe that, to provide provable guarantees on the model prediction over $T$, it suffices to compute the minimum and maximum of the reachable prediction range. 
Unfortunately,  exact direct computation of the minimum and maximum class probabilities over compact sets is not possible, as these would require providing an exact solution of a global non-linear optimisation problem, for which no general method exists \citep{neumaier2004complete}. 
Instead, we approximate each extremum of the prediction range by lower and upper bounds. 
We show how such upper and lower bounds for the minimum and maximum prediction probabilities of the GP can be computed on any given compact set $T$, and then iteratively refine these bounds in a branch-and-bound algorithmic optimisation scheme until convergence to the minimum and maximum is obtained.
The method we propose is anytime (the bounds provided are at every step an over-estimation of the actual classification ranges over $T$, and can thus be used to provide guarantees) and $\epsilon$-exact (the actual values are reached in finitely many steps up to an error $\epsilon$ selected a-priori).
Our framework can handle robustness for
both regression and classification tasks, with analytical formulation for most kernels used in practice, including generalised spectral kernels. 

We implement the methods {in Matlab}\footnote{The code can be found at  \url{https://github.com/andreapatane/check-GPclass}.} and apply our approach to analyse the robustness profile of GP models on 
a synthetic two-dimensional dataset, the SPAM dataset and feature-based analysis of both binary and $3$-class subsets of the MNIST and Fashion-MNIST (F-MNIST) datasets. 
In particular, we compare the guarantees computed by our method with the robustness estimation approximated by adversarial attack methods for GPs  \citep{grosse2018limitations}, discussing in which settings the latter fails. Then, we analyse the effect of approximate Bayesian inference techniques and hyper-parameter optimisation procedures on the GP model robustness.
Across the four datasets analysed here, we observe that approximations based on Expectation Propagation \citep{minka2001expectation} give more robust classification models than {approximations based on} Laplace approximation. We further find that 
GP robustness increases with the number of hyper-parameter training epochs, and that sparse GP model robustness generally increases with the number of training points (for a fixed number of inducing points).
Finally, we show how our framework can be used to perform global interpretability analysis of GP predictions, highlighting differences over LIME~\citep{ribeiro2016should} 

To the best of our knowledge, ours is the first comprehensive framework that provides methods to compute provable guarantees for the adversarial robustness of Gaussian process models. In summary, the paper presents the following contributions:
\begin{itemize}
    \item We design a flexible framework for the bounding of the posterior mean and variance of GPs in compact subsets of the input space.
    \item Using the mean and variance bounds, we develop methods to lower- and upper-bound the minimum and maximum of a GP output over compact sets for the adversarial analysis of GPs in both classification and regression settings.
    \item We incorporate the bounding procedures in a branch-and-bound algorithmic optimisation scheme, which we show converges for any specified error  $\epsilon >0$ in finitely many steps. 
    \item We empirically evaluate the robustness of a variety of GP models on four datasets, for different training regimes including sparse approximations, and demonstrate how our method can be used for global interpretability analysis.
\end{itemize}

A preliminary version of this work appeared in \citep{blaas2020adversarial}. This paper extends previous work in several aspects. In \citep{blaas2020adversarial} we provide analytical bounds only for GP classification using a probit link function and consider GPs with the squared exponential kernels. Here, we also derive analytical bounds in the case of logit link function, show that regression models can be analysed using a subset of the methods developed for classification, and extend our framework to a class of kernel functions that  satisfy certain smoothness conditions (see Section \ref{sec:mean_and_var_opts}).
Furthermore, we extend the experimental evaluation to show that our framework can be employed to analyse the robustness of sparse GP approximations and additionally consider the Fashion-MNIST dataset.


This paper is structured as follows. In Section \ref{sec:background} we introduce background on GP regression and classification. The definition of adversarial robustness of GP models and the problem statements we consider are given in Section \ref{sec:prob_form}. Computation of adversarial robustness of a GP requires lower- and upper-bounding of the variation of the GP mean and variance in a neighbourhood of a test point. These bounds are presented in Section \ref{sec:mean_and_var_opts}  and then employed in Section \ref{sec:adv_rob_method} to compute adversarial robustness for both (binary and multiclass) classification and regression. A branch-and-bound algorithm that incorporates the bounding methods is presented in Section  \ref{sec:bnb_alg}, where we also show that it is guaranteed to converge to the true adversarial robustness of a GP model. Finally, empirical results on multiple datasets are discussed in Section \ref{sec:experiments_adversarial}.

\subsection{Related Works}
\label{Related Works}
Following on from seminal work that drew attention to deep learning models being susceptible to adversarial attacks in computer vision \citep{szegedy2013intriguing} and security \citep{biggio2018wild}, a range of techniques have been proposed for the analysis of adversarial robustness of machine learning models. The developed techniques mainly focus on neural networks and the prevailing approach is to compute worst-case guarantees on the model prediction at a given test point \citep{huang2017safety,katz2017reluplex}. Various approaches have been considered to compute such robustness measures, including  constraint solving~\citep{huang2017safety,katz2017reluplex}, optimisation~\citep{ruan2018reachability,DBLP:journals/jmlr/BunelLTTKK20},  convex relaxation~\citep{zhang2018efficient}, and abstract interpretation~\citep{gehr2018ai2}.
Such methods have also been extended to Bayesian Neural Networks (BNNs) (i.e., neural networks with a prior distribution over their weights and biases) with both sampling-based~\citep{cardelli2019statistical,wicker2021bayesian} and numerical~\citep{wicker20a,berrada2021verifying}  solution methods. However, these techniques rely on the parametric nature of neural networks, and therefore cannot be directly applied to GPs.

While various notions of robustness have been studied for Gaussian process models, such as robustness against outliers \citep{kim2008outlier} or against labelling errors \citep{hernandez2011robust},  studies of adversarial robustness of GPs have been limited to heuristic analyses \citep{grosse2018limitations,bradshaw2017adversarial} and binary classification \citep{smith2019adversarial}. In particular, in \citet{smith2019adversarial},  the authors give guarantees for GPs in a binary classification setting under the $\ell_0$-norm and  only consider the mean of the distribution in the latent space without taking into account the uncertainty intrinsic in the GP framework. In contrast, our approach also considers multi-class classification and regression, takes into account the full posterior distribution, and allows for exact (up to $\epsilon>0$) computation under any $\ell_p$-norm. 

Formal probabilistic guarantees for learning with GPs have been developed in the context of GP optimisation \citep{bogunovic2018adversarially} and GP regression \citep{cardelli2018robustness}.  \citet{cardelli2018robustness}  derive an upper bound on the probability that a function sampled from a trained GP is invariant to bounded perturbations at a specific test point, 
whereas \citet{bogunovic2018adversarially} consider a GP optimisation algorithm, in which the returned solution is guaranteed to be robust to adversarial perturbations with a certain probability. 
We note that our problem formulation is different, and the methods developed in the above papers cannot 
be applied to classification with GPs due to its non-Gaussian nature. Further, our approach yields guarantees on the optimal model decision rather than on the latent GP posterior, is  guaranteed to converge to  any given error $\epsilon >0$ in finite time, and is anytime (i.e., at any time it gives sound  upper and lower bounds of the classification probabilities).



%% file: sections/background.tex
%
%
%
This section provides background material on Gaussian process modelling for regression and classification. More information can be found in \citep{williams2006gaussian}. 
An $\mathbb{R}^m$-valued Gaussian process over a real-valued vector space $\mathbb{R}^d$ is a particular stochastic process $\GP : \Omega \times \mathbb{R}^d \rightarrow \mathbb{R}^m$, where $
\Omega$ is a suitable sample space, such that for every finite subset of input points their joint distribution under the GP is Gaussian.
Namely, denoting with  $\GP(x) := \GP(\cdot,x) : \Omega \rightarrow \mathbb{R}^m$ the random variable induced by the stochastic process in the input point $x$, and given a collection of input points  $\mathbf{x} = [\x{1},\ldots,\x{\tss}]$, with $\x{i} \in \inputspace$, a GP is such that $\GP(\mathbf{x}) \sim \mathcal{N}(\mu(\mathbf{x}),\Sigma_{\mathbf{x},\mathbf{x}})$, where $\mu : \inputspace \rightarrow \outputspace$ is the mean function and $\Sigma : \inputspace \times \inputspace \rightarrow \mathbb{R}^{m^2}$ is the covariance (or \textit{kernel}) function, which fully characterise the behaviour of the GP. 

Consider now a dataset $\dataset = \{ (x^{(i)},y^{(i)}) \; | \; x^{(i)} \in \inputspace, \; y^{(i)} \in \mathcal{Y}, \; i=1,\ldots,\tss  \}$ for some input space $\inputspace$ and output space $\mathcal{Y}$.
We denote with $\mathbf{x} = [\x{1},\ldots,\x{\tss}]$ the aggregate vector of input points, and similarly $\mathbf{y}= [\y{1},\ldots,\y{\tss}]$ is the aggregate vector of output points. 
We let $\mathcal{Y}$ to be (a subset of) $\outputspace$ for regression, and the discrete set $\{1,\ldots,m \}$ in case of an $m$-class classification problem.
%
Gaussian processes provide a probabilistic framework for performing inference over functions, where a prior is combined with data through an appropriate likelihood to obtain a posterior process that is consistent with the prior and data. 
In a Bayesian framework this is done by introducing a latent space $\mathcal{F}  = \outputspace$, and defining a GP prior $\GP$ over the latter by instantiating a specific form for its mean, $\mu$, and kernel, $\Sigma$, functions. 
%
The prior is updated to take into account the information contained in the dataset $\mathcal{D}$ by means of the Bayes formula
   $ p(f(\mathbf{x}) | \dataset) \propto p(\mathbf{y} | f(\mathbf{x})) p(f(\mathbf{x}))$, 
where $p(\mathbf{y} | f(\mathbf{x}))$ denotes the likelihood function,
resulting in the \emph{posterior} distribution,  $p(f(\mathbf{x}) | \dataset)$, over the latent space $\mathcal{F}$. 
%
%
Given a previously unseen point $x^*$, the \emph{predictive posterior} distribution over its associated output  $y^*$ can be obtained by 
marginalising the posterior evaluated on $x^*$ over the latent space, i.e., 
    $p(y^* | \dataset) = \int_{\mathcal{F}}{  p(y^* | f(x^*))  p(f(x^*) | \dataset)  df(x^*)}$.
For practical applications, we typically extract a point value, $\hat{y}(x^*)$,
from the posterior predictive distribution $p(y^* | \dataset)$ that satisfies specific criteria.  
In Bayesian decision theory, one proceeds by assuming a loss function, $L(y^*,\hat{y}^*)$, and minimising it with respect to the posterior distribution on the specific test point, that is,
\begin{align*}
    \hat{y}(x^*) = \argmin_{y \in \mathcal{Y}} \int_{\mathcal{Y}}{  L(y^*,y)  p(y^*  | \dataset)   dy^*}.
\end{align*} 

Since $y$ is a continuous variable for regression models and a discrete variable for classification, different likelihood and loss functions are used in each case, resulting in different treatment for the posterior distribution and the model decision. 
    Below, we review the specific details separately. 


\paragraph{Regression}
For regression models we typically assume a Gaussian likelihood function with uncorrelated noise $\sigma_{\text{noise}}^2$, i.e., $p(y|f) = \mathcal{N}(y|f,\sigma_{\text{noise}}^2 I_\tss)$. 
The posterior distribution over $\mathcal{F}$ is still Gaussian, and is characterised by the following inference equations for its posterior mean and variance:
\begin{align}
    \bar{\mu}(x^*) &= \mu(x^*) + \Sigma_{x^*, \mathbf{x}} \mathbf{t}  \label{eq:mean_posterior} \\
    \bar{\Sigma}({x^*})  &:= \bar{\Sigma}_{x^*,x^*} = \Sigma_{x^*,x^*} - \Sigma_{x^*,\mathbf{x}} S \Sigma_{\mathbf{x},x^*}, \label{eq:var_posterior}
\end{align}
where $S$ and $\mathbf{t}$ are computed using the conditioning formula for Gaussian distributions \citep{williams2006gaussian}.
Namely, $S$ is a matrix in $\mathbb{R}^{\tss \times \tss}$ with $S = (\Sigma_{\mathbf{x},\mathbf{x}} + \sigma^2 I_\tss )^{-1}  $ and $\mathbf{t}$ is a vector in $\mathbb{R}^{\tss}$ with $\mathbf{t} =  S (\mathbf{y}-\mu(\mathbf{x})) $. 
Furthermore, the predictive posterior distribution over $\mathcal{Y}$ has the same mean as the posterior 
and variance equal to that of the posterior plus the underlying noise $\sigma_{\text{noise}}^2$.
Assuming a symmetric loss (e.g.\ the squared distance loss), which we refer to as the canonical loss for regression, 
the optimal \emph{model decision} is simply given by the posterior mean, i.e., $\hat{y}(x^*) = \mu(x^*)$. 

\paragraph{Classification}
For classification models, the likelihood is generally defined in terms of a sigmoid function $p(y = i|f) = \sigma_i(f)$, for $i \in \{1,\ldots,m \}$, as the \textit{probit} or \textit{softmax} function. 
Unfortunately, this does not result in a Gaussian posterior and is intractable.
Instead, analytical approximations are applied to estimate a Gaussian distribution of the form $q(f | \mathcal{D})  = \mathcal{N}( f \, | \, \bar{\mu}(x^*),\bar{\Sigma}({x^*})  ) $, which approximates the true distribution  $p(f | \mathcal{D})$. 
In this paper we consider $q$ derived using the \textit{Laplace} approximation method \citep{williams1998bayesian}, the \emph{Expectation Propagation} (EP) method \citep{minka2001expectation}, as well as several \emph{sparse approximation} techniques \citep{snelson2005sparse}; more details can be found in Section \ref{sec:experiments_adversarial}.
We observe that, in all these settings, the inference equations for $q(f | \mathcal{D})$ have the same 
form as those given in Equations \eqref{eq:mean_posterior} and \eqref{eq:var_posterior},
with $S$ and $\mathbf{t}$ defined depending on the method chosen  \citep{williams2006gaussian}.\footnote{We remark that this form of inference equations is common for Gaussian approximations, as it results from  conditioning formulas for multivariate Gaussian distribution. Our method can thus be applied in any situation in which Gaussian approximations are used (i.e., not necessarily resulting from Laplace or EP techniques).
}
Once the approximate posterior $q$ has been computed, the predictive posterior distribution for class $i \in \{1,\ldots,m \}$  is:
\begin{align}\label{eq:predictive_posterior_gen}
   \pi_i(x^*) := p( y^* = i | \dataset) = \int_{\mathcal{F}}{   \sigma_i(\xi)  \mathcal N(\xi | \bar{\mu}(x^*),\bar{\Sigma}{(x^*)})  d\xi}.
\end{align}
Since Equation \eqref{eq:predictive_posterior_gen} includes a non-linear multi-dimensional integral, its solution cannot in general be found in closed form. 
However, 
when there are two classes, i.e.\ $\mathcal{Y} = \{1,2 \}$, it suffices to compute $\pi_1$ and then simply set $\pi_2 = 1 - \pi_1$.
This allows us to simplify 
the latent variable space as uni-dimensional, so that $\xi \in \mathbb{R}$.
Assuming standard $0$-$1$ loss\footnote{Our method is sufficiently general to accommodate other loss functions, e.g., weighted loss, which can be computed from the predictive posterior.}, which we consider canonical for classification, the optimal \emph{model decision} is the class that maximises the predictive posterior distribution, that is, $ \hat{y}(x^*) = \argmax_{i=1,\ldots,m}  p(y^* = i | \dataset)$.

\section{Problem Formulation}\label{sec:prob_form}
Let $\GP$ be a Gaussian process trained on a dataset $\dataset$. We wish to analyse its \emph{adversarial robustness}, in the sense of studying the worst-case effect of bounded perturbations on the model's optimal decision $\hat{y}(x)$.
For a generic test point $x^*$, we represent the possible adversarial perturbations by defining a compact neighbourhood $T$ around $x^*$, 
and measure the changes in the decisions caused by limiting the perturbations to lie within $T$.
\begin{mydef}[Adversarial Robustness w.r.t.\ Model Decision]\label{def:safety_general}
Let $T \subseteq \inputspace$  be a compact subset and $x^* \in T$.
Consider a GP $\GP$, a loss function $L$ and the resulting optimal decision $\hat{y}(\cdot)$.
Given an $\ell_p$ norm $|| \cdot ||$, we say that $\GP$ is $\delta$-\emph{adversarially robust} in $T$ at a point $x^*$ with respect to the optimal decision induced by $L$ iff 
\begin{align}\label{eq:general_condition}
|| \hat{y}(x^*) - \hat{y}(x) || \leq \delta \qquad \forall x \in T. 
\end{align}
\end{mydef}
In the remainder of this paper, we will formulate a method for the worst-case analysis of the GP decision function $\hat{y}(x)$, which enables computing provable guarantees on whether a given $\GP$ is adversarially robust around a test point $x^*$ (that is, whether it satisfies the condition in Equation \eqref{eq:general_condition}). 
Since regression and classification differ in how optimal decisions are made, which is reflected in the definition of the function $\hat{y}(\cdot)$, to simplify the presentation we will discuss the two cases separately.
\subsection{Classification}\label{sec:class_def}
For classification problems, adversarial robustness is customarily defined in terms of \emph{invariance} of the decision over the neighbourhood of an input~\citep{huang2017safety,ruan2018reachability}. 
This can be obtained by selecting $\delta = 0$ in Definition \ref{def:safety_general}, and noting that for classification $\hat{y}(x) = \argmax_{i \in \{1,\ldots,m\}} \pi_i(x)$.
\begin{mydef}[Adversarial Robustness in Classification]\label{def:safety_classification}
Let $T \subseteq \inputspace$ be a compact subset and $x^* \in T$. Consider a classification GP $\GP$ and its predictive posterior distribution  $\pi_i(x), i \in \{1,\ldots,m\}$, defined as in Equation \eqref{eq:predictive_posterior_gen}.
Then we say that $\GP$ is \emph{adversarially robust} in $T$ at a point $x^*$  iff
\begin{align}\label{eq:condition}
\argmax_{i \in \{1,\ldots,m\}} \pi_i(x) = \argmax_{i \in \{1,\ldots,m\}} \pi_i(x^*) \qquad \forall x \in T. 
\end{align}
\end{mydef}
Adversarial robustness therefore provides guarantees that the classification decision is not influenced by adversarial perturbations applied to $x^*$, as long as the perturbations are constrained to remain within $T$.
Recall that the optimal decision for classification accounts for model uncertainty by moderating class probabilities with respect to the posterior distribution. 
%
In general, the outcome differs from the most likely class, because the decisions are affected by variance. 
%
%

Similarly to \citep{ruan2018reachability}, we note that,
in order to check the condition in Equation \eqref{eq:condition}, it suffices to compute the minimum and maximum of the \textit{prediction ranges} in $T$, i.e.:
\begin{align}\label{eq:min_max_pi}
    \pi_{\min,i}(T) = \min_{x \in T} \pi_i(x) \qquad \pi_{\max,i}(T) = \max_{x \in T} \pi_i(x) 
\end{align}
for $i = 1,\ldots,m$.
%
It is easy to see that the knowledge of $\pi_{\min,i}(T)$ and $\pi_{\max,i}(T)$ for all $i = 1,\ldots,m$ can be used to provide guarantees on the absence of \emph{adversarial attacks} of the model output, where an adversarial attack is a point $x \in T$ that is classified differently from $x^*$, that is, such that $\argmax_{i \in \{1,\ldots,m\}} \pi_i(x) \neq \argmax_{i \in \{1,\ldots,m\}} \pi_i(x^*)$.
More specifically, by letting $\hat{y}^* = \argmax_{i \in \{1,\ldots,m\}} \pi_i(x^*)$ and  defining the vector:
\begin{align}\label{eq:vec4ver}
    \pi^*_{i}(T) = \begin{cases} \pi_{\max,i}(T) \quad \text{if} \; i \neq \hat{y}^* \\
                                       \pi_{\min,i}(T) \quad \text{if} \; i = \hat{y}^*,
                                     \end{cases}
\end{align}
we can check whether the (stronger) condition $\argmax_{i \in \{1,\ldots,m\}} \pi^*_{i}(T) = \argmax_{i \in \{1,\ldots,m\}} \pi_i(x^*)$ holds.
That is, in order to decide whether a GP classification model $\GP$ satisfies Definition \ref{def:safety_classification} around a point $x^*$ we need to solve the following problem.
\begin{myprobl}[Computation of Adversarial Prediction Ranges]\label{prob:adv_pred}
Let $T \subseteq \inputspace$ be a compact subset. Consider a classification GP $\GP$ and its predictive posterior distribution $\pi_i(x), i \in \{1,\ldots,m\}$, defined as in Equation \eqref{eq:predictive_posterior_gen}. For $i = 1,\ldots,m$, compute the adversarial prediction ranges for  $\pi_i(x)$ in $T$, that is:
\begin{align*}
    \pi_{\min,i}(T) = \min_{x \in T} \pi_i(x) \qquad \pi_{\max,i}(T) = \max_{x \in T} \pi_i(x). 
\end{align*}
\end{myprobl}

Unfortunately, the solution of Problem \ref{prob:adv_pred} requires solving $2m$ non-linear optimisation problems, for which no general solution method exists 
\citep{neumaier2004complete}.
We discuss the bounding of Problem \ref{prob:adv_pred} in Sections \ref{sec:bin_bound_class} and \ref{sec:multiclass}, and then show how to refine the bounds in Section \ref{sec:bnb_alg}.\footnote{While we focus on adversarial robustness w.r.t.\ the $0$-$1$ loss, the computation of the prediction ranges poses a more general problem and the methods developed here can be used for classifiers associated to different loss functions (e.g., a weighted classification loss) through an appropriate definition of a vector in Equation \eqref{eq:vec4ver}.}

\subsection{Regression}\label{sec:regr_def}
For regression models, since the output is a continuous variable, we define adversarial robustness in terms of a small, bounded variation of the decision over a compact neighbourhood $T$ of a test point $x^*$.
This follows from Definition \ref{def:safety_general}, since for regression  $\hat{y}(x) = \bar{\mu}(x)$. 
Formally, we have the following.
\begin{mydef}[Adversarial Robustness in Regression]\label{def:adversarial_regr}
Let $T \subseteq \inputspace$ be a compact subset, $x^* \in T$ and consider a GP $\GP$.
We say that $\GP$ is \emph{adversarially $\delta$-robust} in  $T$ at a point $x^*$ with respect to $\ell_p$ norm $||\cdot||$ iff
\begin{align}
    || \bar{\mu}(x^*) - \bar{\mu}(x) ||  \leq \delta \qquad \forall x \in T,
\end{align}
where $\bar{\mu}(x) = \mathbb{E}[\GP(x)]$  is the posterior mean of the GP.
\end{mydef}
This definition is analogous to the computation of the reachable set of outputs (or confidence values) for neural networks~\citep{ruan2018reachability}. 
Since for a GP the mean corresponds to the maximum of the distribution, it thus follows, under the assumption of convergence, that it can be computed by a deterministic scheme that relies on regularised maximum likelihood estimation.
We remark that, in contrast to classification, adversarial robustness for GP regression does not take into consideration model variance, and analyses only the most likely model among those obtained by Bayesian inference.
As a consequence, the computation of adversarial robustness for regression reduces to the adversarial robustness of the posterior mean function.
More specifically, Definition \ref{def:adversarial_regr} can be checked once the value of  $\sup_{x \in T} || \bar{\mu}(x^*) - \bar{\mu}(x) ||$ is known. 
That is, in order to decide whether a GP regression model $\GP$ satisfies Definition \ref{def:adversarial_regr} around a point $x^*$ we need to solve the following problem.
\begin{myprobl}[Computation of Posterior Mean Ranges]\label{prob:mean_pred}
Let $T \subseteq \inputspace$ be a compact subset. Consider a regression GP $\GP$ and its posterior mean $\mu_i(x), i \in \{1,\ldots,m\}$, defined as in Equation \eqref{eq:mean_posterior}. For $i = 1,\ldots,m$, compute the minimum and maximum of the posterior mean  $\mu_i(x)$ in $T$, that is:
\begin{align*}
    \mu_{\min,i}(T) = \min_{x \in T} \mu_i(x) \qquad \mu_{\max,i}(T) = \max_{x \in T} \mu_i(x). 
\end{align*}
\end{myprobl}
As for Problem \ref{prob:adv_pred}, solving Problem \ref{prob:mean_pred} requires the solution of 2m non-linear optimisation problems.
Similarly to classification, for regression we will develop a bound for Problem \ref{prob:mean_pred} in Section \ref{sec:regression_bound} and refine it through a branch-and-bound technique in Section \ref{sec:bnb_alg}.

\subsection{Outline of the Approach}\label{sec:outline_of_approach}


We now give an outline of a computational scheme to solve Problems \ref{prob:adv_pred} and \ref{prob:mean_pred} introduced in Sections \ref{sec:class_def} and \ref{sec:regr_def}, respectively, which will be developed in detail in Section \ref{sec:adv_rob_method}.
We first discuss classification, and then show how the regression scenario can be obtained as a special case of classification.
\paragraph{Classification}
For Problem \ref{prob:adv_pred}, we devise a branch-and-bound optimisation scheme for the lower- and upper-bounding computation of the prediction ranges of a GP classification model over the input region $T$.
In particular, for $i=1,\ldots,m$, we first compute lower and upper bounds for  $ \pi_{\min,i}(T)$ and $ \pi_{\max,i}(T)$, that is, we compute a set of real values
 $\pi_{\min,i}^L(T)$, $\pi_{\min,i}^U(T)$, $\pi_{\max,i}^L(T)$ and $\pi_{\max,i}^U(T)$ such that:
\begin{align}\label{eq:lower_upper_condition1}
    \pi_{\min,i}^L(T) \leq & \pi_{\min,i}(T) \leq \pi_{\min,i}^U(T) \\ 
 \label{eq:lower_upper_condition2}  \pi_{\max,i}^L(T) \leq & \pi_{\max,i}(T)\leq \pi_{\max,i}^U(T).
\end{align}
We refer to $ \pi_{\min,i}^L(T)$ and $ \pi_{\max,i}^U(T)$ as \textit{over-approximations} of the ranges, as they provide pessimistic estimation of the actual values of $\pi_{\min,i}(T)$ and $\pi_{\max,i}(T)$ for the purpose of adversarial robustness, and hence tighter guarantees.
On the other hand, we refer to $\pi_{\min,i}^U(T)$ and $\pi_{\max,i}^L(T)$ as \textit{under-approximations}, because they provide an optimistic estimation of the actual values that we want to compute. 

The branch-and-bound scheme proceeds by iterative refinement of lower and upper bounds for the minimum and maximum of prediction ranges and is illustrated
in Figure~\ref{fig:branch_and_bound} for the simplified case of a GP with a single output value. 
First, 
we compute a lower- and an upper-bound function (the lower-bound function is depicted with a dashed red curve in Figure \ref{fig:branch_and_bound}) for the GP output (solid blue curve) in the region $T$. 
We then find the minimum of the lower-bound function, $\piInfL{T}$ (shown in the plot), and the maximum of the upper bound function, $\piSupU{T}$ (not shown).
Then, valid values for $\piInfU{T}$ and $\piSupL{T}$ can be computed by evaluating the GP predictive distribution on any point in $T$ (a specific $\piInfU{T}$ is depicted in Figure \ref{fig:branch_and_bound}). 
Next, the region $T$ is iteratively subdivided into sub-regions ($R_1$ and $R_2$ in the plot), for which we compute new (tighter) bounds by repeating the procedure previously applied to $T$.
This procedure repeats until the bounds converge up to a desired tolerance $\epsilon > 0$. 
For each iteration, the bounds computed are valid, and therefore our method is anytime and can be terminated after a fixed number of iterations, at a cost of precision.
\begin{figure*}
	\centering
	 \includegraphics[width = 0.99\textwidth]{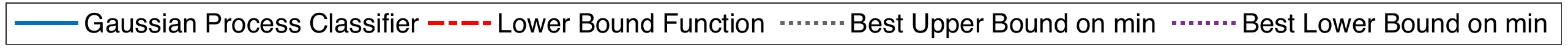} \\
	\includegraphics[width = 0.49\textwidth]{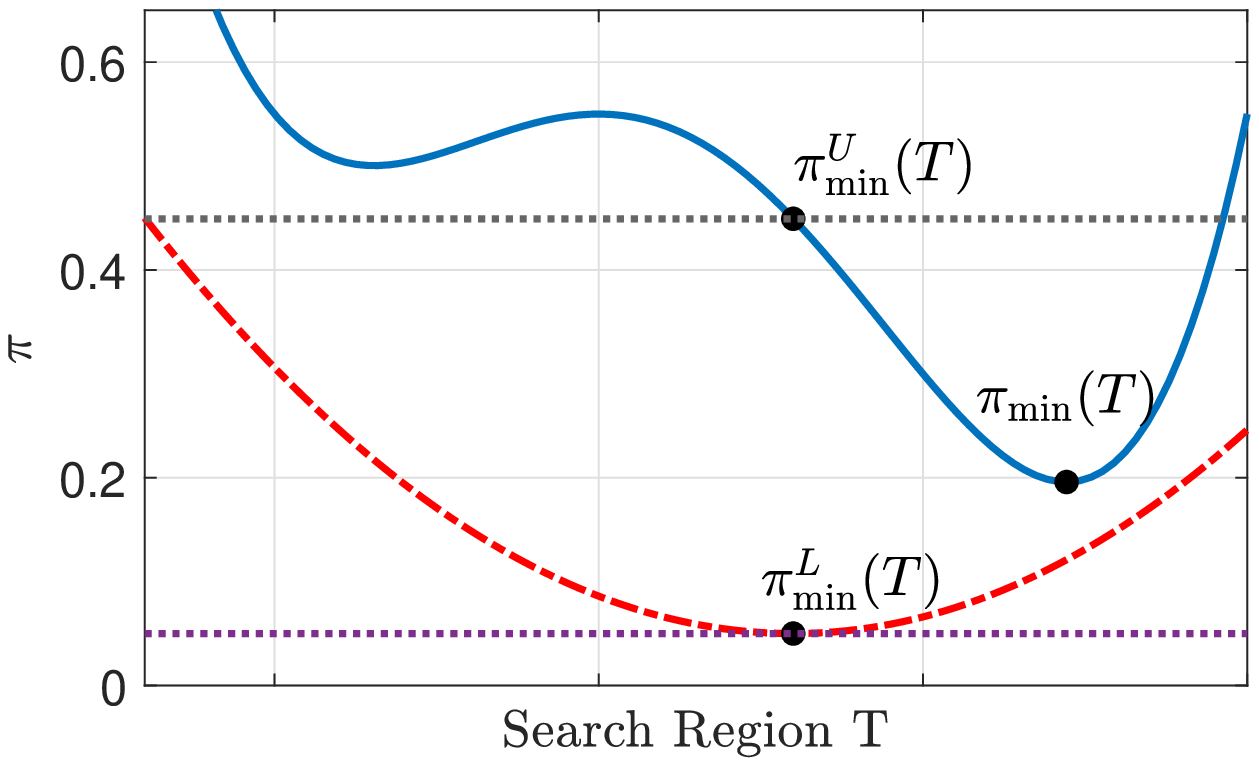}  
	\includegraphics[width = 0.49\textwidth]{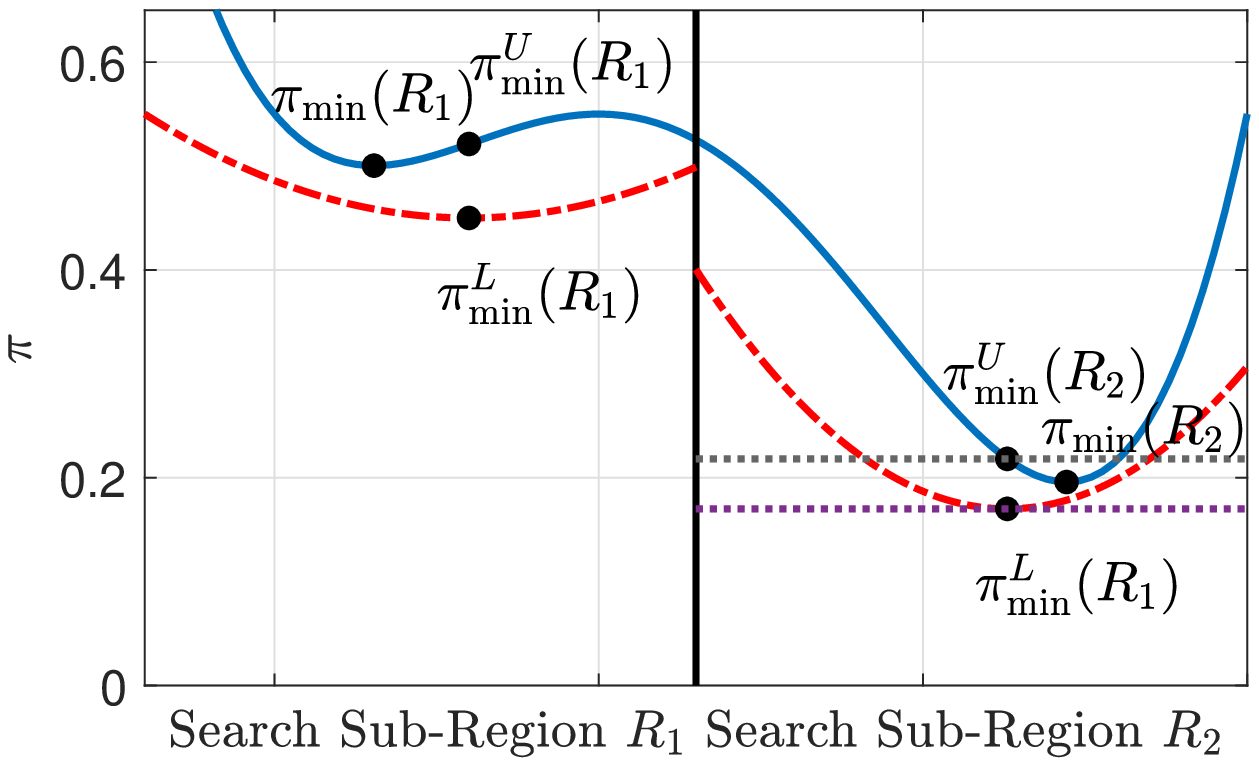}
	\caption{\textbf{Left:} Computation of upper and lower bounds on $\piInf{T}$, i.e., the minimum  of the classification range on   $T$. \textbf{Right:} The search region is repeatedly partitioned into sub-regions (only first partitioning is visualised), reducing the gap between best lower and upper bounds until convergence (up to $\epsilon$). }
	\label{fig:branch_and_bound}
\end{figure*}

The bounds on the predictive distribution depend analytically on the maximum variations of the posterior mean and variance over the region $T$, which we therefore need to compute beforehand.
For this purpose, in Section \ref{sec:mean_and_var_opts}, we develop an optimisation framework  for the computation of a set of real values $\mu^L_{T,i}$, $\mu^U_{T,i}$, $\Sigma^L_{T,i,j}$ and $\Sigma^U_{T,i,j}$  that under- and over-approximate the posterior mean and variance in $T$, i.e.: 
\commentout{Crucially, the computation of the bounds on the predictive distribution that we compute are dependent on upper and lower bounds on the a-posteriori mean and variance over the region $T$.
That is, in the next section, we first develop an optimisation-based framework  for the computation of a set of real values $\mu^L_{T,i}$, $\mu^U_{T,i}$, $\Sigma^L_{T,i,j}$ and $\Sigma^U_{T,i,j}$ such that:
}
\begin{align} \label{eq:mean_variance_bounds1}
\mu^L_{T,i} &\leq \min_{x \in T} \bar{\mu}_i(x) \quad \; \mu^U_{T,i}\geq \max_{x \in T}\bar{\mu}_i(x)\quad \\ 
\label{eq:mean_variance_bounds2} \Sigma^L_{T,i,j} &\leq \min_{x \in T} \bar{\Sigma}_{i,j}(x) \quad \Sigma^U_{T,i,j} \geq \max_{x \in T}\bar{\Sigma}_{i,j}(x)
\end{align}
for a general GP.
We will utilise this framework in Section \ref{sec:adv_rob_method} to compute the desired upper and lower bounds on the ranges of the  predictive posterior distribution.
\paragraph{Regression}
For regression, in Section \ref{sec:regression_bound} we develop a similar branch-and-bound approach to that for classification, except that (see Problem \ref{prob:mean_pred}) we only need to consider the mean of the predictive posterior distribution (discussed in the next section).

%% file: sections/constant_computation.tex
In Section \ref{sec:adv_rob_method} we will develop a method for the computation of  adversarial robustness guarantees for GPs.
This method utilises upper and lower bounds on the variation of the mean and variance in the compact region $T$.
Therefore, in this section, we formulate a general framework for the computation of lower and upper bounds on the posterior mean (Section \ref{sec:mean_opts}) and variance (Section \ref{sec:var_opts}) of a GP model.
Hence, we propose a method for the computation of $\mu^L_{T,i}$, $\mu^U_{T,i}$, $\Sigma^L_{T,i,j}$ and $\Sigma^U_{T,i,j}$ that satisfy Equations \ref{eq:mean_variance_bounds1} and \ref{eq:mean_variance_bounds2}, which will be used in Section \ref{sec:adv_rob_method}.

To simplify the presentation,  we consider a GP with a single output value, eliding the explicit dependence on $i$. 
\commentout{
Notice that the required quantities depend only on separate class indexes $i \in \{1,\ldots,m\}$.
Hence, for simplicity of presentation, in this section we consider a GP with a single output value, and we will omit the explicit dependence from $i$, with the understanding that this represents just one component of the actual posterior GP that we are interested in bounding. 
}
Since $T$ is compact and therefore bounded, it can be covered by
a finite union of hyper-boxes $T_l$, $l=1,\ldots,n_L$, i.e., 
    $T \subseteq \bigcup_{l=1}^{n_L} T_l$,
and furthermore the over-approximation error can be made vanishingly small.
The bounds can thus be computed for each of the boxes, $T_l$, and the minimum and maximum across $l=1,\ldots,n_L$ can be used as bounds for the infimum and supremum over the original set $T$.
Thus, without loss of generality, in the following we assume that $T$ is a box in the input space, i.e., $T= [x^L, x^U]$.


We proceed by restricting the setting to kernel functions that admit an upper-bounding function $U$, which we propagate through inference equations to obtain bounds on mean and variance. 
Our construction admits analytical bounds for a large class of kernels.
\begin{mydef}[Bounded Kernel Decomposition]\label{def:kernel_decomposition}
Consider a one-dimensional kernel function $\Sigma : \inputspace \times \inputspace \rightarrow \mathbb{R}$ and a compact set $T$.
We say that $\left( \varphi ,  \psi , U   \right)$ is a \emph{bounded decomposition} for $\Sigma$ in $T$ if $\Sigma_{x',x''} = \psi \left( \varphi \left( x',x'' \right) \right) $ and the following conditions are satisfied:
\begin{enumerate}
    \item $\varphi  : \inputspace \times \inputspace \rightarrow \mathbb{R} $  is linearly separable and continuously differentiable along each coordinate, so that $\varphi(x',x'') = \sum_{j=1}^{d} \varphi_j(x'_j,x''_j)$;
    \item $\psi : \mathbb{R} \rightarrow \mathbb{R}$ is continuously differentiable and with a finite number of flex points;
    \item $U$  is an upper bounding function such that, for any vector of coefficients $\mathbf{c} = [c_1,\ldots,c_N] \in \mathbb{R}^\tss$ and finite set of associated input points $[\x{1},\ldots,\x{\tss}]$, with $\tss \in \mathbb{N}$, we have that $U(\mathbf{c}) \geq \sup_{x \in T} \sum_{i=1}^{\tss} c_i \varphi(x,\x{i})$.
\end{enumerate}
\end{mydef}
Intuitively, a kernel decomposition separates the part of the kernel function that depends on the two inputs (represented by $\varphi$) with the part of the kernel that relates their dependence to the variance of the GP (represented by $\psi$).
Assumptions $1$ and $2$ usually follow immediately from the smoothness of kernel functions\footnote{The finite number of flex points can be guaranteed, for example, by inspecting the function derivatives.}.  
Assumption $3$ guarantees that we are able to upper bound the kernel function. 
The key idea is that, in view of the linearity of the inference equations for GPs, we can then propagate this bound through the inference equations  to obtain bounds on the posterior mean and variance of the GP.
We remark that, although not all kernel functions $\Sigma$  admit kernel decomposition (for example if they are not smooth), the majority of kernel functions used in practice do. 
In Appendix \ref{sec:decomposition}, we provide explicit computations for the \emph{squared exponential, the Matern, rational quadratic}, and the \emph{periodic} families of kernels, as well as \emph{sums and products} thereof. 
Further, we describe the computation of bounded kernel decompositions for \textit{generalised} (\textit{stationary} and \textit{non-stationary}) \textit{spectral kernels}, which by means of Bochner's theorem can be shown to define a dense subset of the set of all the possible covariance functions \citep{samo2015generalized}.
In the remainder of the paper we assume that we are dealing with a kernel that admits a bounded decomposition.

Before computing bounds on mean and variance, we state the following result (proved in Appendix \ref{sec:lemmas_and_proofs}) that ensures that the knowledge of a kernel decomposition allows us to compute a Lower Bounding Function (LBF) and an Upper Bounding Function (UBF) on the kernel values, linearly on $\varphi$.
\begin{lemma}\label{prop:lbf_ubf_kernel}
Let $\Sigma$ be a kernel and $\left( \varphi ,  \psi , U   \right)$ be a bounded decomposition. 
Let $T = [x^L,x^U] \subseteq \inputspace$ be a box in the input space, then for every $\bar{x} \in \inputspace$ there exists a set of real coefficients  $\bar{a}_L$, $\bar{b}_L$, $\bar{a}_U$ and $\bar{b}_U$ such that:
\begin{align*}
    g_L(x) \defnotation \bar{a}_L + \bar{b}_L \varphi \left( x, \bar{x} \right) \leq \Sigma_{x,\bar{x}} \leq \bar{a}_U + \bar{b}_U \varphi \left( x, \bar{x} \right) =: g_U(x) \quad \forall x \in T.
\end{align*}
In other words, $g_L$ and $g_U$ respectively represent  an LBF and a UBF for the kernel function, given a fixed input point.
\end{lemma}

The above proposition allows us to explicitly compute coefficients of an LBF and a UBF on the overall kernel value, for any fixed point $\bar{x}$ in the input space. 
The main idea is that, since the posterior mean and variance are defined in terms of the summation and multiplication of pieces of the form $\Sigma_{x,\x{i}}$, for all $\x{i}$ in the training dataset $\dataset$, we can compute LBFs and UBFs corresponding to each point in the training set, and propagate them through the inference equations for any unseen test point in $T$. 
By the design of the upper-bounding function $U$, we can then use the resulting LBFs and UBFs to bound the overall mean and variance functions. 
This is formalised in the following two subsections.
\subsection{Bounding the Posterior Mean}\label{sec:mean_opts}
Let $T \subseteq \inputspace$ be an axis aligned hyper-rectangle. In this section we show how to compute a lower bound $\mu^L_T$  for the posterior mean function in $T$, i.e.\ such that $\mu^L_T\leq \inf_{x \in T} \bar{\mu}(x)$, for a kernel $\Sigma$ with an associated bounded kernel decomposition $\left(\varphi,\psi,U\right)$.
Analogous techniques can be used to compute an upper bound $\mu^U_T$ by considering the function $-\bar{\mu}(x)$. 
We will then show that the bounds provided on the mean converge to the actual values as the diameter of the input region $T$ tends to $0$.

For simplicity, we assume that the prior mean function is identically null \citep{williams2006gaussian}.
Then, the posterior mean (Equation \eqref{eq:mean_posterior}) can be written down as:
\begin{align}\label{eq:posterior_with_null_mean}
    \bar{\mu}(x) = \Sigma_{x, \mathbf{x}} \mathbf{t} = \sum_{i=1}^\tss \Sigma_{x, \x{i}} t_i.
\end{align}
A lower bound for the mean function can thus be computed analytically, as shown in the following proposition.
\begin{proposition}\label{prop:mean_bound}
    Let $\Sigma$ be a kernel with bounded decomposition $(\varphi,\psi,U)$.
    Consider $\ali$, $\bli$, $\aui$ and $\bui$, the set of coefficients for LBFs and UBFs associated to each training point $\x{i}$, $i=1,\ldots,\tss$ in an axis-aligned hyper-rectangle $T \subseteq \inputspace$ (computed as for Lemma \ref{prop:lbf_ubf_kernel}).
    Define:
    \begin{align*}
    (\bar{a}^{(i)}_{L} , \bar{b}^{(i)}_{L}) = \begin{cases}
    (\ali , \bli), & \text{if $t_i \geq 0$}\\
    (\aui , \bui), & \text{otherwise}
  \end{cases}
    \end{align*}
    Then
    \begin{align*}
        \mu^L_T \defnotation \sum_{i=1}^\tss \bar{a}^{(i)}_{L} + U([\bar{b}^{(1)}_{L},\ldots,\bar{b}^{(\tss)}_{L}]) \leq \inf_{x \in T}\bar{\mu}(x).
    \end{align*}
\end{proposition}
\begin{proof}
By construction of the coefficients $\ali$, $\bli$, $\aui$ and $\bui$ we have that:
\begin{align*}
   {a}^{(i)}_{L} + {b}^{(i)}_{L} \varphi(x,\x{i})  \leq   \Sigma_{x,\x{i}} \leq  {a}^{(i)}_{U} + {b}^{(i)}_{U} \varphi(x,\x{i}).
\end{align*}
We can propagate the bounding functions through linear transformations (see Lemma \ref{lemmma:linear_prop} in Appendix \ref{sec:lemmas_and_proofs}), 
so that we obtain:
\begin{align}\label{eq:mean_bit}
    \Sigma_{x,\x{i}} t_i \geq  \bar{a}^{(i)}_{L} + \bar{b}^{(i)}_{L} \varphi(x,\x{i}) \quad \forall x \in T.
\end{align}
Hence, we have that $\sum_{i=1}^\tss \left( \bar{a}^{(i)}_{L} + \bar{b}^{(i)}_{L} \varphi(x,\x{i}) \right)$ is an LBF for the posterior mean. The statement of the theorem then follows directly from the definition of $U$.
\end{proof}


\subsubsection{Convergence of Mean Bounds}
We are able to show, importantly,
that the bounds provided for the mean converge uniformly to the actual mean function, when the input region $T$ is small enough.
We first state the following lemma that proves that the LBFs and UBFs given by Proposition  \ref{prop:lbf_ubf_kernel} yield converging bounds.
The proof is provided in Appendix \ref{sec:lemmas_and_proofs}.
\begin{lemma}\label{lemma:convergence_lbf}
Let $\Sigma$ be a kernel with bounded decomposition $(\varphi,\psi,U)$.
Let $T = [x^L,x^U] \subseteq \inputspace$, $\bar{x} \in T$,  and let, for every axis-aligned hyper-rectangle $R \subseteq T$, $g_L^R(x)$ and $g_U^R(x)$ be the LBF and UBF computed on $R$ for $\Sigma_{\bar{x},x}$ using Lemma \ref{prop:lbf_ubf_kernel}.
Then we have that $g_L^R$ and $g_U^R$ converge uniformly to $\Sigma_{\bar{x},x}$ as $\text{diam}(R)\to0$.
\end{lemma}

As the lower bound that we compute on the mean  over $T$ is obtained by summing together the individual LBFs $g_L^R$ computed over each training point $\x{i}$ on $R$, it then follows that convergence of all LBFs $g_L^R$ combined with a tight bounding function $U$ implies convergence of the posterior mean lower bound, and similarly for the upper bound.
This is formally shown in the proposition below.
\begin{proposition}\label{prop:mean_convergence}
Let $\Sigma$ be a kernel with bounded decomposition $(\varphi,\psi,U)$.
Then bounds for the posterior mean $\mu^{L}_{R}$ and $\mu^{U}_{R}$ computed through the application of Proposition \ref{prop:mean_bound} converge if the bounds provided by $U$ do so.
\end{proposition}
\begin{proof}
We discuss the case of $\mu^{L}_{R}$; the arguments are analogous for $\mu^{U}_{R}$.

We have that $\bar{\mu}(x) = \sum_{i = 1}^\tss \Sigma_{x,\x{i}} t_i$. By Proposition \ref{prop:mean_bound}, we obtain that:
\begin{align}\label{eq:temp_prop_stuff_to_prove}
    \sum_{i=1}^\tss t_i \bar{g}_L^{(i)}(x) \leq   \sum_{i = 1}^\tss \Sigma_{x,\x{i}} t_i 
\end{align}
where $\bar{g}_L^{(i)}(x) = g_L^{(i)}(x)$ if $t_i \geq 0$ and $\bar{g}_L^{(i)}(x) = g_U^{(i)}(x)$ otherwise. 
For Lemma \ref{lemma:convergence_lbf} we have that each $g_L^{(i)}$ converges uniformly to $\Sigma_{x,\x{i}}$ for each $\x{i}$.
As $t_i$ is a scalar quantity, we also have that each $ t_i \bar{g}_L^{(i)}(x)$ converges uniformly to $\Sigma_{x,\x{i}} t_i$.
Hence, we obtain that the bounds in Equation \eqref{eq:temp_prop_stuff_to_prove} converge uniformly as $\text{diam}(R) = r\to0$,  by virtue of being a linear combination of bounds that converge uniformly.
The statement of the proposition then follows by the definition of $U$.
\end{proof}
Therefore, convergence of the bounds for the posterior mean is reduced to a property of the kernel bounding function $U$.
%
In Appendix \ref{sec:decomposition} we show how explicit kernel decomposition can be computed for many kernel functions used in practice, where the derived functions $U$ converge to the actual desired values. 
%
%
%
\subsection{Bounding the Posterior Variance}\label{sec:var_opts}
We now show how to find a lower and an upper bound for the posterior variance from Equation \eqref{eq:var_posterior}.
For simplicity, we assume that $\Sigma_{x,x} = \sigma_p^2$ 
for all $x \in \inputspace$,\footnote{This is always the case for stationary kernels. In the general case $\Sigma_{x,x}$ can be replaced by either its maximum or minimum value depending on whether we want to compute the minimum or the maximum of the posterior variance.} so that we need only to compute:
\begin{align}\label{eq:min_var}
     \min_{x \in T}  \bar{\Sigma} ({x}) &= \sigma_p^2 + \min_{x \in T} - \Sigma_{x,\mathbf{x}} S   \Sigma_{x,\mathbf{x}}^T \\
     \label{eq:max_var} \max_{x \in T}  \bar{\Sigma}({x}) &=  \sigma_p^2 - \min_{x \in T} \Sigma_{x,\mathbf{x}} S   \Sigma_{x,\mathbf{x}}^T.
\end{align}
We first show how an upper bound for Equation \ref{eq:max_var} can be computed by means of convex quadratic programming.
\subsubsection{Variance Upper Bound}
The key observation is that $S$ 
given in Equation \eqref{eq:var_posterior} is a positive semi-definite matrix, so that the objective function to optimise in the case of the upper-bounding computation is a quadratic convex function on the variables $\Sigma_{x,\mathbf{x}}$ (but not on the optimisation variable $x$). 
In the following proposition, we show how the problem can be relaxed to obtain a quadratic convex program on the variable $x$ and  a suitably defined vector of slack variables.
\begin{proposition}\label{prop:variance_min}
    Let $\Sigma$ be a kernel with bounded decomposition $(\varphi,\psi,U)$ and $T = [x^L,x^U]$ a box of the input space $\inputspace$.
    Consider $\ali$, $\bli$, $\aui$ and $\bui$, a set of coefficients for LBFs and UBFs associated to each training point $\x{i}$, $i=1,\ldots,\tss$, computed according to Lemma \ref{prop:lbf_ubf_kernel}.
    Let $\mathbf{r} = [r^{(1)},\ldots,r^{(\tss)}]$, $\varphi^{(i)}$, $\varphi^{(i)}_j$, for $ i = 1,\ldots, \tss$ and $ j = 1,\ldots, d$, be slack continuous variables.
    Let $\bar{\sigma}^2$ be the solution of the following convex quadratic programming problem:
    \begin{align*}
        &\min_{x \in T}  \mathbf{r} S \mathbf{r}^T \\
        &\textrm{subject to: } \quad r^{(i)} + \ali + \bli \varphi^{(i)}  \leq 0 \quad i = 1,\ldots, \tss \\
        & \qquad \qquad \quad \; \; \, r^{(i)} - \aui - \bui \varphi^{(i)}  \leq 0 \quad i = 1,\ldots, \tss \\
        & \qquad \qquad \quad \; \; \, a_{j,L}^{(i)} + b_{j,L}^{(i)} x_j - \varphi^{(i)}_j \leq 0  \quad i = 1,\ldots, \tss \quad j = 1,\ldots, d \\
        & \qquad \qquad \quad \; \; \, \varphi^{(i)}_j -a_{j,U}^{(i)} - b_{j,U}^{(i)} x_j  \leq 0 \quad i = 1,\ldots, \tss \quad j = 1,\ldots, d \\
        & \qquad \qquad \quad \; \; \, \varphi^{(i)} = \sum_{j=1}^{d} \varphi^{(i)}_j  \qquad  \qquad \; \; \; i = 1,\ldots, \tss \quad j = 1,\ldots, d.
    \end{align*}
    Then $\Sigma^U_T \defnotation \sigma_p^2 - \bar{\sigma}^2$ is an upper bound for the posterior variance $\bar{\Sigma} ({x}) $ in $T$. 
\end{proposition}
\begin{proof}
By setting $\mathbf{r} = \Sigma_{x,\mathbf{x}}$ in the minimum computation in Equation \eqref{eq:max_var}, we obtain the objective function of the problem statement, $\mathbf{r} S \mathbf{r}^T$, which is quadratic on the vector variable $\mathbf{r}$.
Since $S$ is symmetric and positive semi-definite it follows that the objective function is a quadratic convex function in the slack variable vector $\mathbf{r}$.
In order to obtain a convex program we then need to linearise the constraint $\mathbf{r} = \Sigma_{x,\mathbf{x}}$
We show how this is done for a generic index $i = 1,\ldots,\tss$.

We have that $r^{(i)} = \Sigma_{x,\x{i}} = \psi(\varphi(x,\x{i}))$.
By Lemma \ref{prop:lbf_ubf_kernel} we obtain that:
\begin{align*}
    \ali + \bli \varphi \left( x, \x{i} \right) \leq \Sigma_{x,\x{i}} \leq \aui + \bui \varphi \left( x, \x{i} \right).
\end{align*}
Hence, the dependence of $\psi$ on the constraints can be linearised by considering the following over-approximation for the definition of $r^{(i)}$:
\begin{align*}
     &r^{(i)} + \ali + \bli \varphi \left( x, \x{i} \right) \leq 0 \\
      &r^{(i)} - \aui - \bui \varphi \left( x, \x{i} \right)  \leq 0.
\end{align*}
The final step is to linearise the dependency over  $\varphi \left( x, \x{i} \right)$.
We introduce slack variables $\varphi^{(i)} = \varphi(x,\x{i})$, and $\varphi^{(i)}_j = \varphi_j(x_j,\x{i}_j)$.
For Assumption $1$ of Definition \ref{def:kernel_decomposition}  we have that $\varphi(x,\x{i}) = \sum_{j=1}^{d} \varphi_j(x_j,\x{i}_j)$.
Let $i \in \{1,\ldots,\tss\}$ and let $j \in \{1,\ldots, d \}$, then by applying Lemma \ref{prop:lbf_ubf_kernel} with $\psi \defnotation \varphi_j(\cdot,\x{i}_j)$ and $\varphi \defnotation x$, we obtain that there exists a set of coefficients $  a^{(i)}_{j,L}$, $b^{(i)}_{j,L}$, $  a^{(i)}_{j,U}$ and $b^{(i)}_{j,U}$ such that:
\begin{align*}
    a^{(i)}_{j,L} + b^{(i)}_{j,L} x_j \leq \varphi_j(x_j,\x{i}_j) \leq a^{(i)}_{j,U} + b^{(i)}_{j,U} x_j.
\end{align*}
Hence, we can over-approximate the set of constraints $\varphi^{(i)} = \varphi(x,\x{i})$ and $\varphi^{(i)}_j = \varphi(x_j,\x{i}_j)$ with the following set of linear constraints:
\begin{align*}
    a_{j,L}^{(i)} + b_{j,L}^{(i)} x_j - \varphi^{(i)}_j \leq 0 \\
    \varphi^{(i)}_j -a_{j,U}^{(i)} - b_{j,U}^{(i)} x_j \leq 0 \\
    \varphi^{(i)} = \sum_{j=1}^{d} \varphi^{(i)}_j.
\end{align*}
The formula for $\Sigma^U_T$ then follows by the definition of minimum and by Equation \eqref{eq:max_var}.
\end{proof}
Crucially, the proposition above casts the computation of the quantity $\Sigma^U_T$ as the solution of a convex quadratic programming problem, for which ready-made solver software exists \citep{rosen1986global}.
\subsubsection{Variance Lower Bound}
The situation is, unfortunately, more complicated for the lower-bounding computation of $\min_{x \in T} - \Sigma_{x,\mathbf{x}} S   \Sigma_{x,\mathbf{x}}^T$. 
In fact, though we can write down an optimisation problem akin to that of Proposition \ref{prop:variance_min}, since $S$ is positive definite we have that $-S$ is negative definite, which implies that the function we want to optimise is quadratic concave.
Thus, a number of local minima may exist, and simple quadratic optimisation is not guaranteed to yield the global solution.
However, as we are interested in worst-case scenario analysis, we need to compute the global minimum. 
Unfortunately, this is an NP-hard problem, whose exact solution would be impractical to compute.

Instead, we apply the methods proposed in \citep{rosen1986global} and proceed by computing a safe lower bound to the global minimum, that is, we want to compute a lower bound to the solution of: 
    \begin{align}
        &\min_{x \in T}  - \mathbf{r} S \mathbf{r}^T \label{eq:initial_vairance_opt_problem}\\
        &\textrm{subject to: } \quad r^{(i)} + \ali + \bli \varphi^{(i)}  \leq 0 \quad i = 1,\ldots, \tss \nonumber\\
        & \qquad \qquad \quad \; \; \, r^{(i)} - \aui - \bui \varphi^{(i)}  \leq 0 \quad i = 1,\ldots, \tss \nonumber \\
        & \qquad \qquad \quad \; \; \, a_{j,L}^{(i)} + b_{j,L}^{(i)} x_j - \varphi^{(i)}_j \leq 0  \quad i = 1,\ldots, \tss \quad j = 1,\ldots, d \nonumber\\
        & \qquad \qquad \quad \; \; \, \varphi^{(i)}_j -a_{j,U}^{(i)} - b_{j,U}^{(i)} x_j  \leq 0 \quad i = 1,\ldots, \tss \quad j = 1,\ldots, d \nonumber\\
        & \qquad \qquad \quad \; \; \, \varphi^{(i)} = \sum_{j=1}^{d} \varphi^{(i)}_j  \qquad  \qquad \; \; \; i = 1,\ldots, \tss \quad j = 1,\ldots, d \nonumber
    \end{align}
We highlight the details of the procedure applied to our specific setting below.
First, we start by re-writing the constraints of the optimisation problem above in matrix form.
Next, we introduce the aggregate variable vector $\mathbf{z} = [x_1,\ldots,x_d,\varphi^{(1)},\ldots,\varphi^{(\tss)},\varphi^{(1)}_1,\ldots,\varphi^{(\tss)}_d]$. 
Since the constraints are linear, it is possible to define two matrices $A_r$ and $A_z$ such that the optimisation problem above can be equivalently written down as:
\begin{align}\label{prob:concave}
    \min - \covvec^T S \covvec \\
    \textrm{Subject to:} \quad &A_r \covvec + A_z \mathbf{z} \leq b \nonumber \\
    &\covvec^L \leq \covvec \leq \covvec^U  \nonumber  \\
    & \mathbf{z}^L \leq \mathbf{z} \leq \mathbf{z}^U   \nonumber  
\end{align}
for suitably defined vectors $b$, $\covvec^L$, $\covvec^U$, $\mathbf{z}^L$, $\mathbf{z}^U$.
Now, as $S$ is symmetric and positive definite, there exists a matrix of  eigenvectors $U = [ \mathbf{u}^{(1)}, \ldots ,\mathbf{u}^{(\tss)} ] $ and a diagonal matrix  $\Lambda $ of the associated eigenvalues $\lambda^{(i)}$, for $i = 1,\ldots,\tss$, such that $S = U \Lambda U^T$.
We then define $\hat{r}^{(i)} =  \mathbf{u}^{(i)} \cdot \covvec$ for $i = 1,\ldots,\tss$, the rotated variables, and $\hat{\covvec}$  the aggregated vector of rotated variables, and compute their ranges $[\hat{r}^{(i),L},\hat{r}^{(i),U}]$ by solving the following $2\tss$ linear programming problems:
\begin{align*}
    \min / \max \quad  & \mathbf{u}^{(i)} \cdot \covvec \\
    \textrm{Subject to:} \quad \quad &A_r \covvec + A_z \mathbf{z} \leq b \\
    &\covvec^L \leq \covvec \leq \covvec^U  \\
    & \mathbf{z}^L \leq \mathbf{z} \leq \mathbf{z}^U.   
\end{align*}
Implementing the change of variables into the optimisation problem defined in Equation \eqref{prob:concave}, we obtain:
\begin{align*}
    \min - \hat{\mathbf{r}}^T \Lambda \hat{\mathbf{r}} \\
    \textrm{Subject to:} \quad &A_{\hat{r}} \hat{\mathbf{r}} + A_z \mathbf{z} \leq b  \\
    &\hat{\covvec}^L \leq \hat{\covvec} \leq \hat{\covvec}^U   \\
    & \mathbf{z}^L \leq \mathbf{z} \leq \mathbf{z}^U  
\end{align*}
where we have set $A_{\hat{r}}  = A_{r}  U $. 
We then notice that $\hat{\covvec}^T \Lambda \hat{\covvec} =  \sum_{i=1}^\tss \lambda^{(i)} \hat{r}^{(i) 2}$.
Each summand is a simple one-dimensional quadratic function, for which we can find a linear LBF by relying on Lemma \ref{prop:lbf_ubf_kernel}.  
Let $\alpha^{(i)}$ and $\beta^{(i)}$ be coefficients of such LBFs, then we have that $\alpha^{(i)} + \beta^{(i)} \hat{r}^{(i)} \leq - \lambda^{(i)} \hat{r}^{(i),2}$ for all $i = 1,\ldots,\tss$.
Let $\bm{\beta} = [\beta^{(1)},\ldots,\beta^{(\tss)}]$ and $\hat{\alpha} = \sum_{i = 1}^\tss \alpha^{(i)}$, then we can  lower-bound the optimisation problem defined in Equation \eqref{prob:concave} with the following linear programming problem:
\begin{align}
    \min \left( \hat{\alpha} + \bm{\beta}^T \hat{\covvec} \right) \label{eq:lp_approx} \\
    \textrm{Subject to:} \quad & A_{\hat{r}} \hat{\covvec} + A_z \mathbf{z} \leq b  \nonumber \\
    &\hat{\covvec}^L \leq \hat{\covvec} \leq \hat{\covvec}^U   \nonumber \\
    & \mathbf{z}^L \leq \mathbf{z} \leq \mathbf{z}^U.   \nonumber
\end{align}
Hence, we have that  a solution of the latter problem yields a lower bound for the solution of the optimisation problem in Equation \eqref{eq:initial_vairance_opt_problem}. That is, we have proved the following statement.
\begin{proposition}\label{prop:variance_max}
    Let $\munderbar{\sigma}^2$ be the solution of the linear programming problem defined in Equation \eqref{eq:lp_approx}. 
     Then $\Sigma^L_T \defnotation \sigma_p^2 + \munderbar{\sigma}^2$ is a lower bound for the posterior variance $\bar{\Sigma} ({x}) $ in $T$. 
\end{proposition}
\subsubsection{Convergence of Variance Bounds}
The convergence of the bounds computed for the variance to the actual values in hyper-rectangles $R \subseteq T$, with $\text{diam}(R)\to0$, is an immediate consequence of Lemma \ref{lemma:convergence_lbf}, and proceeds similarly to what we have shown for the posterior mean.
In fact, the objective function for the upper bound (Proposition \ref{prop:variance_min}) is exact, and the over-approximation results only from the feasible region of the optimisation problem.
This is relaxed by using LBFs and UBFs introduced in Lemma \ref{prop:lbf_ubf_kernel}, so that their uniform convergence implies that the over-approximated feasible region converges to the actual one in the limit of the diameter $\text{diam}(R)$ tending to $0$.
Similarly, for the lower-bounding of the variance the only difference arises from the use of Lemma \ref{prop:lbf_ubf_kernel}, also for the lower-bounding of the optimisation function. 
However, this also converges to the actual objective function.
Thus, the exact solution of both optimisation problems converges uniformly to the actual values, for $R$ small enough.
We summarise the discussion as the following proposition.
The proof is a straightforward generalisation of the proof of Proposition \ref{prop:mean_convergence} and is therefore omitted.
\begin{proposition}\label{prop:variance_convergence}
Let $\Sigma$ be a kernel with bounded decomposition $(\varphi,\psi,U)$.
Then bounds on the posterior variance, $\Sigma^{L}_{R}$ and $\Sigma^{U}_{R}$, computed through the application of Propositions \ref{prop:variance_min} and \ref{prop:variance_max} converge if the bounds provided by $U$ do so.
\end{proposition}


%% file: sections/methods.tex
%
%
%
In this section we show how the lower and upper bounds for the posterior mean and variance can be propagated through the predictive distribution of a GP to compute adversarial robustness guarantees, in the sense of ensuring invariance of the GP decision to perturbations constrained to a small neighbourhood around a test point. 
Thus developed bound will then be included in a branch-and-bound scheme in Section \ref{sec:bnb_alg} for its iterative refinement.
 Recall from Section \ref{sec:class_def} and Problem \ref{prob:adv_pred} that for classification this reduces to bounding the minimum and maximum of the prediction ranges over the neighbourhood. 
We first discuss the bound for the two-class classification problem, 
and then show how the two-class bound can be extended to the multi-class setting.
Finally, we discuss how to obtain guarantees for regression (Problem \ref{prob:mean_pred}) as a particular case of the techniques derived for classification.
\subsection{Bounds for Two-class Classification}\label{sec:bin_bound_class}
As discussed in Section \ref{sec:background}, for a two-class GP it suffices to consider a one-dimensional output space, which greatly simplifies the computations.
Namely, we have that the predictive posterior distribution of Equation \eqref{eq:predictive_posterior_gen} evaluated on a generic point $x$ can be simplified to  one-dimensional integral, i.e.:
\begin{align}\label{eq:post_pred_dist_analytic}
    \pi(x) =  \int_\mathbb{R} \sigma(\xi) \mathcal{N}( \xi | \bar{\mu}(x),\bar{\Sigma}(x))d\xi 
\end{align}
where $\bar{\mu}$ and $\bar{\Sigma}$ are the posterior mean and variance functions, respectively, and $\sigma(\cdot)$ denotes the likelihood function.
We give analytical bounds for the case where the likelihood function is either the probit function or the logistic sigmoid, which entail the majority of applications for GP classification \citep{williams2006gaussian}.
A general bound based on latent space discretisation is discussed for the multi-class problem in Section \ref{sec:multiclass}, and can also be used for a generic two-class likelihood function. 

Let $\mu^L_{T}$, $\mu^U_{T}$, $\Sigma^L_{T}$ and $\Sigma^U_{T}$ be lower and upper bounds for the posterior mean and variance of the GP, computed according to the methods discussed Section \ref{sec:mean_and_var_opts}.
We consider the function that describes the dependence of the predictive posterior distribution directly on the mean and variance by dropping their dependence on $x$: 
\begin{align}\label{eq:Pi_def}
    \Pi(\mu,\Sigma) =  \int_\mathbb{R} \sigma(\xi) \mathcal{N}( \xi | \mu, \Sigma)d\xi \quad \text{for} \quad  \mu \in [\mu^L_{T} , \mu^U_{T}], \Sigma \in [\Sigma^L_{T} , \Sigma^U_{T}].
\end{align}
Then, by definition of lower and upper bounds we have that:
\begin{align*}
    \min_{\substack{  \mu \in [\mu^L_T,\mu^U_T] \\ \Sigma \in [{\Sigma}^L_T ,{\Sigma}^U_T]  }}  \Pi(\mu,\Sigma) \leq \min_{x \in T} \pi(x) \quad \text{and} \quad \max_{\substack{  \mu \in [\mu^L_T,\mu^U_T] \\ \Sigma \in [{\Sigma}^L_T ,{\Sigma}^U_T]  }}  \Pi(\mu,\Sigma) \geq \max_{x \in T} \pi(x),
\end{align*}
that is, over-approximations of the prediction ranges can be found by optimising the function $\Pi$ over the mean/variance box domain $[\mu^L_T,\mu^U_T] \times  [\Sigma^L_{T} , \Sigma^U_{T}]$.
In the next two subsections we show how this can be done depending on the particular form of the chosen  likelihood $\sigma$.
%
%
%
\subsubsection{Classification with the Probit Likelihood}
\label{app:probit}
\input{sections/probit.tex}

%
%
%
\subsubsection{Classification via Logistic Likelihood}
We now consider the case where $\sigma$ is defined as the logistic sigmoid.
We will show that the minimum and maximum are to be found in the same extrema as for the probit likelihood.
However, as the predictive distribution cannot be expressed in closed form \citep{williams2006gaussian}, we first show that the derivative of the predictive distribution can be computed by passing the sign of the derivative under the integral sign.

First, we note that upper and lower bounds on the variance $\Sigma$ naturally induce upper and lower bounds on the standard deviation $s = \sqrt{\Sigma}$, which we denote $s^L_T$ and $s^U_T$. 
By substituting $s$ in the definition of $\Pi$ in Equation \eqref{eq:Pi_def}, which yields $\Phi(\mu,s) := \Pi(\mu,s^2) = \Pi(\mu,\Sigma)$,  and changing the integration variable to $t = (\xi - \mu)/s $, we have:
\begin{align*}
\Pi(\mu,\Sigma) =:     \Phi(\mu,s) = \int_{\mathbb{R}} h(t,\mu,s) dt \quad \text{where} \quad  h(t,\mu,s) = \sigma(s t+\mu) \mathcal{N}(t|0,1).
\end{align*}
We now want to compute $\frac{\partial \Phi}{\partial \mu} $ and $\frac{\partial \Phi}{\partial s} $.
It is easy to show that all the conditions to apply differentiation under the integral sign theorem are satisfied.
Thus, we have:
\begin{align*}
    \frac{\partial \Phi (\mu,s)}{\partial \mu} = \int \sigma'(s t + \mu) \mathcal{N}(t|0,1)dt, \quad
     \frac{\partial \Phi (\mu,s)}{\partial s} = \int t \sigma'(s t + \mu) \mathcal{N}(t|0,1)dt.
\end{align*}
By relying on the derivatives, we can establish the following bounds.
Specifically, the following proposition provides a solution for Problem \ref{prob:adv_pred} for two-class classification with the logistic likelihood.
\begin{proposition}\label{prop:bound_logistic}
Consider $T$, $\pi(x)$, $\mu^L_{T}$, $\mu^U_{T}$, $\munderbar{\Sigma}^*$ and $\bar{\Sigma}^*$ defined as in Proposition \ref{proposition:bounds_probit}. 
Let $\sigma (\xi)$ be the sigmoid, then we have that:
\begin{align} \label{eq:lb_sigmoid1}
   \piInfL{T} &\defnotation \Pi \left( \mu_T^L,  \munderbar{\Sigma}^*  \right) \leq \piInf{T} \quad \\
    \label{eq:lb_sigmoid2}
    \piSup{T} &\leq \Pi \left( \mu_T^U, \bar{\Sigma}^*   \right) =: \piSupU{T}.
\end{align}
\end{proposition}
\begin{proof}
We show that the derivatives have the same sign as the probit, and then the proof follows as for probit.
More specifically, we have that:
\begin{align*}
    \frac{\partial \Phi (\mu,s)}{\partial \mu} &=  \int \sigma'(s t + \mu) \mathcal{N}(t | 0,1)dt > 0
\end{align*}
since the sigmoid is a monotonically increasing function. 

For the derivative with respect to $s$ we want to show that:  
\begin{align*}
    \frac{\partial \Phi (\mu,s)}{\partial s} = \int t \sigma'(s t + \mu) \mathcal{N}(t|0,1)dt = \begin{cases} < 0 \quad \text{if} \quad \mu > 0  \\  = 0 \quad \text{if} \quad \mu = 0  \\ > 0 \quad \text{if} \quad \mu < 0   \end{cases}.
\end{align*}
The case for $\mu = 0$ is trivial. 
For the remaining cases we have:
\begin{align*}
    &\int t \sigma'(s t + \mu) \mathcal{N}(t|0,1)dt =  \int_{- \infty}^0 t \sigma'(s t + \mu) \mathcal{N}(t|0,1)dt + \int^{+ \infty}_0 t \sigma'(s t + \mu) \mathcal{N}(t|0,1)dt \\
    =& \int^{+ \infty}_0 t \sigma'(-s t + \mu) \mathcal{N}(t|0,1)dt + \int^{+ \infty}_0 t \sigma'(s t + \mu) \mathcal{N}(t|0,1)dt \\
    =&  \int^{+ \infty}_0 t \left( \sigma'(\mu + s t) -  \sigma'(\mu - s t) \right) \mathcal{N}(t|0,1)dt
\end{align*}
and Lemma \ref{lemma:sign_of_sigmoid_derivative} (see Appendix \ref{sec:lemmas_and_proofs}) can be applied to get the sign of the integral, since $t$ and $\mathcal{N}(t|0,1)$ are always positive in $[0, + \infty)$.
\end{proof}

Though the methods provided in this section suffice for the solution of Problem \ref{prob:adv_pred} in the two-class case, in Section \ref{sec:bnb_alg} we will show how the bounds described above can be utilised to develop a branch-and-bound scheme for their refinement to ensure convergence to $\piInf{T}$ and $\piSup{T}$. 
Before we do this, we show in the next subsection how to compute bounds for multi-class classification.
%

%

\subsection{Bounds for Multi-class Classification}
\label{sec:multiclass}
In this section we generalise the results for two-class classification.
Given a class index $i \in \{1,\ldots, m \},$ we are interested in computing upper and lower bounds on the $i$th component of the predictive posterior distribution $\pi_i(x )$ (see Equation \eqref{eq:predictive_posterior_gen}) for every $x \in T$, with $T$ an axis-aligned hyper-rectangle in the input space.
For simplicity, we explicitly tackle only the softmax likelihood, but similar arguments can be applied to the case of the multi-dimensional probit, as well as other likelihood functions that have similar monotonicity properties.

%
In the following we show that bounds on the multi-class predictive distribution
can be computed by discretising the integral over the latent space.
\begin{proposition}
\label{Theorem:BOundsGenericMulticlass}
Consider a predictive posterior distribution $\pi(x)$  defined as in Equation \eqref{eq:predictive_posterior_gen}, an input box $T \subseteq \outputspace$, and define $\pi_{\min,i}(T)$ and $\pi_{\max,i}(T)$ as in Equation \eqref{eq:min_max_pi}.
Let $\mathcal{S}=\{S_l = [a_l,b_l] \mid l \in \{1,\ldots,M \}\}$  be a finite partition of the latent space $\mathcal{F} = \mathbb{R}^{m}$, with $[a_l,b_l] = [a_{l,1},b_{l,1}] \times \ldots \times  [a_{l,m},b_{l,m}] $.
Then, for $i \in \{1,\ldots,m \}$:
\begin{align}
    &\pi_{\min,i}(T) \geq  \sum_{l=1}^{M} \sigma_i(\munderbar{\xi}^{l}) \min_{x \in T} \int_{S_l} \mathcal{N}(\xi | \bar{\mu}(x),\bar{\Sigma}(x))d\xi \nonumber
   \\
    &\pi_{\max,i}(T) \leq  \sum_{l=1}^{M} \sigma_i(\bar{\xi}^{l}) \max_{x \in T} \int_{S_l} \mathcal{N}(\xi  | \bar{\mu}(x),\bar{\Sigma}(x))d\xi. \nonumber
\end{align}
where
\begin{align*}
    \munderbar{\xi}^{l} &=  [b_{l,1},\ldots,b_{l,i-1},a_{l,i},b_{l,i+1},\ldots,b_{l,m}] \\
    \bar{\xi}^{l} &=   [a_{l,1},\ldots,a_{l,i-1},b_{l,i},a_{l,i+1},\ldots,a_{l,m}].
\end{align*}
\end{proposition}
\begin{proof}
We prove the statement for the minimum; the arguments for the maximum are analogous.
By simple properties of integrals and definition of the minimum we have that:
\begin{align*}
    \pi_{\min,i}(T) =& \min_{x \in T} \int_{\outputspace} \sigma(\xi) \mathcal{N}(\xi | \bar{\mu}(x),\bar{\Sigma}(x))d\xi = \min_{x \in T} \sum_{l=1}^M \int_{S_l} \sigma(\xi) \mathcal{N}(\xi | \bar{\mu}(x),\bar{\Sigma}(x))d\xi \\
    \leq & \sum_{l=1}^M   \min_{x \in T} \int_{S_l} \sigma(\xi) \mathcal{N}(\xi | \bar{\mu}(x),\bar{\Sigma}(x))d\xi.
\end{align*}
Taking the partial derivatives of the softmax function with respect to coordinate $k \in \{1,\ldots,m\}$ we have that:
\begin{align*}
    \frac{\partial \sigma_i(\xi)}{\partial \xi_k} = \begin{cases} \sigma_i(\xi) (1 - \sigma_i(\xi)) \quad &\text{if} \; k=i\\  - \sigma_i(\xi) \sigma_k(\xi)   \quad &\text{if} \; k\neq i \end{cases}
\end{align*}
and hence we obtain that the $i$-th component of the softmax function is monotonically increasing along the direction $i$ and monotonically decreasing along all the other dimensions $k \neq i$.
Thus, its minimum in a generic axis-aligned hyper-rectangle $[a_{l,1},b_{l,1}] \times \ldots \times  [a_{l,m},b_{l,m}]$ will be found in the vertex defined as $\munderbar{\xi}^{l} = [b_{l,1},\ldots,b_{l,i-1},a_{l,i},b_{l,i+1},\ldots,b_{l,m}]$.
Therefore, we have that the chain of inequalities above can be lower-bounded by computing the softmax on $\munderbar{\xi}^{l}$ and taking it outside of the integral computation, which yields: 
\begin{align*}
    \sum_{l=1}^{M} \sigma_i(\munderbar{\xi}^{l}) \min_{x \in T} \int_{S_l} \mathcal{N}(\xi | \bar{\mu}(x),\bar{\Sigma}(x))d\xi.
\end{align*}
\end{proof}
Summing up, Proposition \ref{Theorem:BOundsGenericMulticlass} guarantees that, for all $x\in T$, $\pi_i(x)$  can be upper- and lower-bounded by solving $M$ optimisation problems over a multi-dimensional Gaussian integral.
In Proposition \ref{Prop:MultiCLass} below, we show that upper and lower bounds for the integral of a multi-dimensional Gaussian distribution, such as those appearing in Proposition \ref{Theorem:BOundsGenericMulticlass}, can be obtained by optimising a marginalised product of uni-dimensional Gaussian integrals over both the input and the latent space.

We first introduce the following notation.
We denote with $\bar{\mu}_{i:j}(x)$ the subvector of $\bar{\mu}(x)$ containing only the components from $i$ to $j$, with $i\leq j$, and similarly we define $\bar{\Sigma}_{i:k,j:l}(x)$ to be the submatrix of $\bar{\Sigma}(x)$ containing rows from $i$ to $k$ and columns from $j$ to $l$, with $i\leq k$ and $j \leq l$. 
\begin{proposition}
\label{Prop:MultiCLass}
Let $S=\prod_{i=1}^{m}[ a_{i}, b_{i}] \subseteq \outputspace$ be an axis-aligned hyper-rectangle in the latent space, and consider the posterior mean and variance functions $\bar{\mu}(x)$ and $\bar{\Sigma}(x)$.
For $i \in \{1,\ldots,m-1 \}$ and $f_\mathcal{I} \in \mathbb{R}^{m-i-1}$, define $\mathcal{I}=(i+1):m$ and 
\begin{align}
    \bar{\mu}^f_i(x) &=\bar{\mu}_i(x)-\bar{\Sigma}_{i,\mathcal{I}}(x)\bar{\Sigma}_{\mathcal{I},\mathcal{I}}^{-1}(x)(f_\mathcal{I}-\bar{\mu}_{\mathcal{I}}(x)) \label{eq:marg_latent_mean}\\
    \bar{\Sigma}^f_i(x) &= \bar{\Sigma}_{i,i}(x)-\bar{\Sigma}_{i,\mathcal{I}}(x)\bar{\Sigma}_{\mathcal{I},\mathcal{I}}^{-1}(x)\bar{\Sigma}_{i,\mathcal{I}}^T(x). \label{eq:marg_latent_var}
\end{align}
Let $S_\mathcal{I}=\prod_{j=i+1}^{m} [ a_{i},b_{i}]$, then we have that:
\begin{align}
\max_{x\in T}\int_{S} \mathcal{N}(\xi|\bar{\mu}(x),\bar{\Sigma}(x)) d\xi &\leq \nonumber \\ \max_{x\in T} \int_{a_m}^{b_m}&\mathcal{N}(\xi|\bar{\mu}_{m}(x), \bar{\Sigma}_{m,m}(x) )d\xi  \prod_{i=1}^{m-1}  \max_{\substack{x\in T \\ f\in S_\mathcal{I}}} \int_{a_i}^{b_i} \mathcal{N}(\xi|\bar{\mu}^f_i(x),\bar{\Sigma}^f_i(x) )d\xi \label{eq:max_multi_int_form}\\
\min_{x\in T}\int_{S} \mathcal{N}(\xi|\bar{\mu}(x),\bar{\Sigma}(x))d\xi &\geq \nonumber\\ \min_{x\in T} \int_{a_m}^{b_m} &\mathcal{N}(\xi|\bar{\mu}_{m}(x), \bar{\Sigma}_{m,m}(x) )d\xi  \prod_{i=1}^{m-1}  \min_{\substack{x\in T \\ f\in S_\mathcal{I}}}  \int_{a_i}^{b_i}  \mathcal{N}(\xi|\bar{\mu}^f_i(x),\bar{\Sigma}^f_i(x) )d\xi. \label{eq:min_multi_int_form}
\end{align}
\end{proposition}
\begin{proof}
We consider the case of the minimum; the maximum follows similarly.

Consider the latent posterior process $\GP$, whose mean and variance function we denote with $\bar{\mu}(x)$  and  $\bar{\Sigma}(x)$.
Then, we have 
\begin{align*}
\min_{x\in T}&\int_{S}\mathcal{N}(\xi|\bar{\mu}(x),\bar{\Sigma}(x)) d\xi = \min_{x \in T} {P(\GP(x)\in S)} = \min_{x \in T}  P(a_i\leq \GP_i(x)\leq b_i , i=1,\ldots,m) = \\ 
&\min_{x \in T} P(a_m\leq \GP_m(x)\leq b_m) \prod_{i=1}^{m-1} {P(a_i\leq \GP_i(x)\leq b_i| \GP_{\mathcal{I}}(x) \in  S_\mathcal{I} )} \geq  \\
&\quad\quad \text{(By Lemma \ref{Lemma:supProbab} included in the Appendix \ref{sec:lemmas_and_proofs})}\\
& \min_{x \in T} P(a_m\leq \GP_m(x)\leq b_m) \prod_{i=1}^{m-1} \min_{f_\mathcal{I} \in S_\mathcal{I}} P(a_i\leq \GP_i(x)\leq b_i | \GP_\mathcal{I}(x)=f_\mathcal{I} ) \geq \\
& \min_{x \in T} P(a_m\leq \GP_m(x)\leq b_m) \prod_{i=1}^{m-1}\min_{\substack{x \in T \\f_\mathcal{I} \in S_\mathcal{I}}} P(a_i\leq \GP_i(x)\leq  b_i  | \GP_\mathcal{I}(x)=f_\mathcal{I} ). 
\end{align*}
Notice that, for each $i \in \{1,\ldots,m-1 \}$, $ P(a_i\leq \GP_i(x)\leq  b_i  | \GP_\mathcal{I}(x)=f_\mathcal{I} ) $ is the integral of a uni-dimensional Gaussian random variable conditioned  on a jointly Gaussian random variable.
The statement of the proposition then follows by the application of the conditioning equations for Gaussian distributions.
\end{proof}

Proposition \ref{Prop:MultiCLass} reduces the computation of the multi-class bounds to a product of extrema computations over univariate Gaussian distributions. 
To solve this, we first need to compute lower and upper bounds for the conditional latent mean and the conditional latent variance defined in Equations \eqref{eq:marg_latent_mean} and \eqref{eq:marg_latent_var}.
Observe that Equations \eqref{eq:marg_latent_mean} and \eqref{eq:marg_latent_var} can be expressed as a rational function in the entries of the mean vector, variance matrix and latent variable vector.
We can thus propagate the upper and lower bound of each entry from the mean vector and covariance matrix down through the rational function equations by simple interval bound propagation techniques, which results in an upper and lower bound on $\bar{\mu}^f_i(x) $ and $\bar{\Sigma}^f_i(x) $ for $x\in T$ and $f\in S_\mathcal{I}$, which we denote with $\mu^{L,f}_{i,T}$, $\mu^{U,f}_{i,T}$,  $\Sigma^{L,f}_{i,T}$ and $\Sigma^{U,f}_{i,T}$.
This process can then be iterated backward from $i=m$ to $i=1$, up until all the required bounds are computed.
Unfortunately, because of the need to symbolically compute a matrix inversion, the explicit formulas for the computation of  $\mu^{L,f}_{i,T}$, $\mu^{U,f}_{i,T}$,  $\Sigma^{L,f}_{i,T}$ and $\Sigma^{U,f}_{i,T}$ in general are rather convoluted and long (though still in the form of a simple ratio between polynomials). 

Once those bounds are computed, we rely on the following lemma for the solution of the optimisation problem over the Gaussian integrals. 

\begin{lemma}\label{Lemma:Gaussian}
Consider $ S_\mathcal{I}$, $a_i$, $b_i$, $\bar{\mu}_i^f(x)$ and $\bar{\Sigma}_i^f(x)$ defined as in Proposition \ref{Prop:MultiCLass}, an input box $T \subseteq \inputspace$, and $\mu^{L,f}_{i,T}$,  $\mu^{U,f}_{i,T}$, $\Sigma^{L,f}_{i,T}$ and $\Sigma^{U,f}_{i,T}$, lower and upper bounds on $\bar{\mu}_i^f(x)$ and $\bar{\Sigma}_i^f(x)$ in $T$ computed as discussed above.
Define $\zeta \defnotation [x,f]$ and its input region as $Z = T \times S_\mathcal{I}$.
Let $i=\{1,\ldots,m \}$, $\mu^c_i = \frac{a_i+b_i}{2}$ and $\Sigma^c_i(\mu) = \frac{{(\mu-a_i)^2-(\mu-b_i)^2}}{{2 \log \frac{\mu-a_i}{\mu-b_i}}} $.
Then it holds that:
\begin{align}
    &\max_{\zeta \in Z}\int_{a_i}^{b_i} \mathcal{N}( \xi | \bar{\mu}_i^f(x),\bar{\Sigma}_i^f(x))d \xi \leq \int_{a_i}^{b_i} \mathcal{N}( \xi \vert \bar{\mu}^* , \bar{\Sigma}^* ) d\xi \nonumber \\
    &= \frac{1}{2} \left( \erf \left(\frac{\bar{\mu}^*-a_i}{\sqrt{2\bar{\Sigma}^*}}\right) - \erf \left(\frac{\bar{\mu}^*-b_i}{\sqrt{2\bar{\Sigma}^*}}\right)  \right) \label{eq:gen_like2}\\
    &\min_{\zeta \in Z}\int_{a_i}^{b_i} \mathcal{N}(\xi | \bar{\mu}_i^f(x),\bar{\Sigma}_i^f(x))d\xi  \geq \int_{a_i}^{b_i} \mathcal{N}( \xi \vert \munderbar{\mu}^* , \munderbar{\Sigma}^* ) d\xi \nonumber \\
    &= \frac{1}{2} \left ( \erf \left(\frac{\munderbar{\mu}^*-a_i}{\sqrt{2\munderbar{\Sigma}^*}}\right) - \erf \left(\frac{\munderbar{\mu}^*-b_i}{\sqrt{2\munderbar{\Sigma}^*}}\right)  \right) \label{eq:inf_gauss_int}
\end{align}
where we have:
\begin{align*}
    \bar{\mu}^* &= \argmin_{\mu \in [ \mu^{L,f}_{i,T}, \mu^{U,f}_{i,T} ]} \vert \mu^c_i - \mu \vert,  \quad
    \bar{\Sigma}^* = \begin{cases}
                        \Sigma^{L,f}_{i,T} \quad &\text{if} \quad \bar{\mu}^* \in [a_i,b_i] \\
                        \argmin_{\Sigma \in [\Sigma^{L,f}_{i,T}, \Sigma^{U,f}_{i,T}]} \vert \Sigma^c_i(\bar{\mu}^*) - \Sigma \vert \quad &\text{otherwise}
                        \end{cases} \\
    \munderbar{\mu}^* & = \argmax_{\mu \in [  \mu^{L,f}_{i,T},  \mu^{U,f}_{i,T} ]} \vert \mu^c_i - \mu \vert,  \quad
    \munderbar{\Sigma}^* = \argmin_{\Sigma \in \{\Sigma^{L,f}_{i,T}, \Sigma^{U,f}_{i,T}\}} [\erf(b_i \vert \munderbar{\mu}^*, \Sigma) - \erf(a_i \vert \munderbar{\mu}^*, \Sigma)]. 
\end{align*}
\end{lemma}
By iterating the computation of Lemma \ref{Lemma:Gaussian} for each integral in Proposition \ref{Prop:MultiCLass}, we obtain the bounds on the predictive distribution.
The discretised bound can also be used for two-class classification, in cases where a likelihood function different from the probit and the logistic sigmoid is desired.

\subsection{Bounds for Regression}\label{sec:regression_bound}
While computing adversarial robustness guarantees for classification models involves the computation of upper and lower bounds on the GP posterior predictive distribution, the analysis is much simpler for regression.
As stated in Section \ref{sec:background} and formalised in Problem \ref{prob:mean_pred}, for the canonical loss function the optimal decision corresponds to the posterior latent mean function $\bar{\mu}(x)$ of the posterior GP distribution, whose computation is given in Section \ref{sec:prob_form}.
Guarantees over the decision can then be made simply by relying on upper and lower bounds for the mean function, that is, $\mu^{L}_{i,T}$ and $\mu^{U}_{i,T}$ for every $i=1,\ldots,m$, which makes over-approximation of Definition \ref{def:adversarial_regr} much faster and simpler in practice.
\begin{proposition}
Consider a box $T \subseteq \inputspace$ of the input space, a test point $x^* \in T$, an $\ell_p$ metric $||\cdot||$ in the output space $\outputspace$ and a $\delta > 0$. 
Let $\bar{\mu}$ be the predictive posterior mean, and $\mu^{L}_{i,T}$ and $\mu^{U}_{i,T}$, for every $i=1,\ldots,m$, its upper and lower bounds computed according to Proposition \ref{prop:mean_bound}.
Define $\mu^*_T$ as the vector of entries:
\begin{align*}
    \mu^*_{T,i} = \begin{cases} \mu^{L}_{i,T} \quad \text{if} \quad  | \bar{\mu}_i(x^*) - \mu^{L}_{i,T} |  \geq | \bar{\mu}_i(x^*) - \mu^{U}_{i,T} | \\
    \mu^{U}_{i,T} \quad \text{if}  \quad | \bar{\mu}_i(x^*) - \mu^{L}_{i,T} |  < | \bar{\mu}_i(x^*) - \mu^{U}_{i,T} |. \end{cases}
\end{align*}
Then:
\begin{align*}
    \sup_{x \in T} || \bar{\mu}(x^*) - \bar{\mu}(x)   || \leq || \bar{\mu}(x^*) - \mu^*_T   ||.
\end{align*}
Consequently, if:
\begin{align*}
    || \bar{\mu}(x^*) - \mu^*_T   || \leq \delta
\end{align*}
then the GP is $\delta$-robust in $x^*$ w.r.t.\ $T$ and norm $|| \cdot ||$. 
\end{proposition}
\begin{proof}
By construction, we have that $\bar{\mu}_i(x) \in [ \mu^{L}_{i,T} , \mu^{U}_{i,T}  ]$ for every $x \in T$. 
Hence, by monotonicity of $\ell_p$ norms along the coordinate directions and  by definition of $\bar{\mu}(x)$, it follows that $\sup_{x \in T} || \bar{\mu}(x^*) - \bar{\mu}(x)   || \leq || \bar{\mu}(x^*) - \mu^*_T   ||$.
Thus:
\begin{align*}
    \delta \geq  || \bar{\mu}(x^*) - \mu^*_T   || \geq \sup_{x \in T} || \bar{\mu}(x^*) - \bar{\mu}(x)   || \geq  || \bar{\mu}(x^*) - \bar{\mu}(x)   ||, \quad \text{for} \quad x \in T.
\end{align*}
In particular, $\delta  \geq || \bar{\mu}(x^*) - \bar{\mu}(x)   ||$, which is equivalent to Definition \ref{def:adversarial_regr}.
\end{proof}


%

%% file: sections/probit.tex
We first consider the probit likelihood, i.e.,  $\sigma (\xi)= \Phi ( \lambda \xi)$ is the cdf of the univariate standard Gaussian distribution scaled by $\lambda > 0$.
In this case, the predictive distribution can be written down in closed form, which greatly simplifies the computation of the bounds:
\begin{equation*}
    \pi(x) = \Phi \left( \frac{\bar{\mu}(x)}{ \sqrt{ \lambda^{-2} + \bar{\Sigma} (x) }} \right).
\end{equation*} 
We can use this explicit form to derive analytic upper and lower bounds by direct inspection of the predictive distribution derivatives with respect to the induced mean and variance variables.
The following proposition provides a solution for Problem \ref{prob:adv_pred} in the case of two-class classification with the probit likelihood.
\begin{proposition}\label{proposition:bounds_probit}
Consider a predictive posterior distribution $\pi(x)$ defined as in Equation \eqref{eq:post_pred_dist_analytic}, input box $T \subseteq \inputspace$, and  $\mu^L_{T}$, $\mu^U_{T}$, $\Sigma^L_{T}$ and $\Sigma^U_{T}$, lower and upper bounds on the GP posterior variance, computed as detailed in Section \ref{sec:mean_and_var_opts}.
Let $\sigma (\xi)= \Phi ( \lambda \xi)$, with $\lambda > 0$, then we have that
\begin{align} \label{eq:lb_probit1}
   \piInfL{T} &\defnotation \Phi \left( \frac{\mu_T^L}{\sqrt{\lambda^{-2} + \munderbar{\Sigma}^*  }} \right) \leq \piInf{T} \quad \\
    \label{eq:lb_probit2}
    \piSup{T} &\leq \Phi \left( \frac{\mu_T^U}{\sqrt{\lambda^{-2} + \bar{\Sigma}^*  }} \right) =: \piSupU{T}
\end{align}
with
\begin{align*}
    \munderbar{\Sigma}^* = \begin{cases}  
                                {\Sigma}^U_T \quad &\text{if} \quad \mu_T^L \geq 0 \\
                                {\Sigma}^L_T \quad &\text{otherwise}
                            \end{cases} \qquad
    \bar{\Sigma}^* = \begin{cases}  
                                {\Sigma}^L_T \quad &\text{if} \quad \mu_T^U \geq 0 \\
                                {\Sigma}^U_T \quad &\text{otherwise}.
                            \end{cases}
\end{align*}
\end{proposition}
\begin{proof}
We have:
\begin{align*}
    \Pi(\mu,\Sigma) = \Phi \left( \frac{\mu}{ \sqrt{ \lambda^{-2} + {\Sigma} }} \right).
\end{align*}
As $\Phi$ is monotonically increasing, it suffices to optimise for the argument $\phi (\mu, \Sigma) = \frac{\mu}{\sqrt{\lambda^{-2} + \Sigma}}$. 
By computing the partial derivatives it is easy to see that $\frac{\partial \phi (\mu,\Sigma)  }{\partial \mu} > 0$ for all values of $\mu$ and $\Sigma$. Therefore, for every value of $\Sigma$ the minimum is obtained for $\mu = \mu_T^L$.
On the other hand, for the derivative wrt to $\Sigma$ we have that:
\begin{align*}
    \frac{\partial \phi (\mu_T^L,\Sigma)  }{\partial \Sigma}
    \begin{cases} < 0  \quad \textrm{if} \quad  \mu_T^L > 0 \\
    = 0 \quad \textrm{if} \quad  \mu_T^L = 0 \\
    > 0 \quad \textrm{if} \quad  \mu_T^L < 0 
    \end{cases}
\end{align*}
as $\Sigma > 0$. Hence, given $\mu_T^L$, we have that $\phi$ is monotonic in $\Sigma$ and the proposition follows.
\end{proof}

%% file: sections/bnb.tex
In this section we formulate a branch-and-bound algorithmic scheme that incorporates the lower- and upper-bounding procedures for Gaussian process models introduced in Section \ref{sec:adv_rob_method} 
and prove  its convergence up to any a-priori specified  $\epsilon > 0 $. 
For simplicity of exposition, we restrict the discussion to two-class classification, noting that the
multi-class classification and regression problems follow analoguously
by substituting appropriate bounding procedures. 
The main idea behind branch-and-bound optimisation is to alternate between bounding the function we are interested in optimising in our 
input box $T$ and splitting $T$ into smaller boxes, i.e., \textit{candidate search regions}, on which we compute the bound in the next iteration.
This procedure creates a search tree, in which descending depth implies smaller search regions.
The intuition is that, as we explore the branch-and-bound search tree depth-first, the search regions become smaller, so that the bounds get closer to the true function, and we thus slowly converge to the actual optimum.
By computing lower and upper bounds on the quantity of interest, we are then able to prune our search tree for regions in which optimal values cannot occur.

We now describe the proposed branch-and-bound scheme for the computation of lower and upper bounds for $\piInf{T}$ derived in Section \ref{sec:bin_bound_class}, which is summarised in Algorithm \ref{alg:bnb_sketch}. 
After initialising $\piInfL{T}$ and $\piInfU{T}$ to trivial values and 
the exploration regions stack $\mathbf{R}$ to the singleton $\{T\}$, the main optimisation loop is entered until convergence (lines 2--10).
Among the regions in the stack, we select the region $R$ with the most promising lower bound (line 3).
After bounding posterior mean and variance in $R$ (line 4), we refine its lower bound using Proposition \ref{proposition:bounds_probit} for the probit likelihood and Proposition \ref{prop:bound_logistic} for the logistic sigmoid likelihood (line 5), as well as its upper bound through evaluation of points in $R$ (line 6). If further exploration of $R$ is necessary for convergence (line 7), then the region $R$ is partitioned into two smaller regions $R_1$ and $R_2$, which are added to the  regions stack and inherit $R$'s bound values (line 8).
We perform the split by randomly selecting an index $j \in \{1,\ldots,d\} $ from the input dimensions, and by splitting $R$ at the mid-point along the $j$th dimension.
Finally, the freshly computed bounds local to $R \subseteq T$ are used to update the global bounds for $T$ (line 10). Namely, $\piInfL{T}$ is updated to the smallest value among the $\piInfL{R}$ values for $R \in \mathbf{R}$, while $\piInfU{T}$ is set to the lowest observed value yet explicitly computed in line 6.
\begin{algorithm} 
\caption{Branch and bound for  $\piInf{T}$}\label{alg:bnb_sketch}
\textbf{Input:} Input space subset $T$; error tolerance $\epsilon>0$; latent mean/variance functions $\bar{\mu}(\cdot)$ and $\bar{\Sigma}(\cdot)$. \\
\textbf{Output:} Lower and upper bounds on $\piInf{T}$ with $\piInfU{T} - \piInfL{T} \leq \epsilon$
\begin{algorithmic}[1]
\State \emph{Initialisation:} Stack of regions $\mathbf{R} \gets \{ T \}$; \quad $\piInfL{T} \gets - \infty $; \quad $\piInfU{T} \gets + \infty $
\While{ $\piInfU{T} - \piInfL{T} > \epsilon$}
\State Select region $R \in \mathbf{R}$ with lowest bound $\piInfL{R}$
and delete it from stack
\State Find $[\mu^L_R,\mu^U_R]$ and $[\Sigma^L_R,\Sigma^U_R]$ applying Propositions \ref{prop:mean_bound}, \ref{prop:variance_min} and \ref{prop:variance_max} over $R$
\State Compute $\piInfL{R}$ from $[\mu^L_R,\mu^U_R]$ and $[\Sigma^L_R,\Sigma^U_R]$ using Proposition \ref{proposition:bounds_probit} or \ref{prop:bound_logistic} resp.
\State Find $\piInfU{R}$ by evaluating $\pi(x)$ in a point in $R$
\If{$\piInfU{R} - \piInfL{R} > \epsilon$}
\State Split $R$ into two sub-regions $R_1,R_2$, add them 
 to stack
\State Use $\piInfL{R},\piInfU{R}$ as initial 
  bounds for both sub-regions $R_1,R_2$
\EndIf
\State Update $\piInfL{T}$ and $\piInfU{T}$ with current best 
bounds found 
\EndWhile
\State \textbf{return} $[\piInfL{T},\piInfU{T}]$
\end{algorithmic}
\end{algorithm}

We remark that to derive a branch-and-bound scheme for multi-class classification (respectively, regression) it suffices to replace line $5$ in Algorithm \ref{alg:bnb_sketch} with the bounding methods of Section \ref{sec:multiclass} (respectively, Section \ref{sec:regression_bound}). 

\paragraph{Computation of Under-approximations}
As discussed in Section \ref{sec:outline_of_approach}, in order to obtain valid values for $\piInfU{T}$ and $\piSupL{T}$ it suffices to evaluate the GP posterior predictive distribution in any point of $T$.
However, the closer $\piInfU{T}$ and $\piSupL{T}$ are to $\piInf{T}$ and $\piSup{T}$, respectively, the faster a branch-and bound-algorithm will converge.
By solving the optimisation problems associated to  $\mu^L_T, \mu^U_T$, $\Sigma^L_T $ and $\Sigma^U_T$, we obtain four extrema points in $T$ on which the GP assumes the optimal values for the posterior mean and variance bounds.
As these points belong to $T$ and provide extreme points for the latent function, they are promising candidates for the evaluation of $\piInfU{T}$ and $\piSupL{T}$.
We thus
evaluate the GP predictive posterior distribution on all four extremal points and select the one that yields the best bound. 

\paragraph{Convergence}

By construction it is clear that, if Algorithm  \ref{alg:bnb_sketch}  terminates, the resulting values over- and under-approximate the true value $\piInf{T}$ with a known error $\epsilon > 0$.
We now show, by relying on the theory of convergence for branch-and-bound algorithms, that the loop of lines $2-9$ terminates in a finite number of iterations.
In particular, to prove convergence of a branch-and-bound scheme up to an error $\epsilon>0$ it suffices to show that the two following conditions hold \citep{balakrishnan1991branch}:
\begin{enumerate}
    \item \textit{Consistency Condition}: $ \piInfL{R} \leq \piInf{R} \leq \piInfU{R} \qquad \forall R \subseteq T.$
    \item \textit{Uniform Convergence}: $\forall \epsilon > 0 \; \; \exists \, r > 0 \; \; \text{s.t.} \; \forall R\subseteq T$ with $\text{diam}(R) \leq r \Rightarrow \vert \piInfU{R} - \piInfL{R}  \vert \leq \epsilon$.
\end{enumerate}
Intuitively, the first condition ensures that the computed bounds are consistent across all the subsets of the initial input region $T$.
This is clearly satisfied by construction, 
see Section \ref{sec:bin_bound_class}.
The second condition enforces that the lower and the upper bounds converge uniformly to each other as we reduce the maximum diameter of the branch-and-bound search region to zero. 
In the following theorem we show that the bound based on latent space discretisation has the uniform convergence property and converges in finitely many steps.
Consequently, as the analytical bounds that we compute for the probit and the logistic function are tighter than for discretisation, they will also converge.
For simplicity of exposition, we prove convergence for two-class classification, which also captures regression as a special case; we provide details below for how the result can be generalised to the multi-class case.
%
%
\begin{theorem}
\label{TH:convergenceBrenchAndBOund}
Let $T$ be a box in the input space $\inputspace$.
Consider a two-class classification GP with posterior mean and variance given by $\bar{\mu}(x)$ and $\bar{\Sigma}(x)$. 
Assume that $\mu^L_R$, $\mu^U_R$, ${\Sigma}^L_R $, ${\Sigma}^U_R$ are bounding functions for the posterior mean and variance such that:
\begin{align}\label{eq:mean_and_variance_cond}
    \mu^L_R \to \min_{x\in R} \bar{\mu}(x), \quad \mu^U_R \to \max_{x\in R} \bar{\mu}(x), 
    \quad \Sigma^L_R \to \min_{x\in R} \bar{\Sigma}(x), \quad \Sigma^U_R \to \max_{x\in R} \bar{\Sigma}(x)
\end{align}
whenever 
$\text{diam}(R)\to 0$.
Then, for $\epsilon >0$, there exists a partition of the latent space $\mathcal{S}$ and $\bar{r}>0$  such that, for every $R \subseteq T$ with $\text{diam}(R) < \bar{r}$, it holds that 
\begin{align}\label{eq:cond2prove}
    \vert \piInfU{R} - \piInfL{R} \vert \leq \epsilon. 
\end{align}
\end{theorem}
\begin{proof}
Consider an $\epsilon > 0$, and a generic axis-aligned hyper-rectangle $R\subseteq T$ of diameter $\text{diam}(R) \defnotation r > 0$ less than a fixed $\bar{r} > 0$.
We want to find a value for $\bar{r}$ for which the condition in Equation \eqref{eq:cond2prove} is surely met.
We start by observing that $\piInfU{R}$ is defined by computing the predictive posterior distribution on a fixed point of $R$. 
Let $\bar{x} \in R$ be such a point, and define $\bar{\mu} \defnotation \bar{\mu}(\bar{x})$ and  $\bar{\Sigma} \defnotation \bar{\Sigma}(\bar{x})$, then we have that:
\begin{align*}
    \piInfU{R} = \int \sigma(\xi) \mathcal{N}(\xi | \bar{\mu},\bar{\Sigma})d\xi.
\end{align*}
Now consider a generic $M > 0$; we define the discretisation of the latent space $\mathcal{S}_M = \{ [a_l,b_l] \; |  \; l=1\ldots,M \}$ with the following equations:
\begin{align*}
    a_1 &= -\infty \\
    b_l &= \sigma^{-1} \left( \sigma(a_{l}) + \frac{1}{M} \right) \quad l=1,\ldots,M\\
    a_{l+1} &= b_l \qquad \qquad \qquad \qquad \;\;  l=1,\ldots,M,
\end{align*}
that is, we discretise the $y$-axis into $M$ equally distanced intervals and map that discretisation back to the $x$-axis through the link function, $\sigma^{-1}$.
We then have that the left-hand-side of Equation \eqref{eq:cond2prove} can be written explicitly as:
\begin{align}\label{eq:partial_passage1}
 \left\vert   \int \sigma(\xi) \mathcal{N}(\xi | \bar{\mu},\bar{\Sigma} )d\xi -  \sum_{l=1}^{M}  \sigma(a_l) \min_{\substack{  \mu \in [\mu^L_R,\mu^U_R] \\ \Sigma \in [{\Sigma}^L_R ,{\Sigma}^U_R]  }} \int_{a_l}^{b_l} \mathcal{N}(\xi | \mu,\Sigma)d\xi  \right\vert.
\end{align}
%
Let $\mu^{*,(l)}$ and $\Sigma^{*,(l)}$ be the solutions to the $l$th minimisation problems defined inside the summation of the equation above, then we have:
\begin{align}
    &\left\vert   \int \sigma(\xi) \mathcal{N}(\xi | \bar{\mu},\bar{\Sigma})d\xi -  \sum_{l=1}^{M}  \sigma(a_l)  \int_{a_l}^{b_l} \mathcal{N}(\xi | \mu^{*,(l)},\Sigma^{*,(l)}) d\xi  \right\vert \nonumber\\
    = & \left\vert \sum_{l=1}^{M}  \left(  \int_{a_l}^{b_l} \sigma(\xi) \mathcal{N}(\xi | \bar{\mu},\bar{\Sigma})d\xi -   \sigma(a_l)  \int_{a_l}^{b_l}  \mathcal{N}(\xi | \mu^{*,(l)},\Sigma^{*,(l)})d\xi \right)  \right\vert \nonumber\\
    \leq & \left\vert \sum_{l=1}^{M}  \left( \left(\sigma(a_l) + \frac{1}{M}\right) \int_{a_l}^{b_l}  \mathcal{N}(\xi | \bar{\mu},\bar{\Sigma})d\xi -   \sigma(a_l)  \int_{a_l}^{b_l}  \mathcal{N}(\xi | \mu^{*,(l)},\Sigma^{*,(l)})d\xi \right)  \right\vert \nonumber\\
    \leq & \left\vert  \frac{1}{M} \sum_{l=1}^{M}   \int_{a_l}^{b_l}  \mathcal{N}(\xi | \bar{\mu},\bar{\Sigma})d\xi \right\vert + \left\vert \sum_{l=1}^{M}   \sigma(a_l)\int_{a_l}^{b_l} \left( \mathcal{N}(\xi | \bar{\mu},\bar{\Sigma}) -    \mathcal{N}(\xi | \mu^{*,(l)},\Sigma^{*,(l)})  \right) d\xi        \right\vert \nonumber\\
    \leq &  \frac{1}{M} \left\vert    \int_\mathbb{R}  \mathcal{N}(\xi | \bar{\mu},\bar{\Sigma})d\xi \right\vert +  \sum_{l=1}^{M}   \sigma(a_l) \left\vert\int_{a_l}^{b_l} \left( \mathcal{N}(\xi | \bar{\mu},\bar{\Sigma}) -   \mathcal{N}(\xi | \mu^{*,(l)},\Sigma^{*,(l)})  \right) d\xi        \right\vert \nonumber\\
    \leq & \frac{1}{M} + \sum_{l=1}^{M}  \left\vert\int_{a_l}^{b_l} \left( \mathcal{N}(\xi | \bar{\mu},\bar{\Sigma}) -    \mathcal{N}(\xi | \mu^{*,(l)},\Sigma^{*,(l)})  \right) d\xi        \right\vert. \label{eq:svolgimento_teorema}
\end{align}
Now, thanks to the conditions in Equation \eqref{eq:mean_and_variance_cond}, we have that as $r\to 0$ both mean and variance converge to the actual maximum and minimum values in $R$. 
By further observing that $\bar{\mu}$ and $\bar{\Sigma}$ are by construction always inside the (vanishing) interval $[\mu^L_R,\mu^U_R] \times [{\Sigma}^L_R , {\Sigma}^U_R]$, then for continuity of the Gaussian pdf we have that for each $l=1,\ldots,M$: 
\begin{align*}
    \lim_{r\to0}     \left\vert\int_{a_l}^{b_l} \left( \mathcal{N}(\xi | \bar{\mu},\bar{\Sigma}) -  \mathcal{N}(\xi | \mu^{*,(l)},\Sigma^{*,(l)})  \right) d\xi        \right\vert  = 0
\end{align*}
which means that the second term in Equation \eqref{eq:svolgimento_teorema} can be made vanishingly small, in particular less than $\frac{\epsilon}{2}$. 
By selecting $M = \lceil \frac{2}{\epsilon} \rceil$ the theorem statement holds.
\end{proof}

%


We have proved in Propositions \ref{prop:mean_convergence} and \ref{prop:variance_convergence} that the bounds for the mean and variance of Section \ref{sec:mean_and_var_opts} guarantee that the condition in Equation \eqref{eq:mean_and_variance_cond} holds.
For multi-class classification (case $m>2$), Theorem \ref{TH:convergenceBrenchAndBOund} can be generalised by further noticing that the error introduced by Proposition \ref{Prop:MultiCLass} also vanishes. 
For any $\epsilon > 0$, to ensure that convergence holds for the multi-class problem one has to select a number of discretisation boxes of the order of $\frac{1}{\epsilon^m}$. 

\subsection{Time Complexity}\label{sec:complexity_adversarial}

The method we have developed for the computation of adversarial robustness properties of GPs relies on the bounding of the posterior GP statistics, integrated within a branch-and-bound scheme for the iterative refinement of the bound. 
\paragraph{Cost of Bounding}
Consider a kernel $\Sigma$ with bounded kernel decomposition $(\varphi,\psi,U)$, and let $\mathcal{K}$ denote the time complexity for the evaluation of the bounding function $U$. This is dependent on the particular function chosen, and in Appendix \ref{sec:decomposition} we discuss its value for each kernel that we analyse.
The time complexity for the computation of the mean bound is $\mathcal{O}(m \mathcal{K})$, where $m$ is the output dimension of the GP.
The computation of an upper bound on the posterior variance requires solving a convex quadratic problem, whose computational complexity is cubic in the number of input variables \citep{nesterov1994interior}, i.e.\ $\mathcal{O}((d + 2\tss + \tss d)^3 )$, where $d$ is the input dimensionality of the GP and $\tss$ is the number of training points. 
Concerning the computation of the lower bound on the variance, we have to solve $2 \tss  + 1$ linear programming problems, where $\tss$ is the size of the training set. 
This again depends on the number of optimisation variables and can be done in $\mathcal{O}((d + 2\tss + \tss d)^{2.5}\log(d + 2\tss + \tss d)  )$ \citep{cohen2021solving}.
We emphasise that, while computing the mean is straightforward, bound computations for the variance are more involved.
As a result,  adversarial robustness for regression can be obtained much faster in practice than for classification.

\paragraph{Cost of Refinement}
Once the bounds on the mean and variance have been computed, refining them through branch-and-bound up to a desired threshold $\epsilon > 0$ has a worst-case cost exponential in the number of dimensions of $T$.
Furthermore, for multi-class classification, to guarantee convergence we have to discretise the region into a grid of size  $\frac{1}{\epsilon^m}$, where $m > 2$ is the number of classes. 
This adds to the overall time complexity, which in the multi-class case is exponential also with respect to the number of classes.

%% file: sections/experiments.tex
We 
employ our methods to experimentally analyse the robustness of GP models in adversarial settings.
We give results for fours datasets: (i) Synthetic2D, generated by shifting a two-dimensional standard-normal either along the first (class 1) or second dimension (class 2); (ii) the SPAM dataset \citep{Dua:2019}; (iii) a two-class subset of the MNIST dataset \citep{lecun1998mnist} with classes $3$ and $8$ (i.e., MNIST38) and a three-class subset with classes $3$, $5$ and $8$ (i.e., MNIST358); (iv) a two-class subset of FashionMNIST (F-MNIST) \citep{xiao2017} with classes ``t-shirt/top'' and ``shirt'' (which we refer to as F-MNIST-TS) and a three-class subset with classes ``t-shirt/top'', ``shirt'' and ``pullover'' (F-MNIST-TSP).
\paragraph{Training}
We learn classification GP models using a squared-exponential kernel and zero mean prior and select the hyper-parameters by means of MLE  \citep{williams2006gaussian}. 
For the Synthetic2D dataset we learn the GP over $1000$ training samples and test it over $200$ test samples, obtaining an accuracy of $\approx 98\%$.
For the SPAM dataset we first standardise the data to zero mean and unit variance.
Then, we perform feature-reduction by iteratively training an $\ell_1$-penalised logistic regression classifier and discarding the least relevant features, up until test set accuracy starts to diminish.
This procedure leaves us with 11 features out of the initial 57.
We then train a two-class classification GP over the resulting reduced feature vector.
The GP thus computed achieves a test set accuracy of around $93\%$.

For MNIST and F-MNIST we first sub-sample the images to $14 \times 14$ pixels\footnote{This reduces the number of hyper-parameters that need to be estimated by MLE and increases the numerical stability of the GP, while achieving comparable accuracy.}, and use similar learning settings as for the SPAM dataset, with $1000$ training samples randomly picked from the two datasets.
We achieve a test set accuracy of around $98\%$ for MNIST38 and $90\%$ for  F-MNIST-TS.
Finally, for the two multi-class problems we use the softmax likelihood function and training setting similar to the two-class classification problems, obtaining a test set accuracy of around $93\%$ for MNIST358 and $85\%$ for F-MNIST-TSP.

We rely on the GPML Matlab toolbox for the training of two-class GPs \citep{rasmussen2010gaussian} and on the GPStuff toolbox for the training of multi-class GPs and sparse GP models \citep{vanhatalo2013gpstuff}.

\paragraph{Parameter Selection for the Analysis}
We compute adversarial robustness in neighbourhoods of the form $T = [x^* - \gamma, x^* + \gamma]$ around a given point $x^*$ and for a range of $\gamma > 0$.
Unless otherwise stated, we run the branch-and-bound algorithm until convergence up to an error threshold $\epsilon = 0.01$.
For MNIST38 and F-MNIST-TS we perform feature-level analysis for scalability reasons, similarly to \cite{ruan2018reachability}.
Namely, we restrict our methods to salient patches of each image only, as detected by SIFT \citep{lowe2004distinctive}.
We note that  any other image feature extraction method can be used instead. 

In the remainder of this section, we discuss results concerning three types of analyses.
First (Section \ref{subsec:safety}), on four samples selected from the above datasets, we provide empirical evidence illustrating the advantages of computing guarantees (as those provided by our branch-and-bound method) versus evaluating model robustness using gradient-based attacks.
Next we consider the robustness of GPs learned by using a selection of latent-variable methods (Section \ref{subsec:robustness}) and sparse approximation techniques (Section \ref{subsec:sparse}), discussing the adversarial robustness properties of these state-of-the-art approximate inference methods. 
Finally, we show how the techniques developed here for adversarial robustness can be applied to perform interpretability analysis of GP predictions (Section \ref{subsec:interpretability}).

\begin{figure}[h!]
	\centering
	{\hspace*{0cm}  \includegraphics[width = 0.24\textwidth]{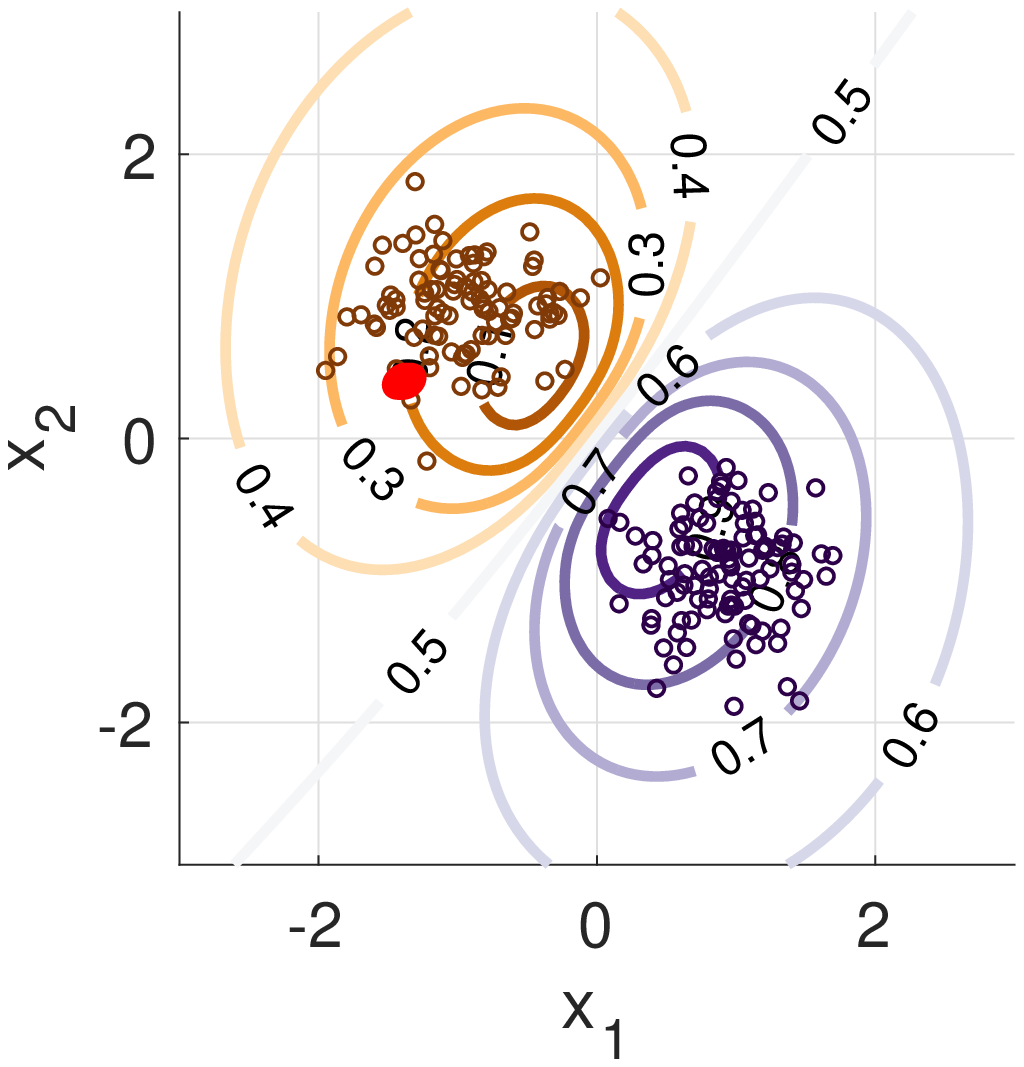}}
	{\hspace*{0cm}  \includegraphics[width = 0.24\textwidth]{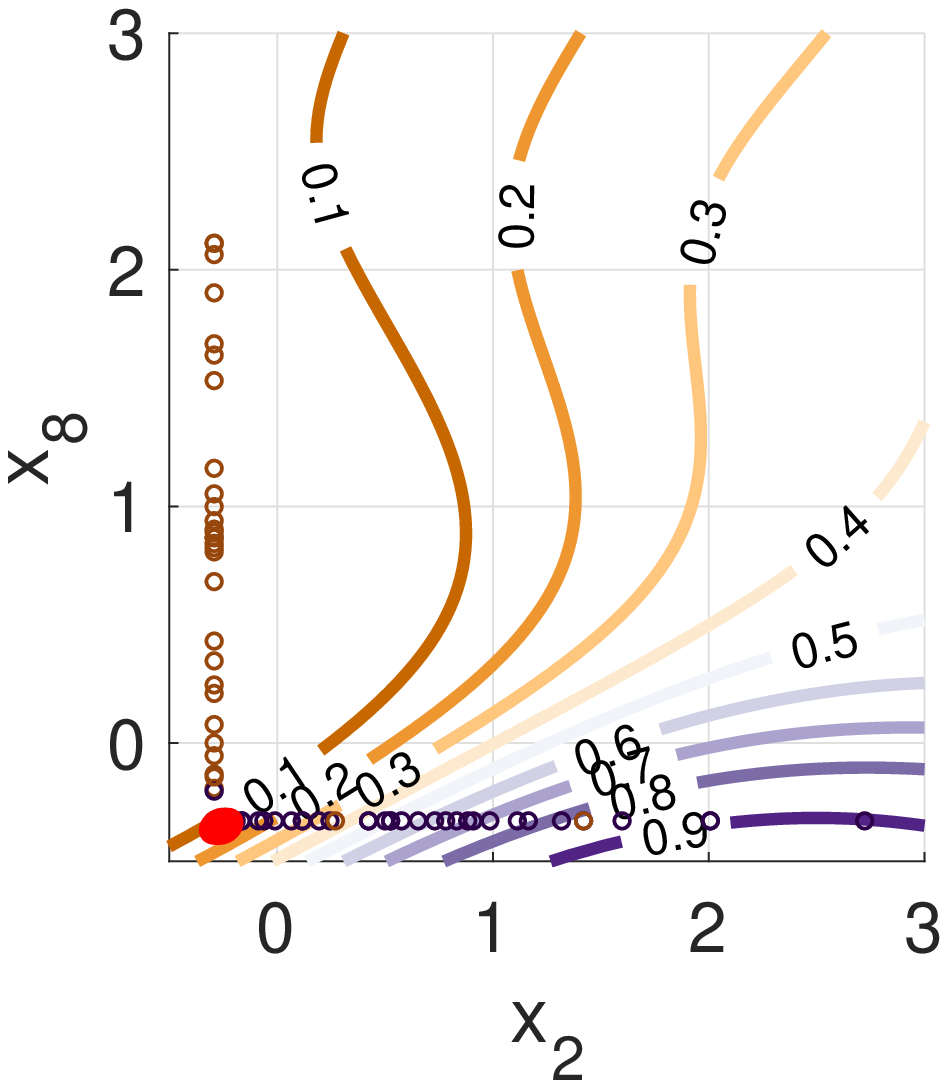}}
	{\hspace*{0.5cm} \includegraphics[width = 0.2\textwidth]{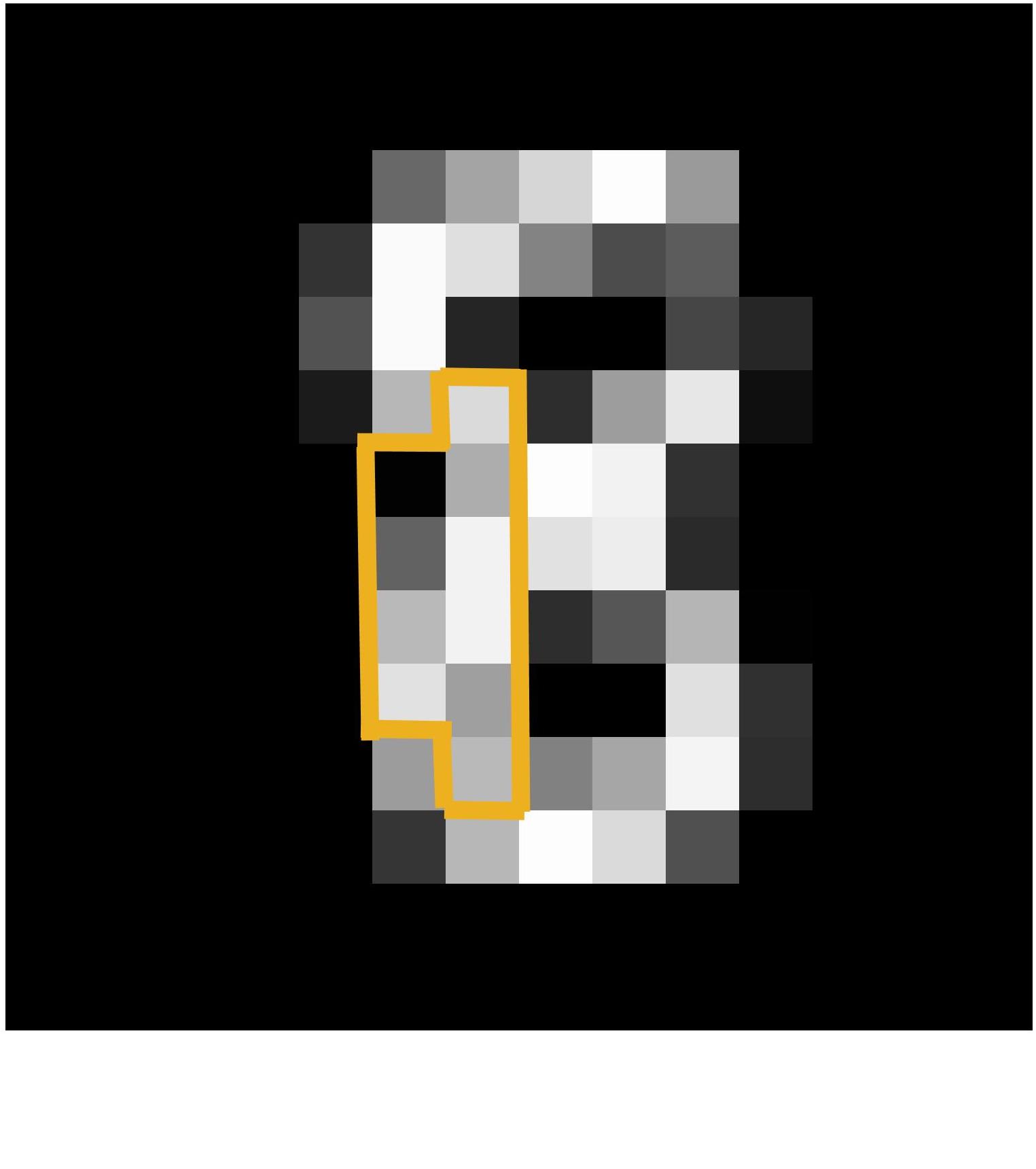}}
	{\hspace*{0.5cm}\includegraphics[width = 0.2\textwidth]{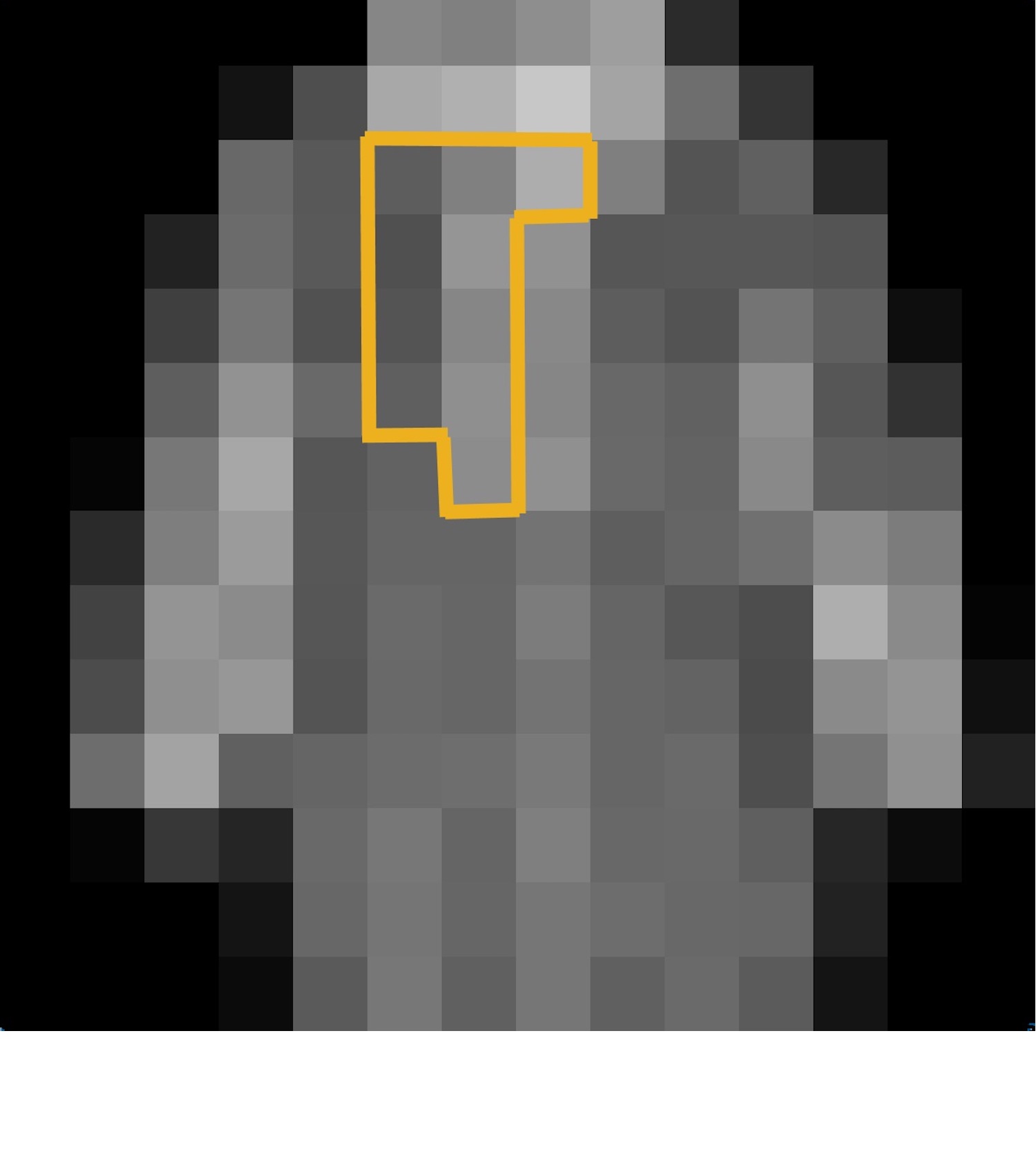}}\\
	{
	\hspace*{0cm} \includegraphics[width = 0.23\textwidth]{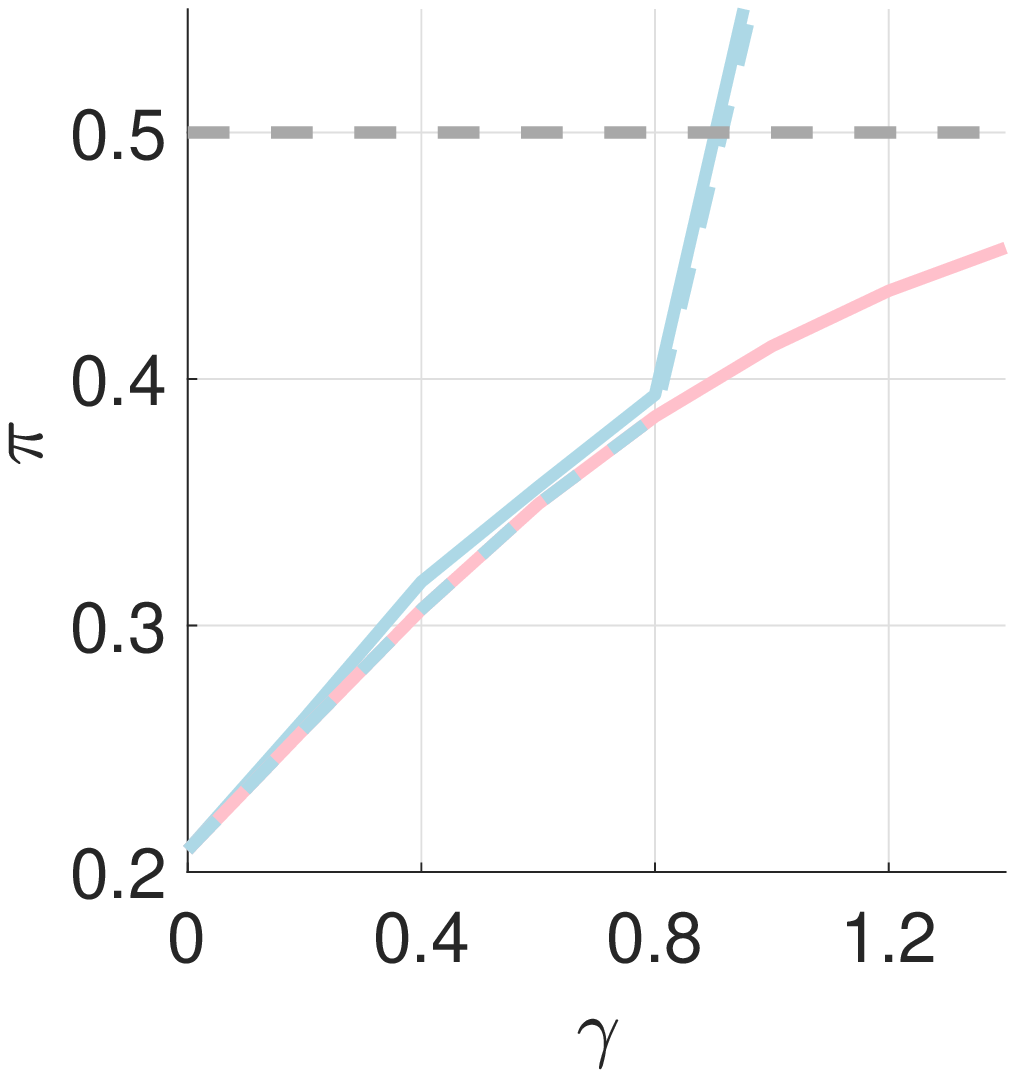}}
	\hspace*{0cm} {\includegraphics[width = 0.23\textwidth]{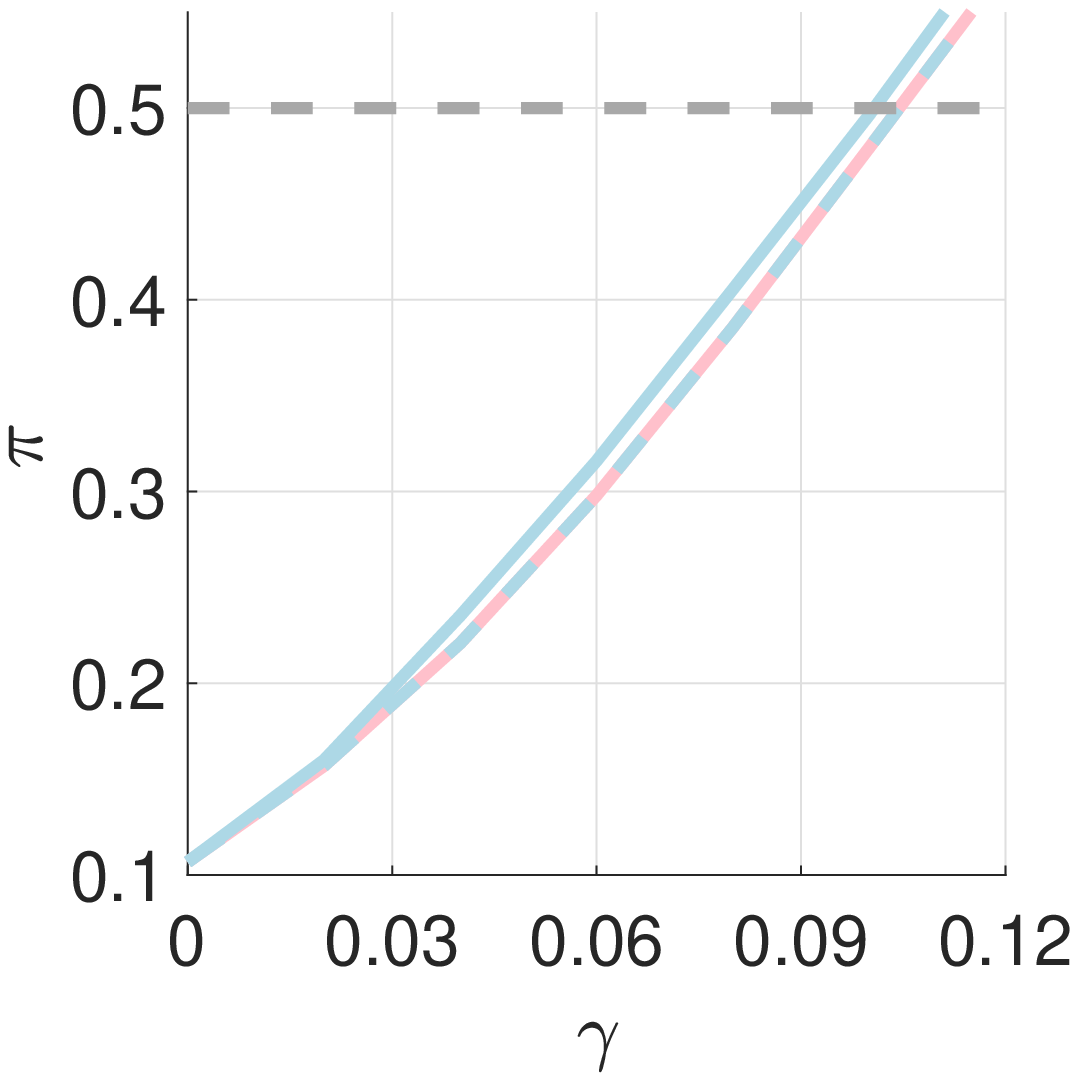}} 
	\hspace*{0cm} {\includegraphics[width = 0.23\textwidth]{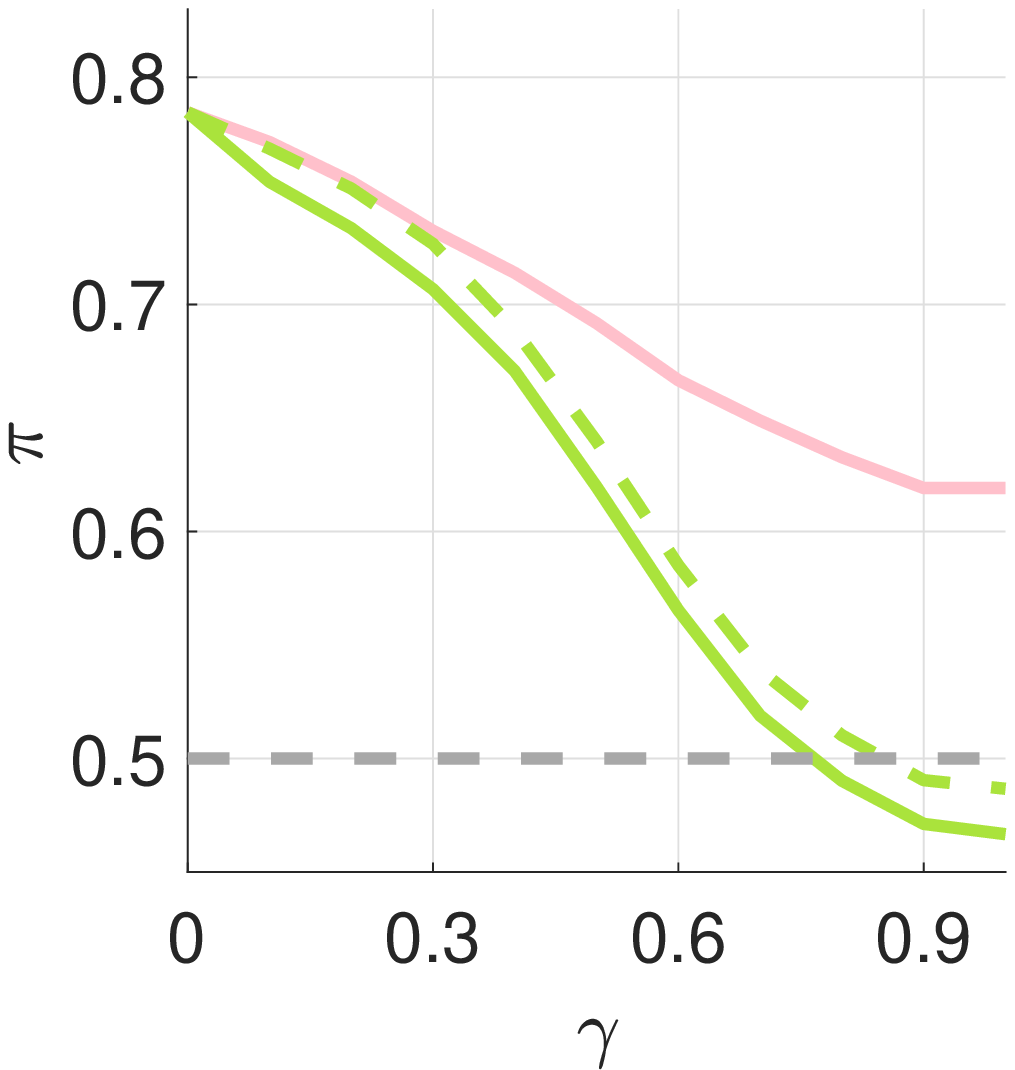}}
	\hspace*{0cm} {\includegraphics[width = 0.23\textwidth]{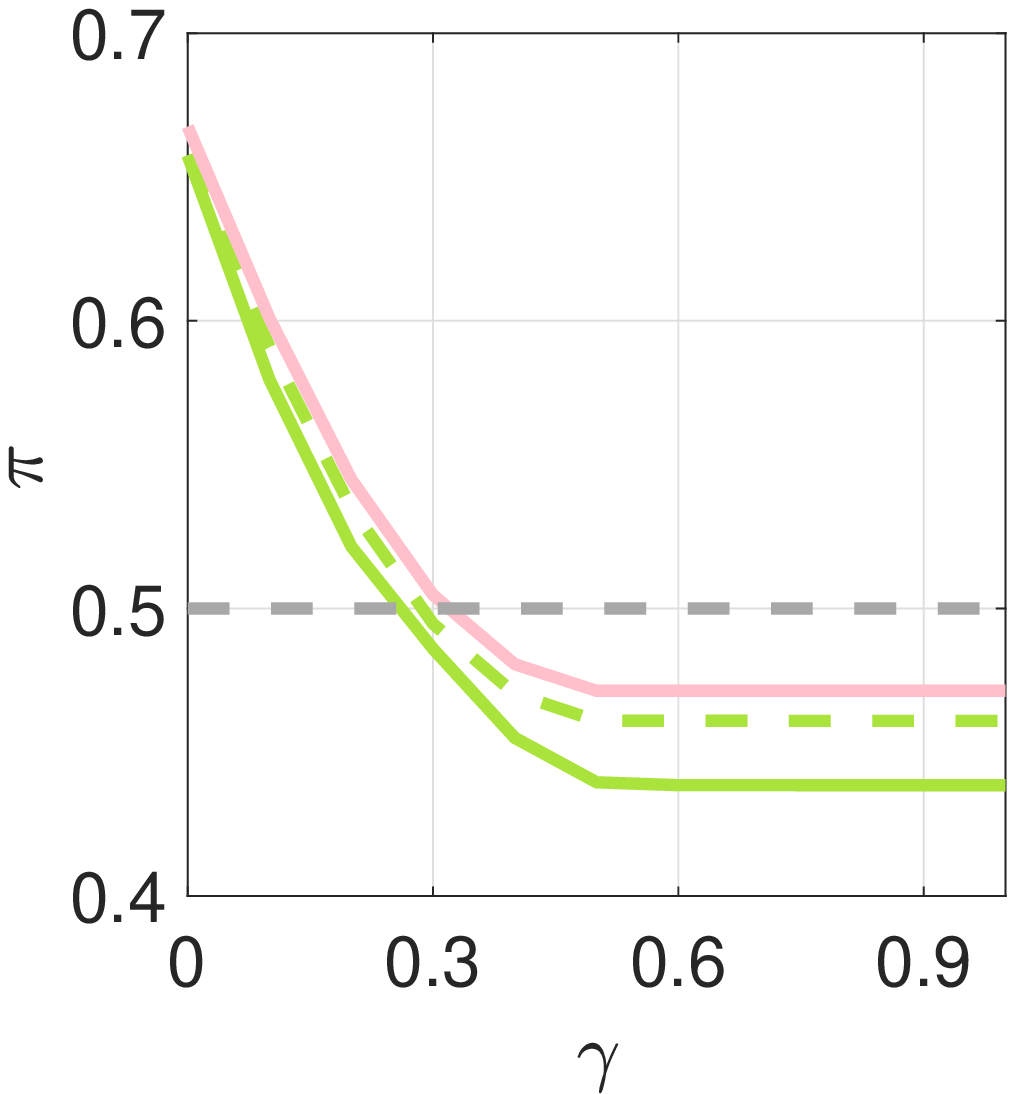}}\\
	{\hspace*{0cm}  \includegraphics[width = 0.7\textwidth]{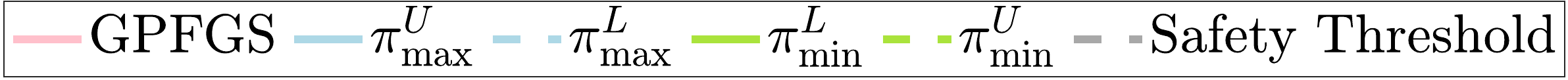}} 
	\caption{\textbf{First row}: Contour plot and test points for Synthetic2D; projected contour plot and test points for 2 dimensions of SPAM (dimensions $2$ and $8$ as selected by $\ell_1$-penalised logistic regression); sample of 8 from MNIST38 along with 10 pixels selected by SIFT; sample of shirt from F-MNIST-TS along with the 10 pixels selected by SIFT. \textbf{Second row}: Safety analysis for the four selected test points. Shown are the upper and lower bounds on $\piSup{T}$ (solid and dashed blue curves), $\piInf{T}$ (solid and dashed green curves), and the GPFGS adversarial attack (pink curve). 
	}
	\label{fig:safety1}
\end{figure}

\subsection{Local Adversarial Safety}
\label{subsec:safety}
We study local adversarial safety for four points selected from the Synthethic2D, SPAM, MNIST38 and F-MNIST-TS datasets and summarise the results in Figure \ref{fig:safety1}. 
To this end, we set $T \subseteq \mathbb{R}^d$ to be a $\ell_\infty$ $\gamma-$ball around the chosen test point 
and iteratively increase $\gamma$ ($x$-axis in the second row plots), checking whether there are adversarial examples in $T$.
Namely, if the point is originally assigned to class 1 (respectively class 2) we check whether the minimum classification probability in $T$ is below the decision boundary threshold, that is, if $\piInf{T} < 0.5$ (resp. $\piSup{T} > 0.5$). 
We compare the values provided by our method (blue solid and dashed line for class 2, green solid and dashed line for class 1) with GPFGS (a gradient-based method for attacking GPs mean prediction \citep{grosse2017wrong}, pink curve in the plot).

Naturally, as $\gamma$ increases, the neighbourhood region $T$ becomes larger, hence the confidence for the initial class can decrease.
Interestingly, while our method succeeds in finding adversarial examples in all cases shown (i.e.\ both the lower and upper bound on the computed quantity cross the decision boundary), the heuristic attack fails to find adversarial examples on the Synthetic2D and MNIST38 data points, and it underestimates the effect of the worst-case perturbation for F-MNIST-TS (i.e., even though the pink line crosses the decision boundary, it remains above the dashed green line).
This occurs because GPFGS builds on linear approximations of the GP predictive posterior, hence failing to find solutions when its non-linearities are significant.
In particular, near the point selected for the Synthetic2D dataset (red dot in the contour plot) the gradient of the GP points away from the decision boundary.
Therefore, no matter what value $\gamma$ takes, GPFGS will not increase above $0.5$ in this case (pink line of the bottom-left plot). 
On the other hand, for the SPAM dataset, the GP model is locally linear around the selected test point (red dot in top right contour plot).
Our observations are further confirmed by the MNIST38 and F-MNIST-TS examples that yield results similar to those for Synthetic2D.

\subsection{Local Adversarial Robustness}
\label{subsec:robustness}
We now evaluate the empirical distribution of the adversarial robustness of the trained GP models.
To this end, we introduce a quantitative measure of robustness analogous to that used by \cite{ruan2018reachability}. More specifically, we consider the difference between the maximum and minimum prediction probability in the region $T$, $\delta =  \pi^U_{\max}(T) - \pi^L_{\min}(T) $, which utilises the computed quantities.
We evaluate the moments of the empirical distribution of values of $\delta$ on $50$ randomly selected test points for each of the four datasets considered.
Note that a smaller value of $\delta$ implies a more robust model. 
Furthermore, we analyse how the GP model robustness is affected by the training procedure used. 
To achieve this, we compare the robustness obtained when using either the Laplace or the Expectation Propagation (EP) \citep{williams2006gaussian} posterior approximations technique, and
investigate the influence of the number of marginal likelihood evaluations (epochs) performed during MLE hyper-parameter optimisation on robustness. 
\begin{figure}[h]
	\centering
	{\includegraphics[width = 0.45\textwidth]{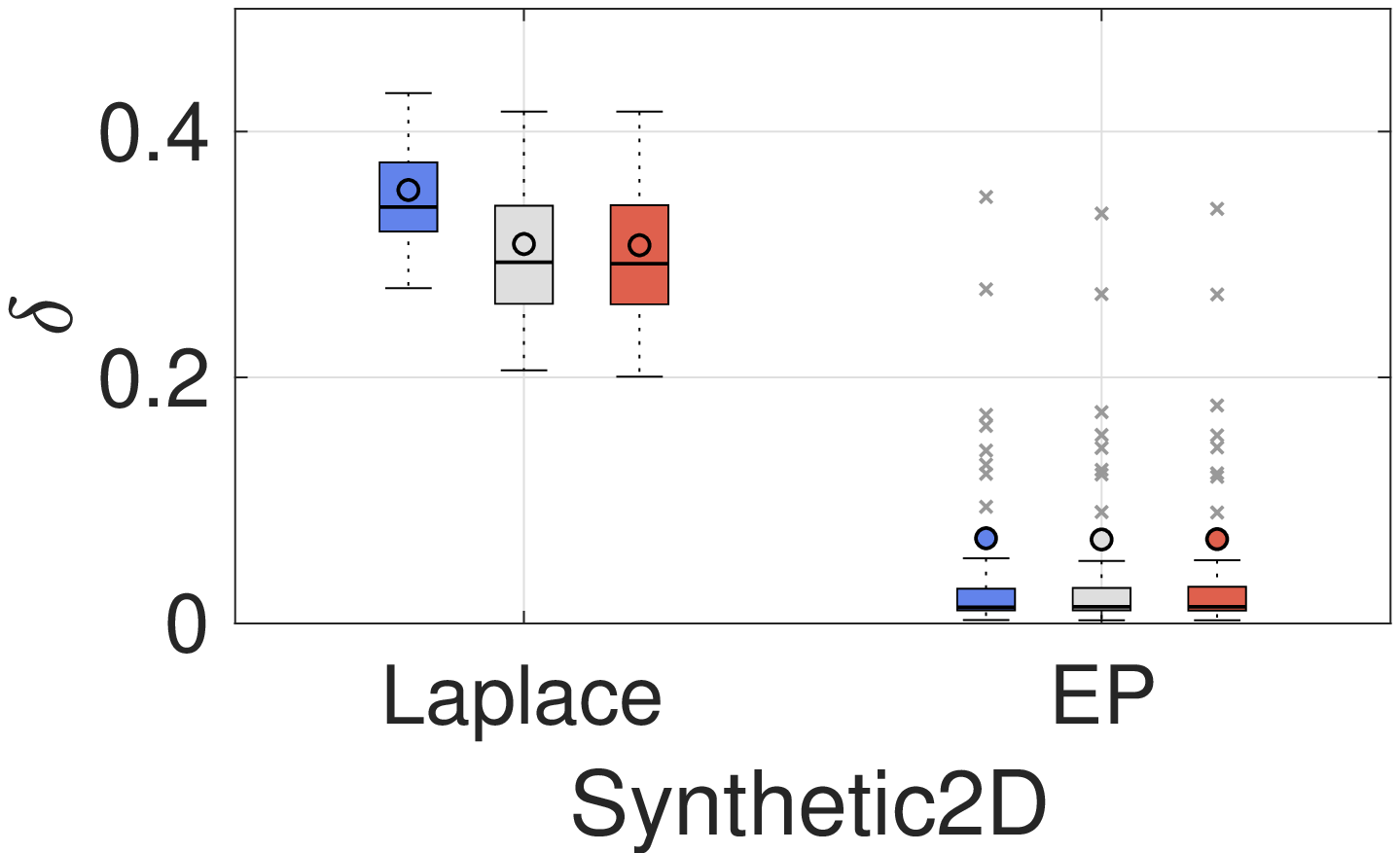}
	\includegraphics[ width = 0.45\textwidth]{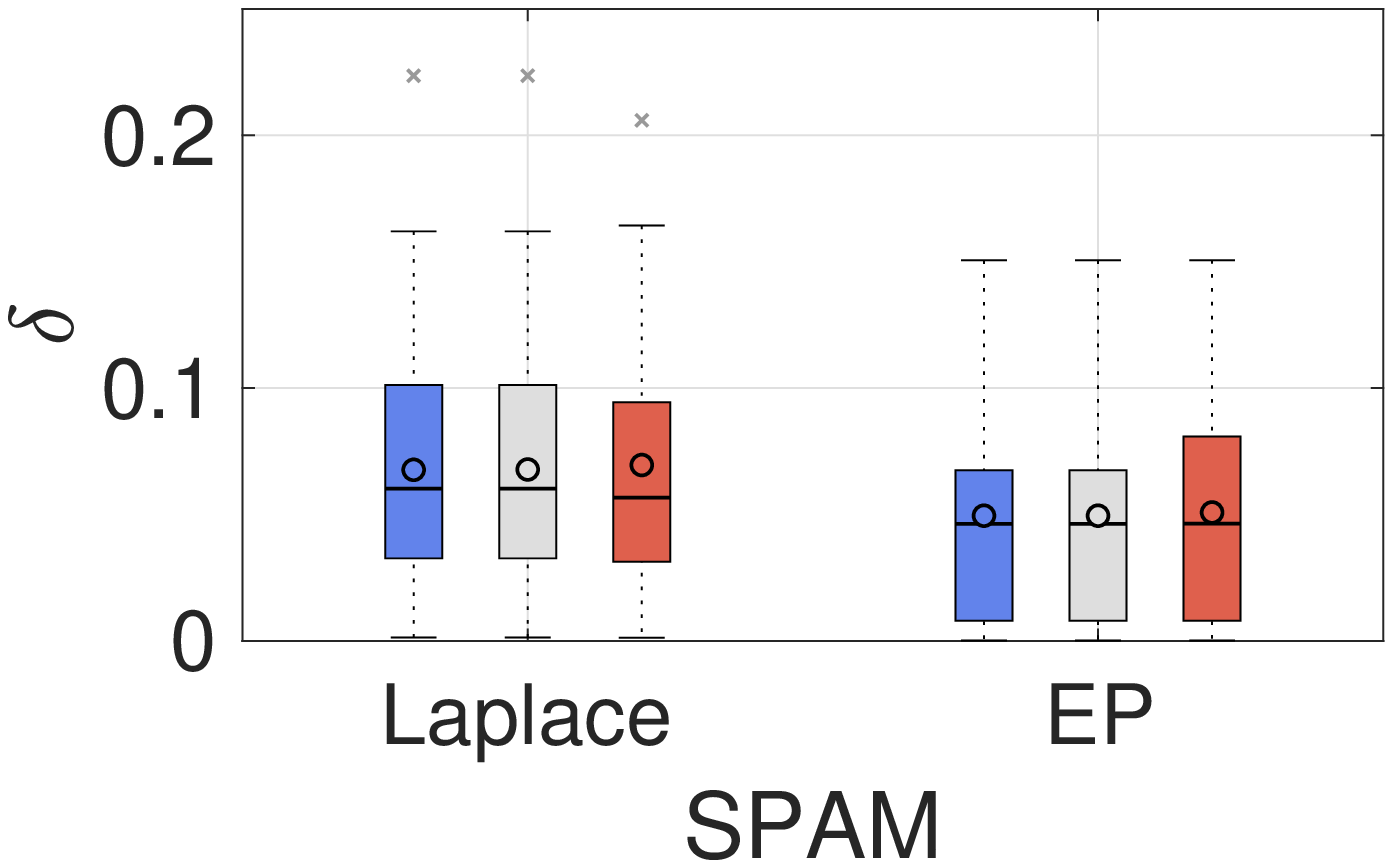} \\
	\includegraphics[ width = 0.45\textwidth]{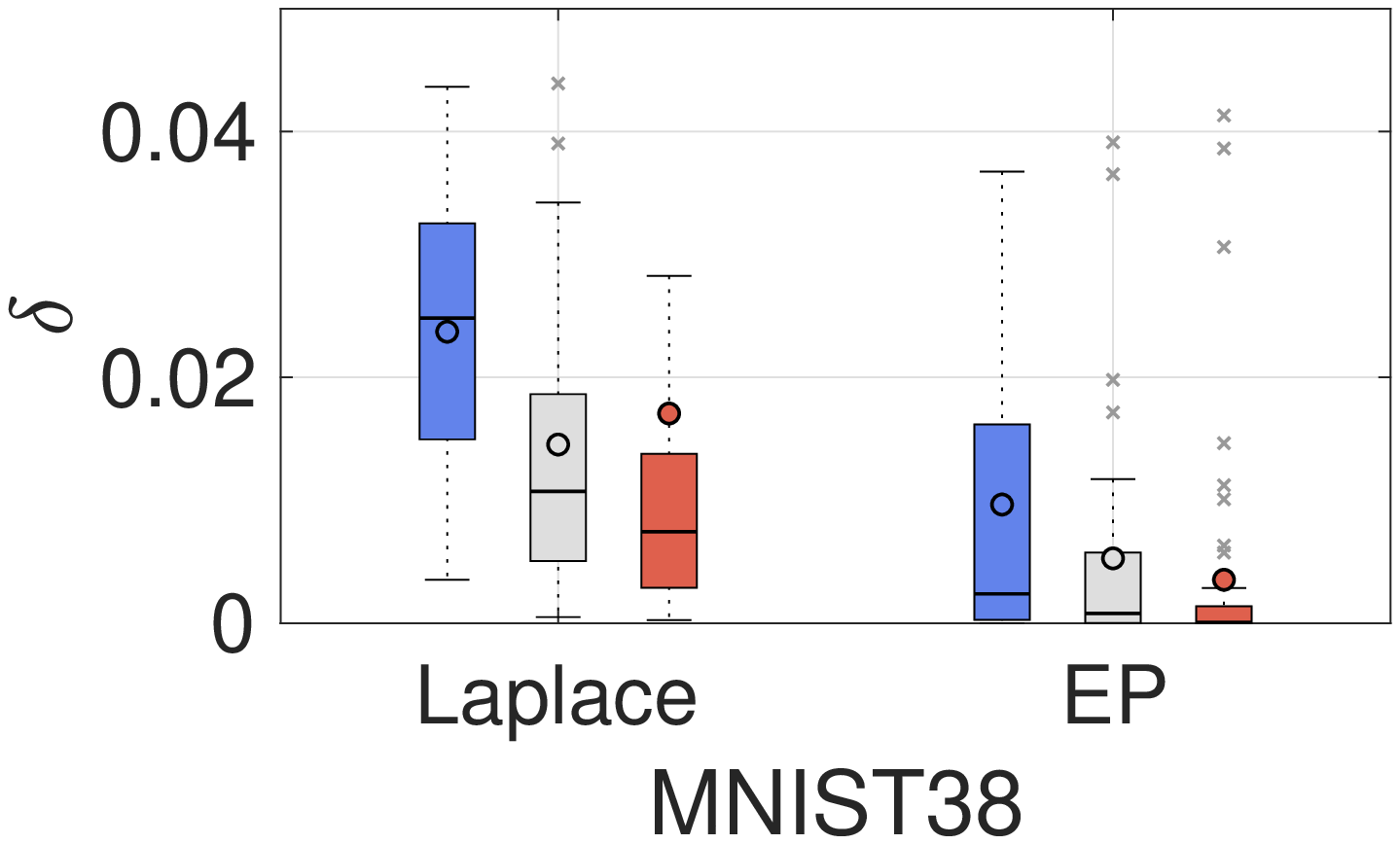}
	\includegraphics[ width = 0.45\textwidth]{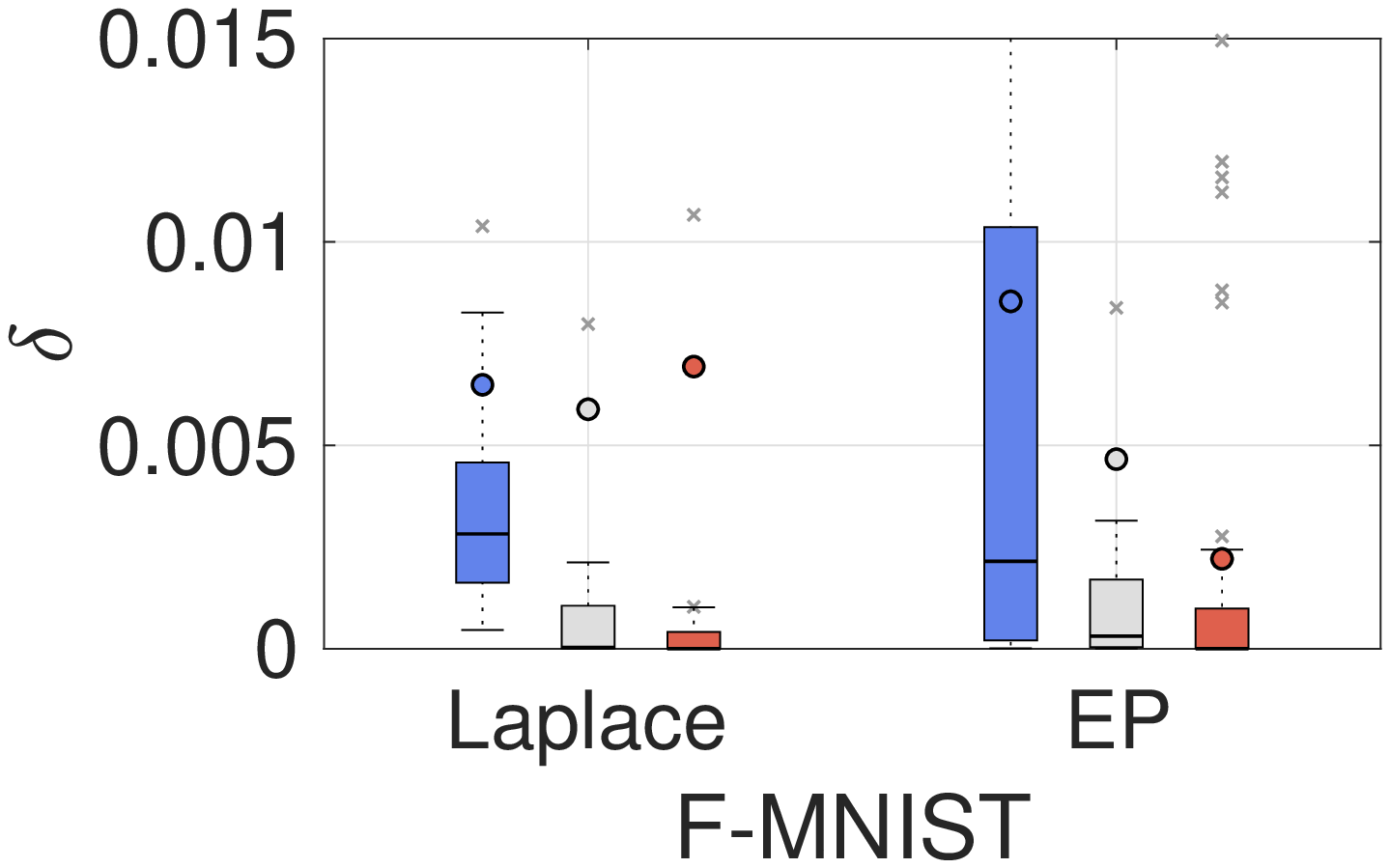}} \\
	 {\hspace*{0.cm} \includegraphics[width = 0.5\textwidth]{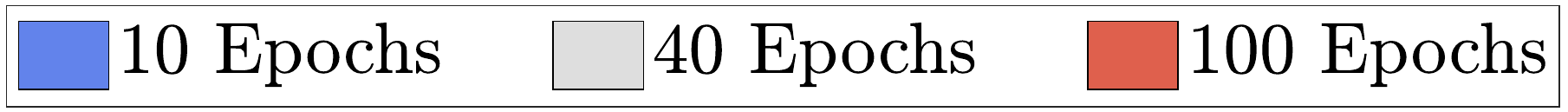}}
	 
	\caption{Boxplots for the distribution of robustness on the four datasets studied, 
	comparing Laplace and EP approximation for $\gamma = 0.1$. 
	}
	\label{fig:Robustness}
\end{figure}

Results for this analysis are depicted in Figure \ref{fig:Robustness}, 
for $10$, $40$ and $100$ hyper-parameter optimisation epochs. 
As explained above, the analyses for the MNIST38 and F-MNIST-TS samples are restricted to the most influential SIFT features only, 
and thus $\delta$ values for them are 
smaller in magnitude than 
for the other two datasets (for which all the input variables are simultaneously changed).
Interestingly, 
this empirical analysis demonstrates that
GPs trained with EP are consistently more robust than those trained using Laplace.
In fact, for both Synthetic2D and MNIST38, EP yields a model about 5 times more robust than Laplace.
For SPAM, the difference in robustness is the least pronounced.
While Laplace approximation works by local approximations, EP calibrates mean and variance estimation by a global approach, which generally results in a more accurate approximation \citep{williams2006gaussian}.
%
These results quantify and confirm for GPs that the posterior distribution is robust to adversarial attacks in the limit, as theorethically analysed by \cite{carbone2020robustness} in the case of over-parameterised Bayesian neural networks, of which GPs are a particular case \citep{de2018gaussian}.
We observe that the values of $\delta$ decrease as the number of training epochs increases, and thus robustness improves with the increase in the number of training epochs.
More training in Bayesian settings may imply better calibration of the latent mean and variance function to the observed data. 

\subsection{Robustness of Sparse Approximations}\label{subsec:sparse}
In Section \ref{subsec:robustness} we have empirically observed that a more refined training procedure may lead to more robust GP models. In standard GP settings it is infeasible to work with large-scale datasets that approximate the exact data manifold, as inference scales with the cube of the number of data (and storage with its square) \citep{bauer2016understanding}.
For large-scale datasets, sparse GPs \citep{bauer2016understanding} are customarily used for approximating the GP posterior distribution.
While sparse GP approximations are usually evaluated in terms of mean and uncertainty calibration, here we consider  adversarial robustness of GP sparse approximation techniques.

As inference equations for sparse approximation can be generally cast in the form of Equations \eqref{eq:mean_posterior}--\eqref{eq:var_posterior}, our methodology can be applied directly, modulo the definition of the matrix $S$, vector $\mathbf{t}$ and the vector of inducing points $\mathbf{u}$ (that is, the set of eventually synthetic points on which training is performed). 
We rely on the EP latent method and compare the results for FIC, DTC, and VAR sparse approximation methods \citep{quinonero2005unifying} on the MNIST38 and F-MNIST-TS datasets.
We vary the number of training points from $250$ to $7500$ and the number of inducing points (selected at random from the training points) from $100$ to $500$.
For each of the resulting GPs we analyse the empirical distribution of $\delta$-robustness on 50 randomly sampled test points with respect to their most relevant features (as detected by SIFT) with $\gamma = 0.15$.

\begin{figure}
    \centering
    \includegraphics[width=0.49\columnwidth]{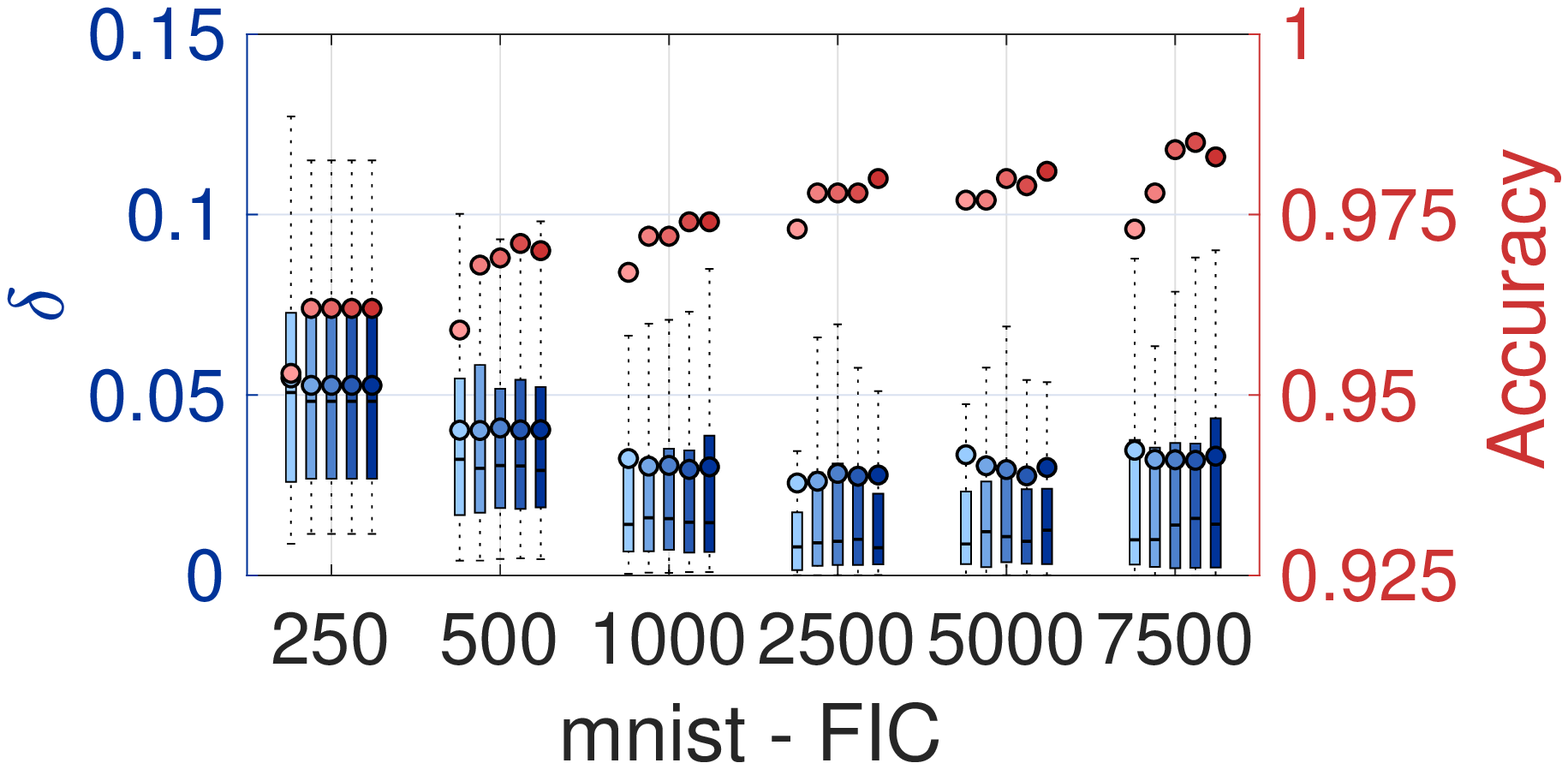}
     \includegraphics[width=0.49\columnwidth]{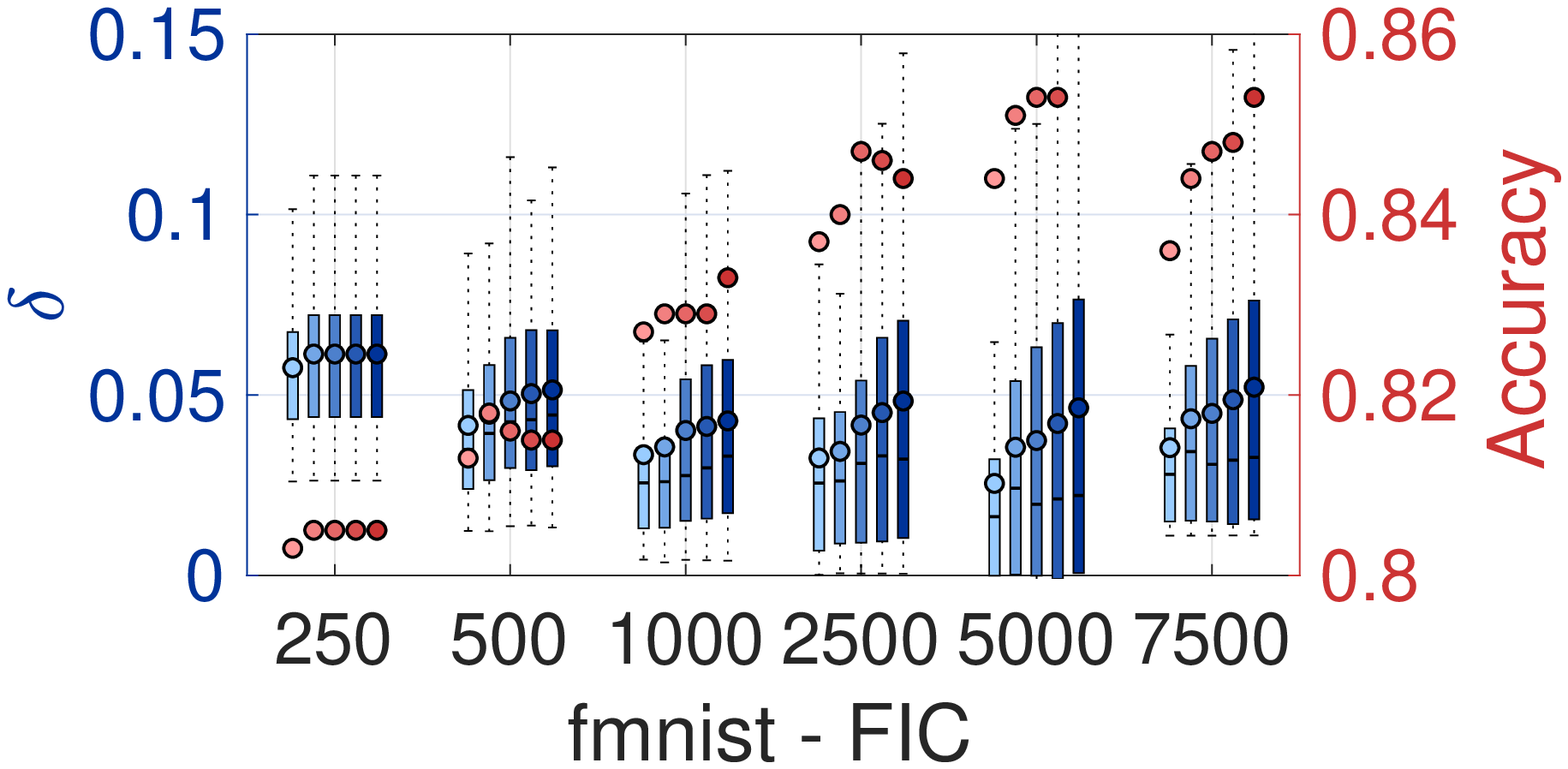} \\
     \includegraphics[width=0.49\columnwidth]{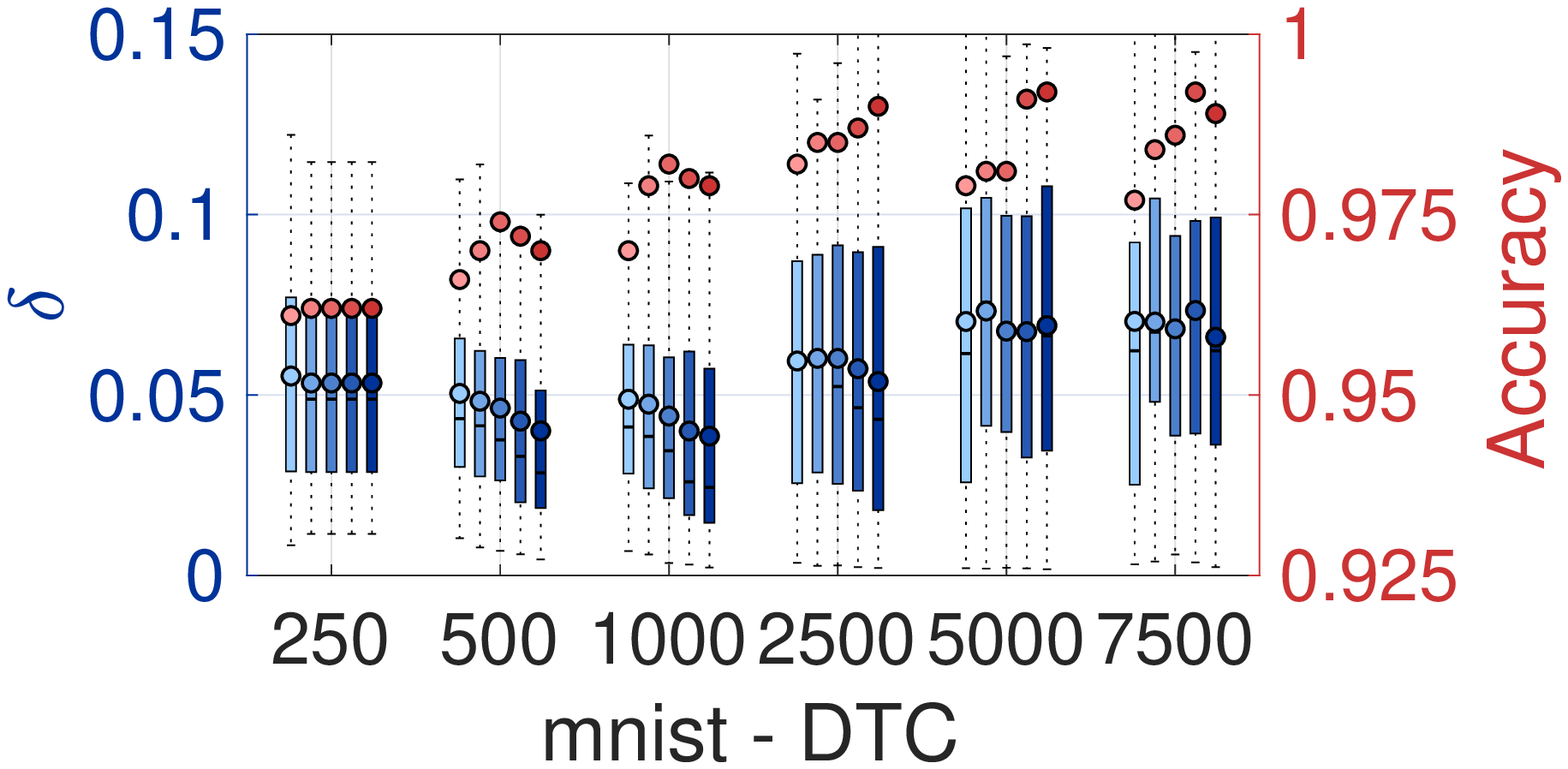}
     \includegraphics[width=0.49\columnwidth]{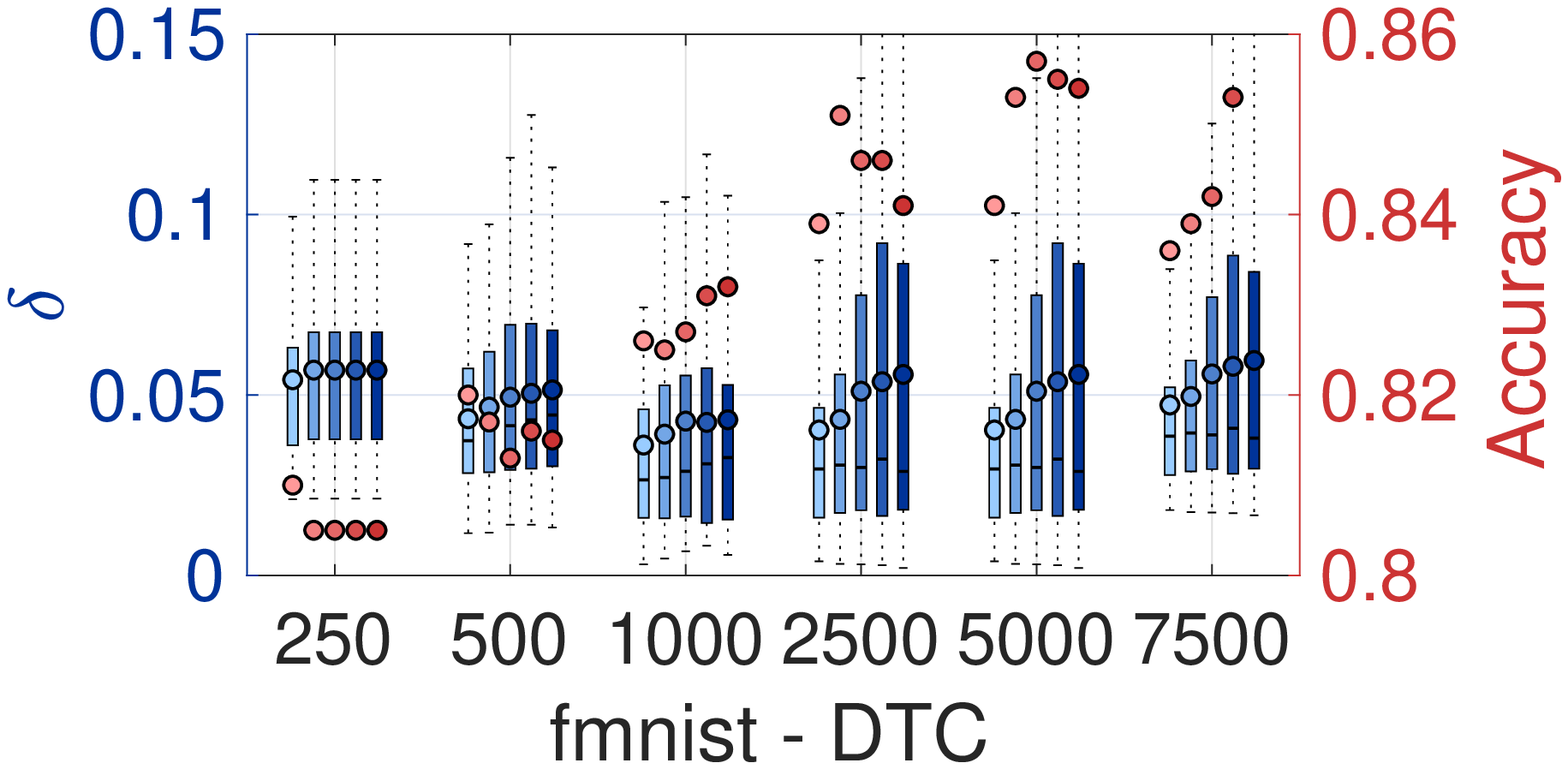} \\
     \includegraphics[width=0.49\columnwidth]{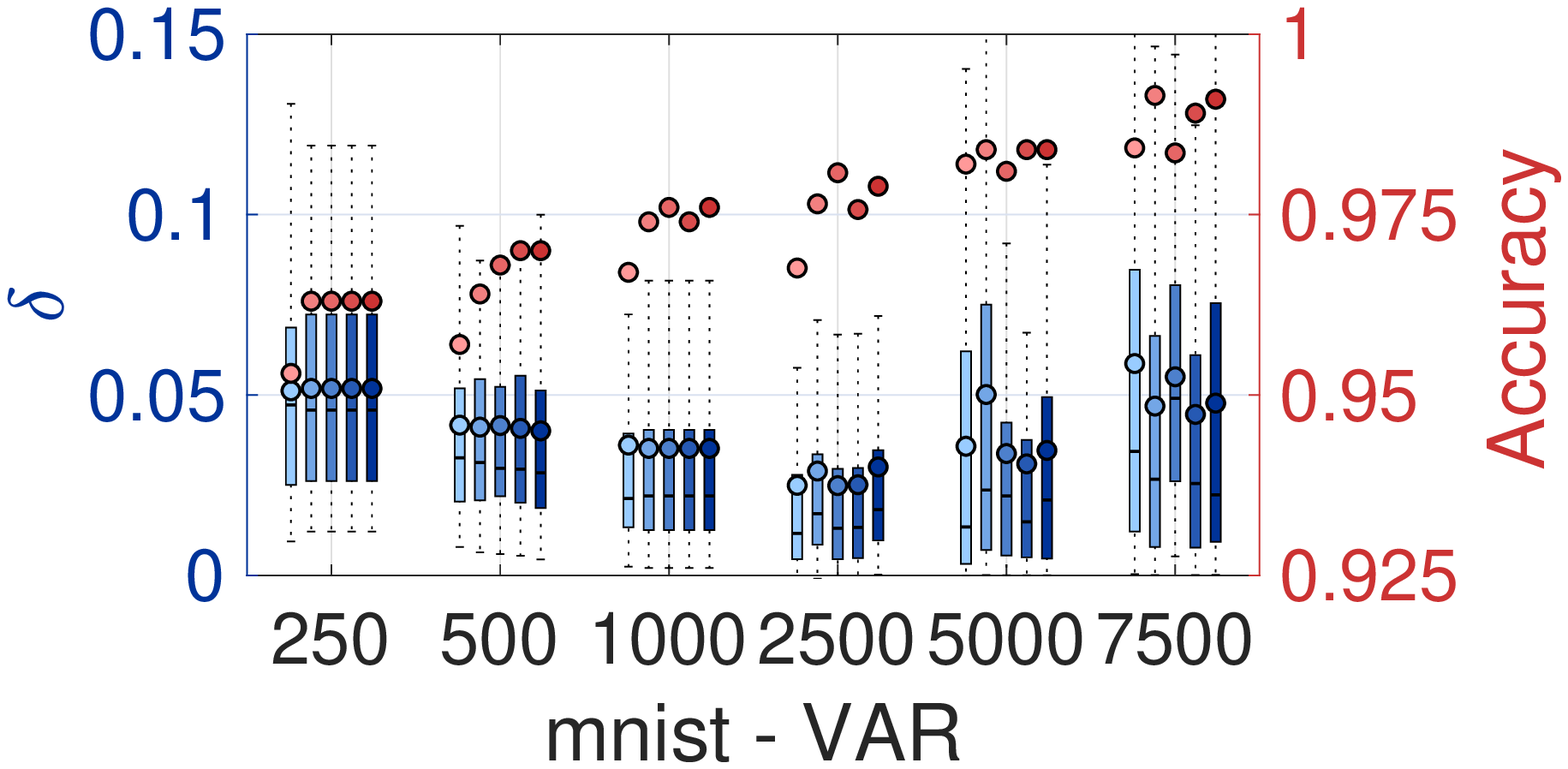}
     \includegraphics[width=0.49\columnwidth]{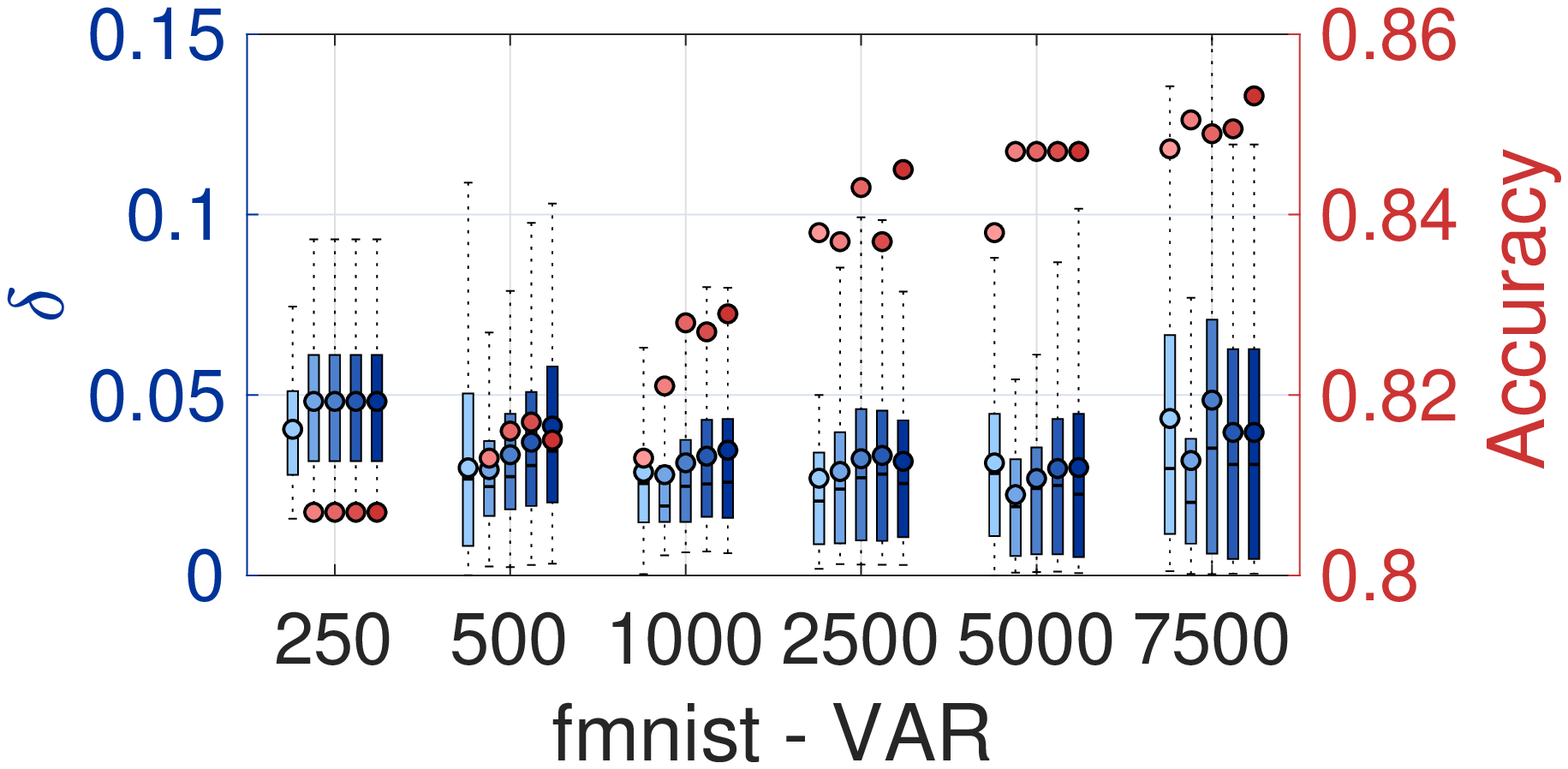}
    \caption{Empirical distribution of $\delta$-robustness for $\gamma = 0.15$. \textbf{First Row:} FIC sparse approximation. \textbf{Second Row:} DTC sparse approximation. \textbf{Third Row:} VAR sparse approximation.}
    \label{fig:sparse_gps}
\end{figure}

Results for this analysis are plotted in Figure \ref{fig:sparse_gps}, where boxplots are grouped according to the number of training points, with each boxplot in each group representing an increase of 100 inducing points 
(starting from 100).
The test set accuracy of each GP, as estimated over 1000 test samples, is plotted in the same figure on a separate $y$-axis (red dots).
In agreement with the literature on sparse GPs \citep{bauer2016understanding}, we observe that an increasing number of training and/or inducing points generally leads to more accurate models.
Among the two datasets analysed here, this aspect is more pronounced on F-MNIST ($\approx 6 \%$ increase), which poses a more complex classification task than MNIST ($\approx 2\%$ increase), so that the GP further benefits from more information from data.

The robustness trends instead depend on the approximation techniques used.
For FIC and VAR, we generally obtain that more training input points corresponds to an increase in the model robustness (i.e., lower value of $\delta$). More specifically, sparse GPs successfully take into account the information from a larger pool of training samples in refining its posterior estimation.
Unfortunately, for the VAR models the EP computations become numerically unstable after $2500$ training samples and we have to increase the data jitter (which results in a widening of the boxplot and reduced robustness).
For DTC, instead, we observe that the robustness slightly worsens in the case of MNIST and remains stable for F-MNIST.
Finally, we remark that the number of inducing points has little effect on the overall robustness when compared to the number of training points used.

\subsection{Interpretability Analysis}
\label{subsec:interpretability}
%
%
Adversarial robustness and model prediction interpretability are closely linked together \citep{darwiche2020three}.
To demonstrate this, we can utilise the bounds we compute on $\piInf{T}$ and $\piSup{T}$ to formulate an interpretability metric similar to that defined for linear classifiers in \citep{ribeiro2016should} and implemented in a black-box tool called LIME.
In particular, we consider a test point $x^*$ and the one-sided input box $T^i_{\gamma}(x^*) = [x^*, x^* + \gamma e_i]$ (where $e_i$ denotes the vector of $0$s except for $1$ at dimension $i$).
We compute how much the maximum and minimum values can change over the one-sided intervals in both directions:
\begin{align*}
    \mathbf{\Delta}^i_{\gamma}(x^*) = \left(\piSup{T^i_{\gamma}(x^*)} -\piSup{T^i_{-\gamma}(x^*)}\right)
    + \left(\piInf{T^i_{\gamma}(x^*)}-\piInf{T^i_{-\gamma}(x^*)}\right). 
\end{align*}
If increasing the value of dimension $i$ makes the model favour assigning lower class probabilities, we would expect this value to be negative and vice versa. 
Intuitively, this provides a non-linear generalisation of numerical gradient estimation, which resembles exactly the metric used by \citet{ribeiro2016should} as $\gamma$ tends to $0$ or if the model considered is linear. 
Global estimation measures can be computed by estimating the expected value of $\mathbf{\Delta}^i_{\gamma}(x^*)$ with $x^*$ sampled from a test set.
However, since our method relies on the analytic form of the inference equations of GPs (rather than being model-agnostic, which LIME is), we are able to formally bound these quantities, which allows as to provide guarantees over interpretability results.
Next, we  first graphically demonstrate why linear approximation can be misleading for global interpretability analysis for the 2D-synthetic and SPAM datasets, and then show how we can rely on formal quantification of  interpretability to investigate the adversarial vulnerability of a GP model around specific test points.

\paragraph{Global Interpretability Analysis for 2D-Synthethic and SPAM}
\begin{figure}[h]
	\centering
	{ \includegraphics[width = 0.48\textwidth]{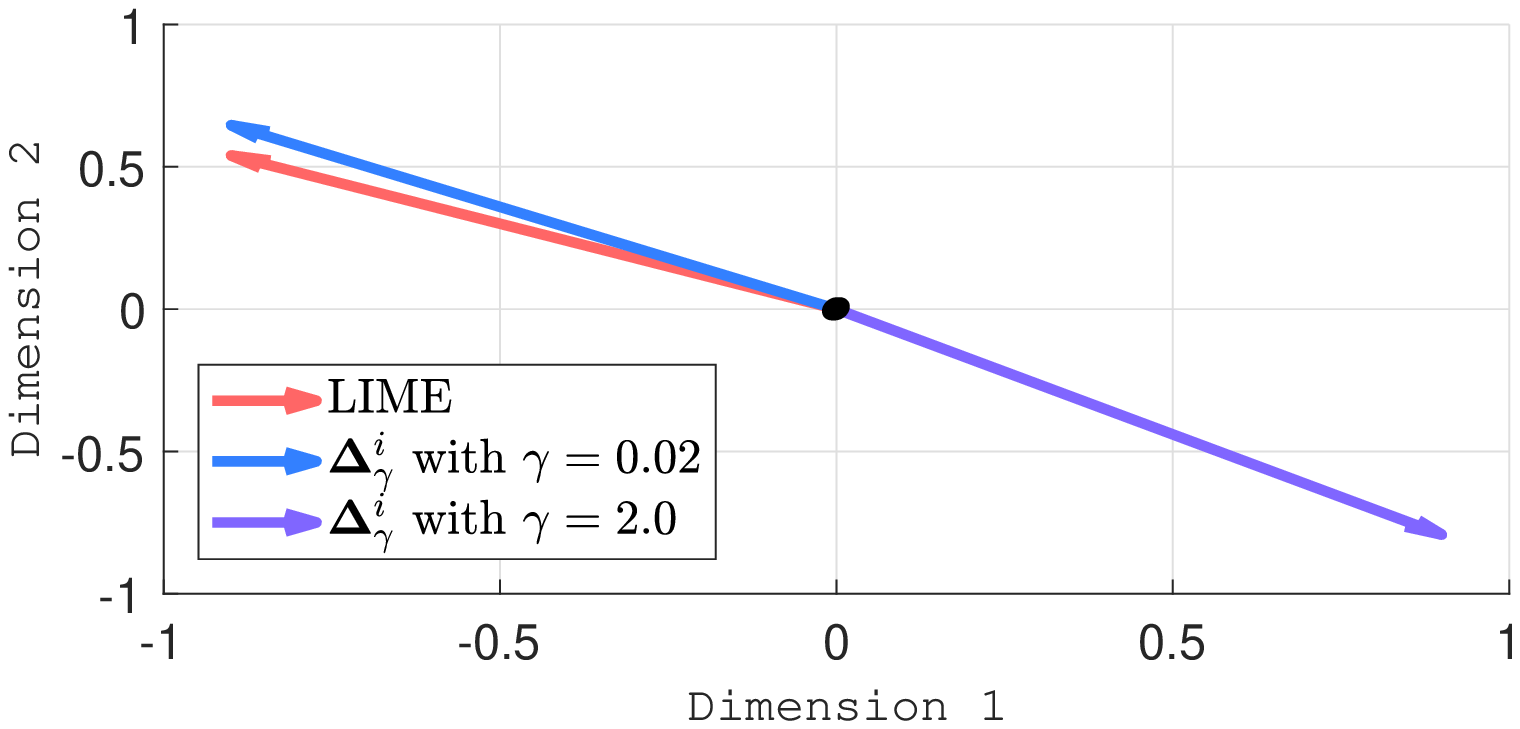}}
	{  \includegraphics[width = 0.48\textwidth]{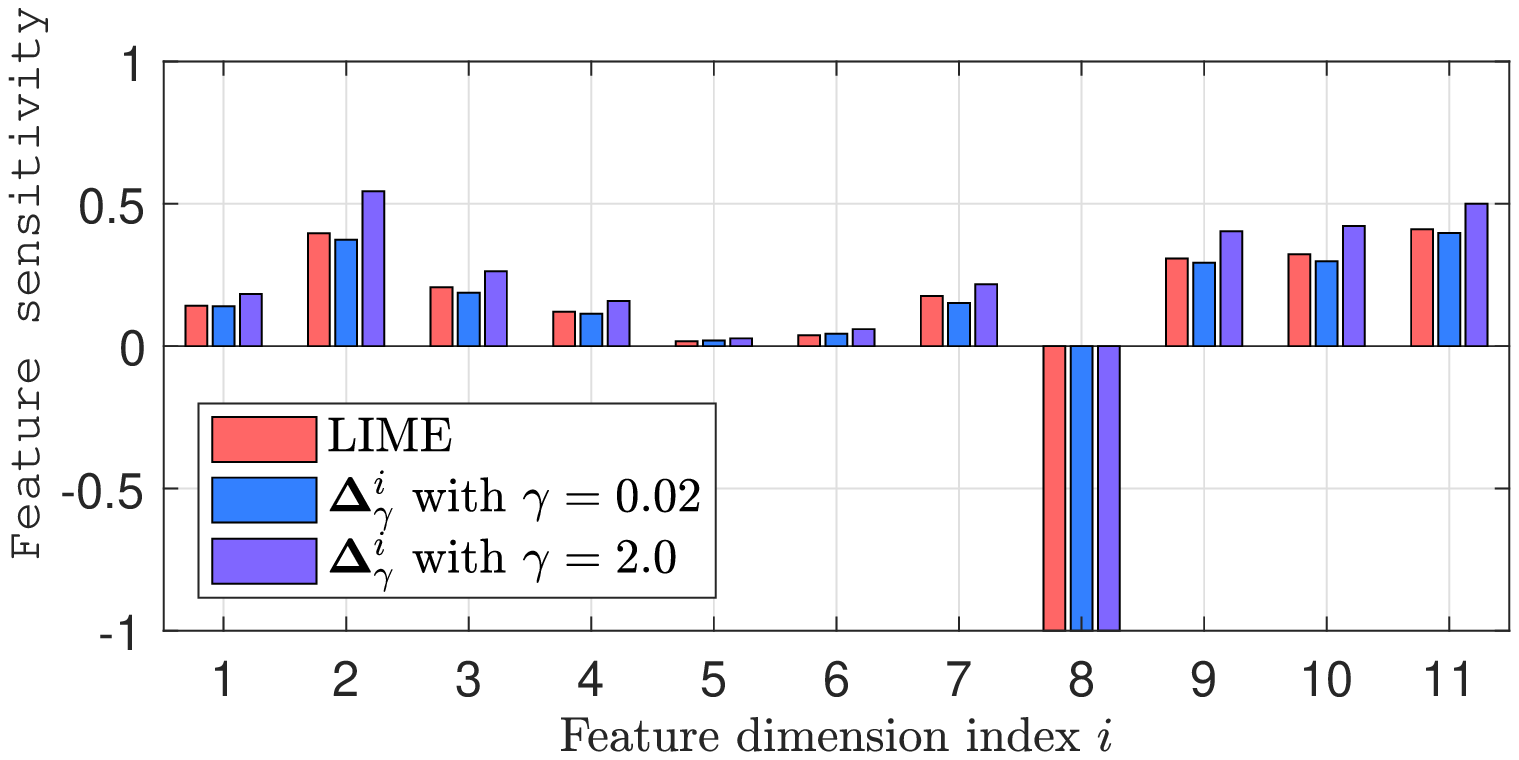}}
	\caption{Global interpretability, $\mathbf{\Delta}^i_{\gamma}$, as analysed by LIME and our method. \textbf{Left:} Results for the Synthetic2D dataset. \textbf{Right:} Results for the SPAM dataset.} 
	\label{fig:interpret_appendix}
\end{figure}
We perform global interpretability analysis on GP models trained on the Synthetic2D and SPAM datasets, estimating the expected value of $\mathbf{\Delta}^i_{\gamma}$ with $50$ random test points.
The results are shown in Figure \ref{fig:interpret_appendix}.
For Synthetic2D (top row), LIME suggests that a higher probability of belonging to class 1 (depicted as the direction of the arrow in the plot) corresponds to lower values along dimension 1 and higher values along dimension 2.
As can be seen in the corresponding contour plot in Figure \ref{fig:safety1} (top left), the exact opposite is true, however.
LIME, as it is built on linearity approximations, fails to take into account the global behaviour of the GP.
When using a small value of $\gamma$ our approach obtains similar results to LIME.
However, with $\gamma = 2.0$ the global relationship between input and output values is correctly captured.
For SPAM, on the other hand (Figure \ref{fig:interpret_appendix}, bottom), due to the linearity of the dataset and the GP, a local analysis correctly reflects the global picture.


\paragraph{Interpretability for MNIST358 and F-MNIST-TSP predictions}
\begin{figure}[h]
  	\centering
	\includegraphics[clip = on, trim = 10mm 10mm 10mm 10mm ,width = 0.16\textwidth]{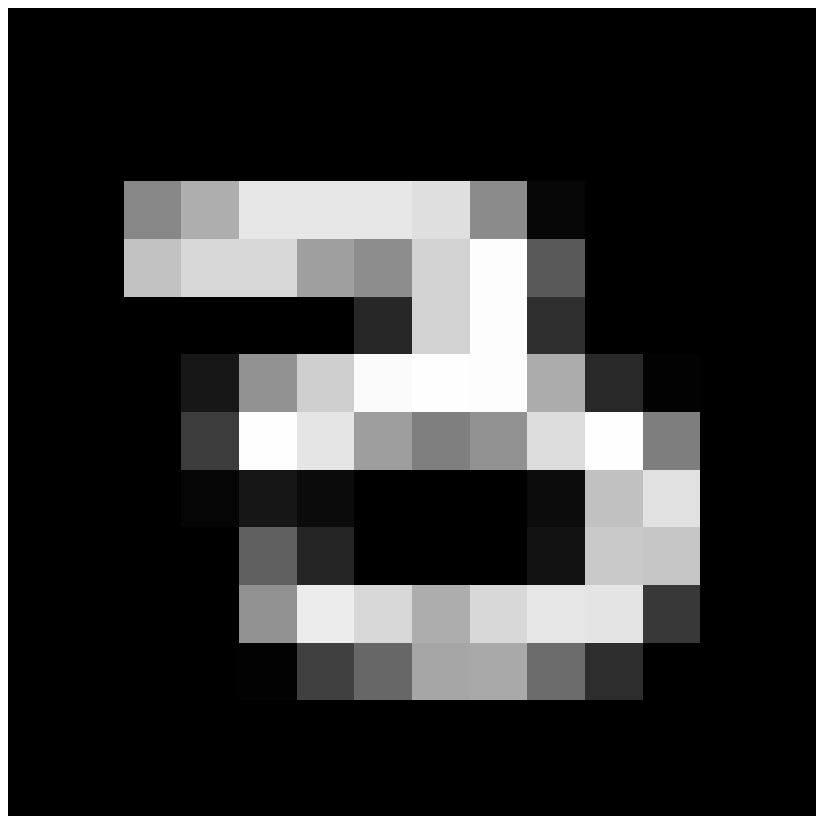}
	\includegraphics[clip = on, trim = 10mm 10mm 10mm 10mm ,width = 0.16\textwidth]{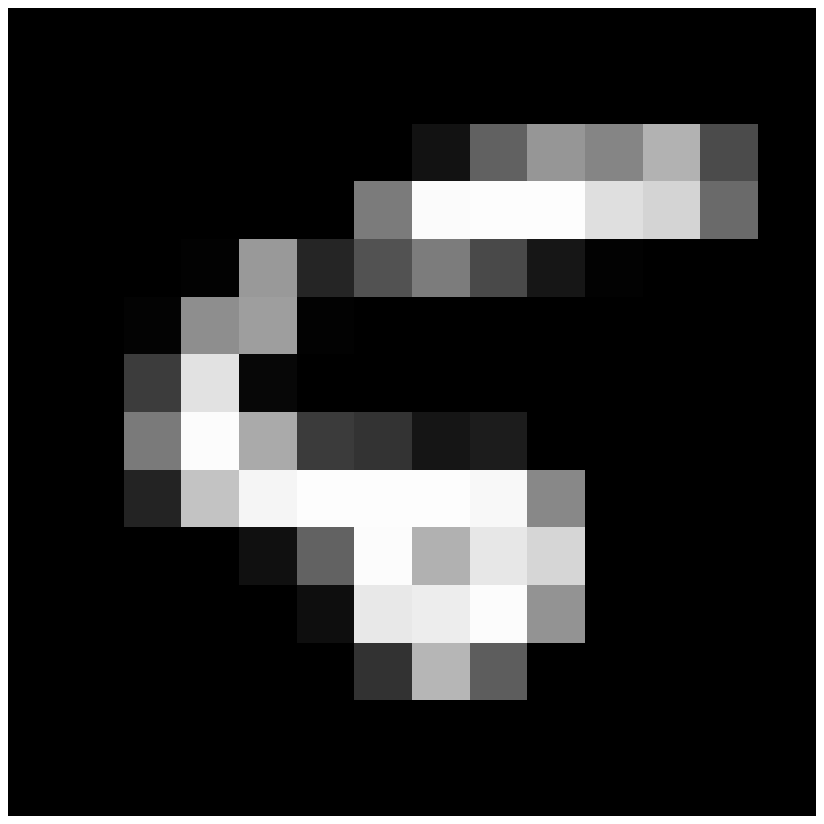}
	\includegraphics[clip = on, trim = 10mm 10mm 10mm 10mm ,width = 0.16\textwidth]{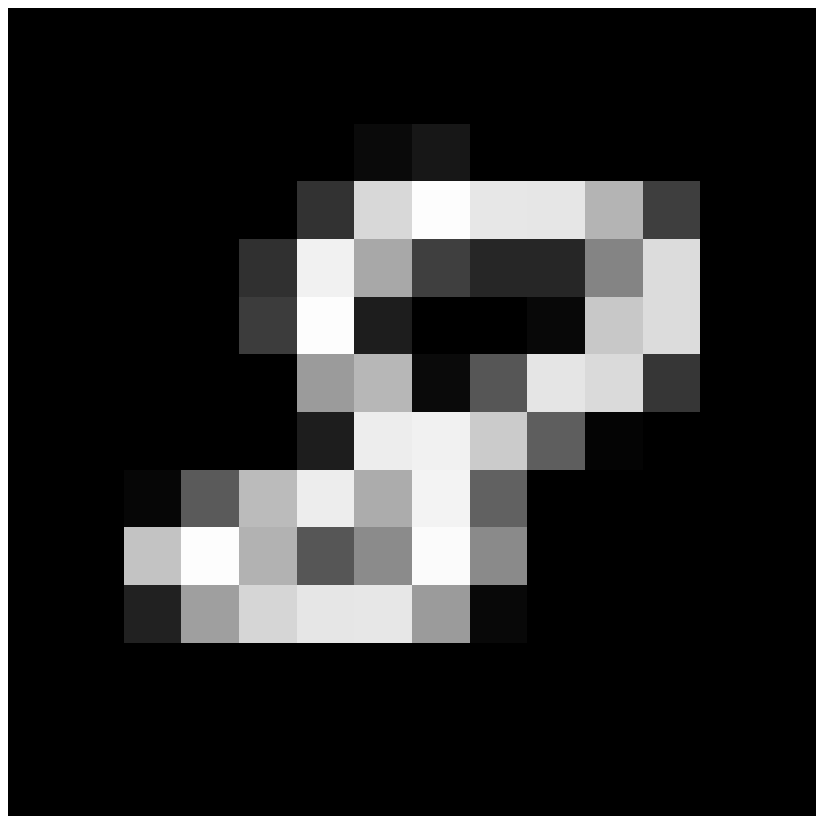}
	\includegraphics[clip = on, trim = 10mm 10mm 10mm 10mm ,width = 0.16\textwidth]{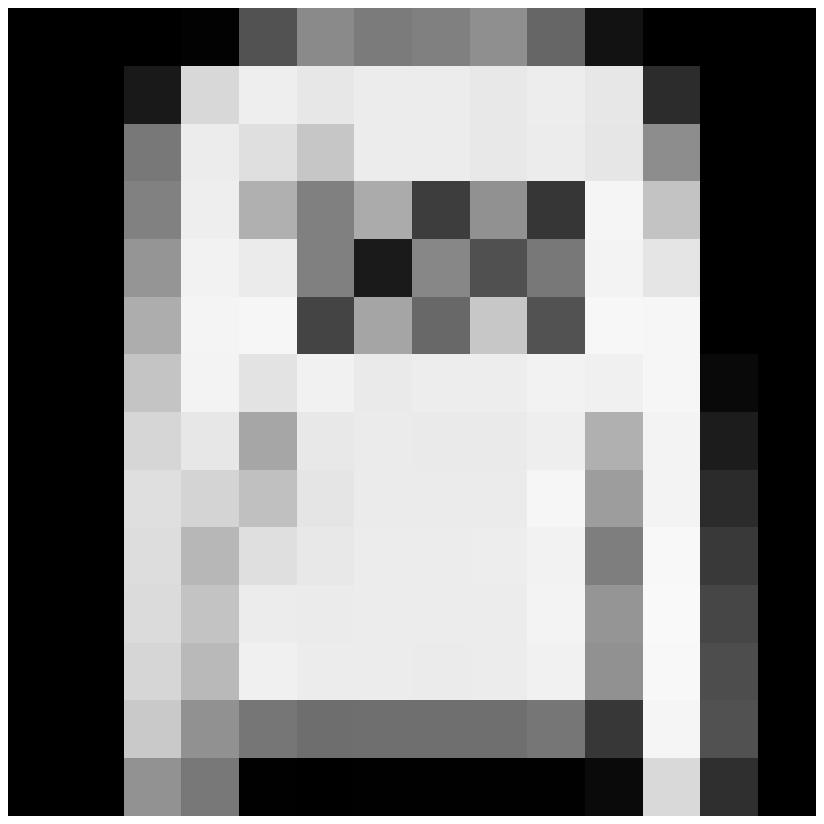}
	\includegraphics[clip = on, trim = 10mm 10mm 10mm 10mm ,width = 0.16\textwidth]{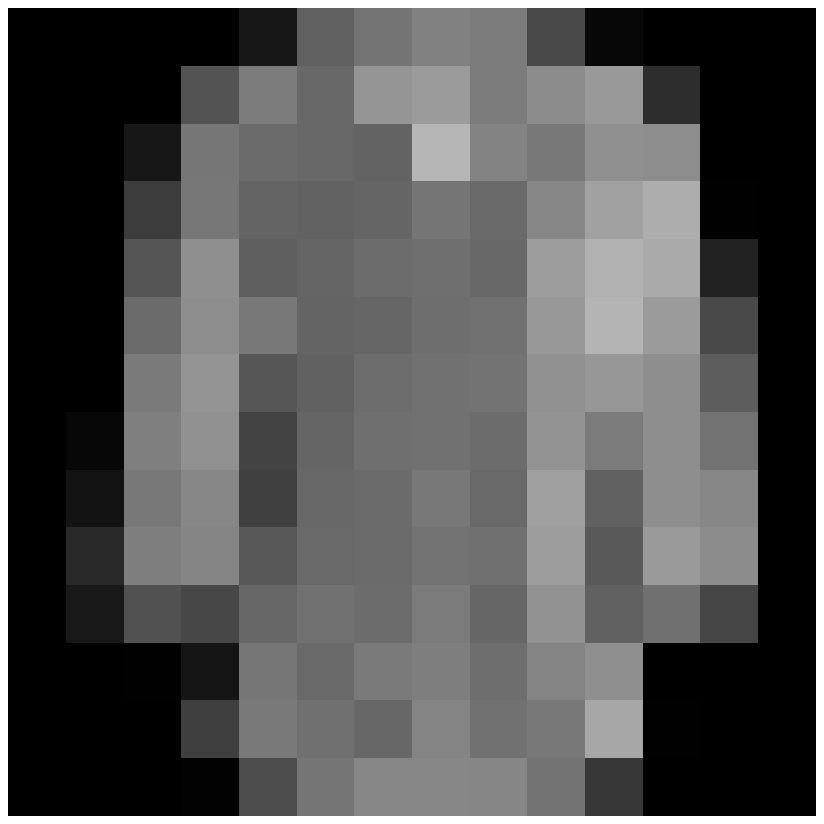}
	\includegraphics[clip = on, trim = 10mm 10mm 10mm 10mm ,width = 0.16\textwidth]{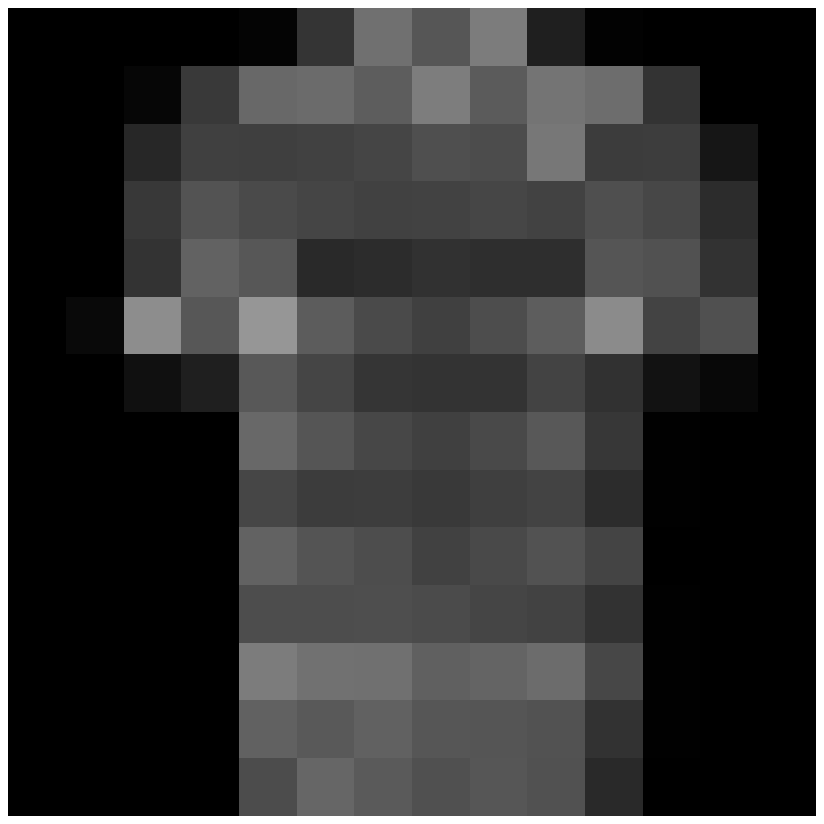}\\ 
	
	\includegraphics[clip = on, trim = 10mm 10mm 10mm 10mm ,width = 0.16\textwidth]{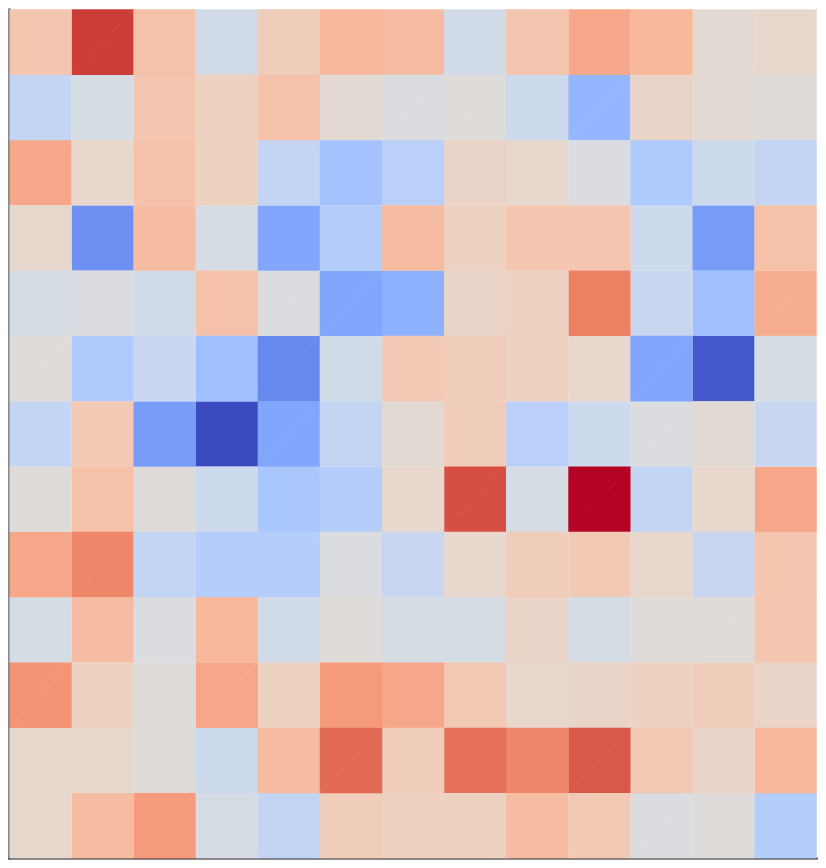}
	\includegraphics[clip = on, trim = 10mm 10mm 10mm 10mm ,width = 0.16\textwidth]{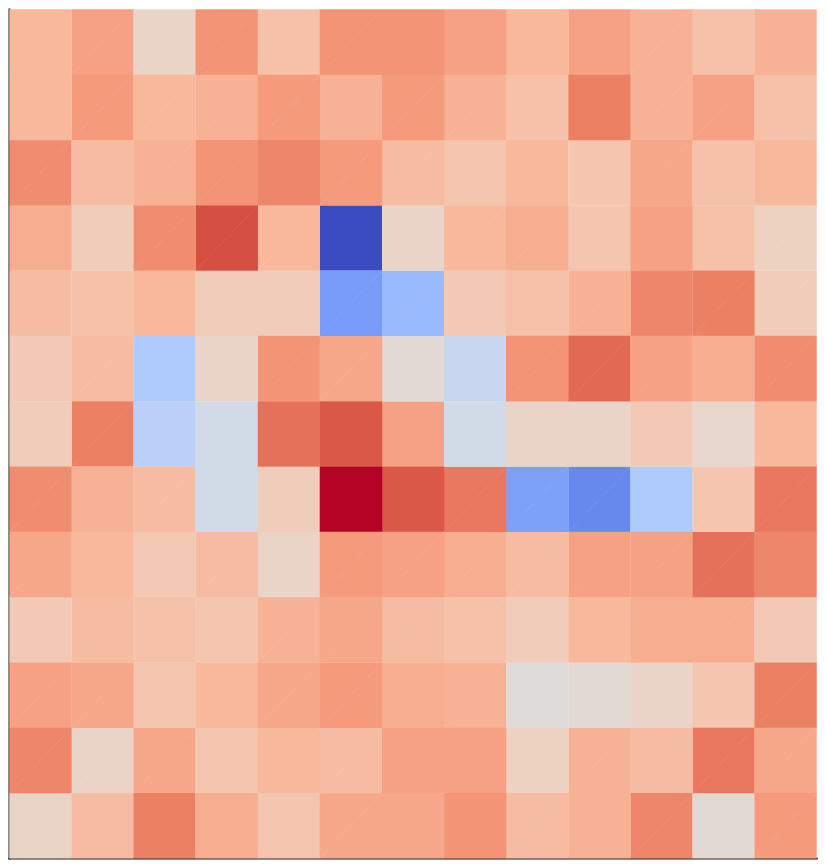}
	\includegraphics[clip = on, trim = 10mm 10mm 10mm 10mm ,width = 0.16\textwidth]{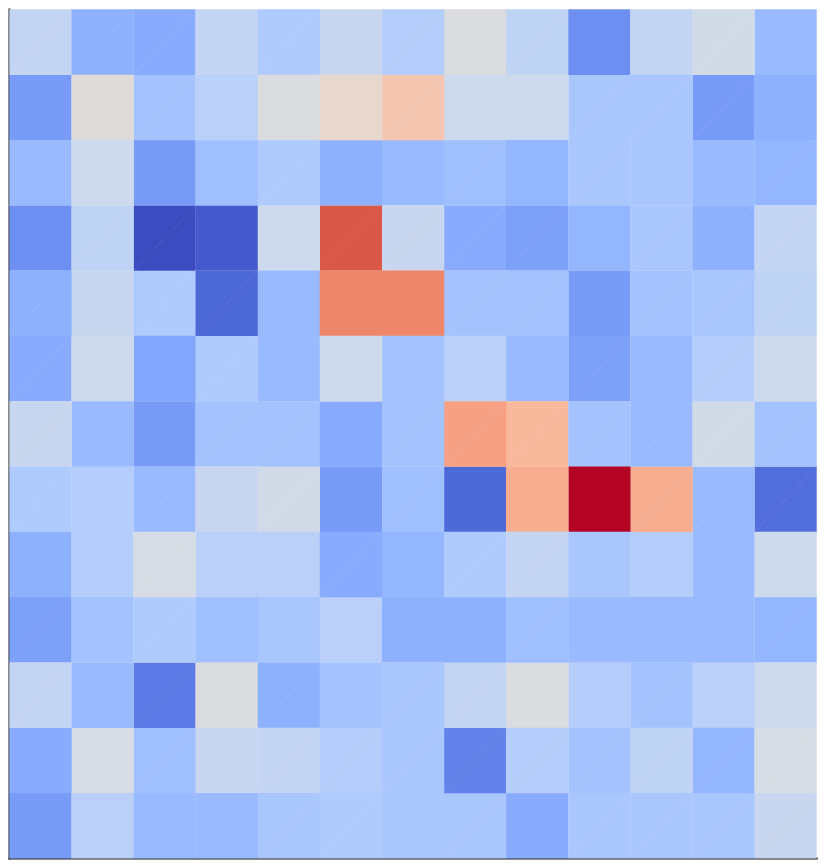} 
	\includegraphics[clip = on, trim = 10mm 10mm 10mm 10mm ,width = 0.16\textwidth]{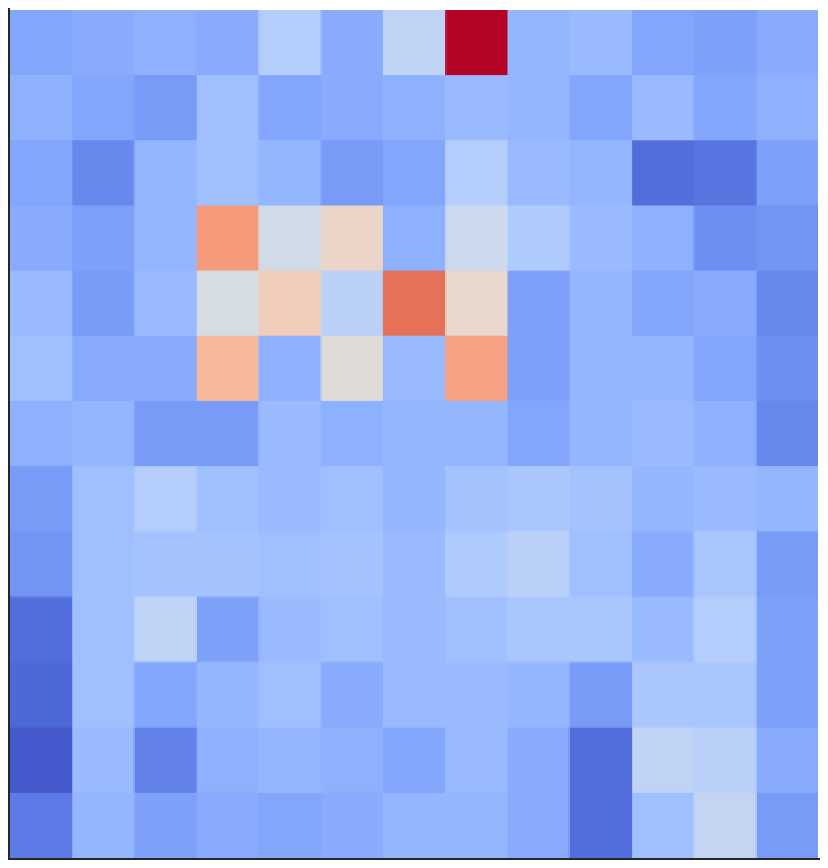} 
	\includegraphics[clip = on, trim = 10mm 10mm 10mm 10mm ,width = 0.16\textwidth]{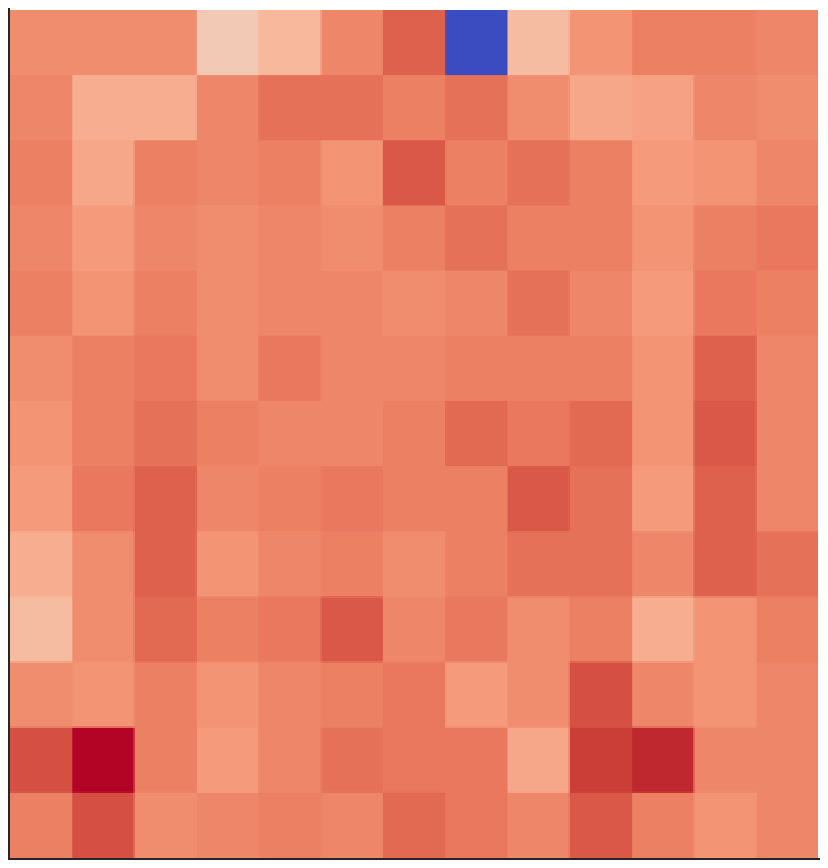} 
	\includegraphics[clip = on, trim = 10mm 10mm 10mm 10mm ,width = 0.16\textwidth]{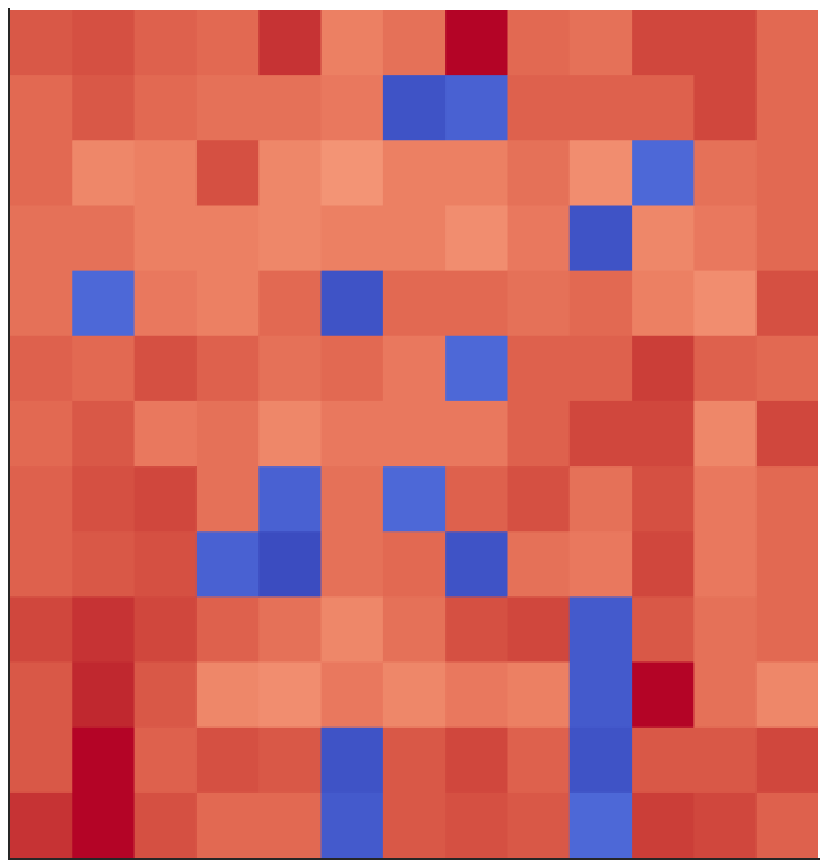} \\
	
	\includegraphics[clip = on, trim = 0mm 0mm 0mm 0mm ,width = 0.16\textwidth]{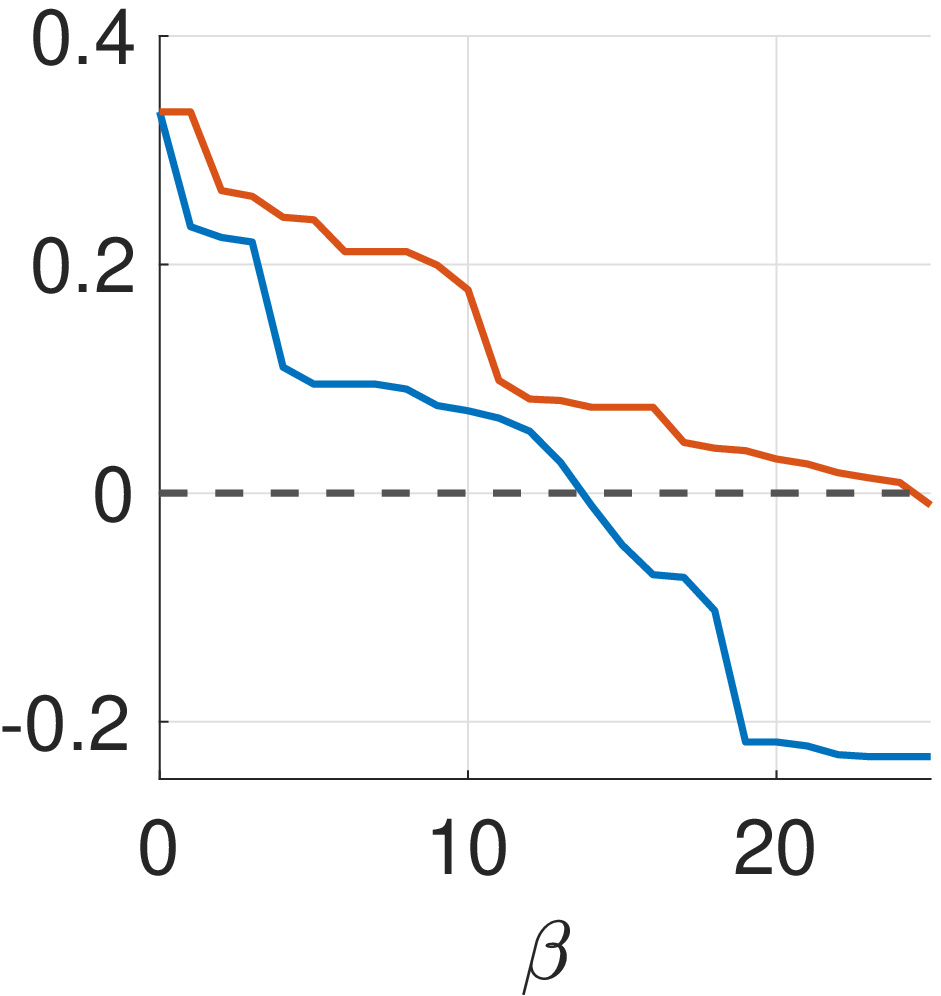}
	\includegraphics[clip = on, trim = 0mm 0mm 0mm 0mm ,width = 0.16\textwidth]{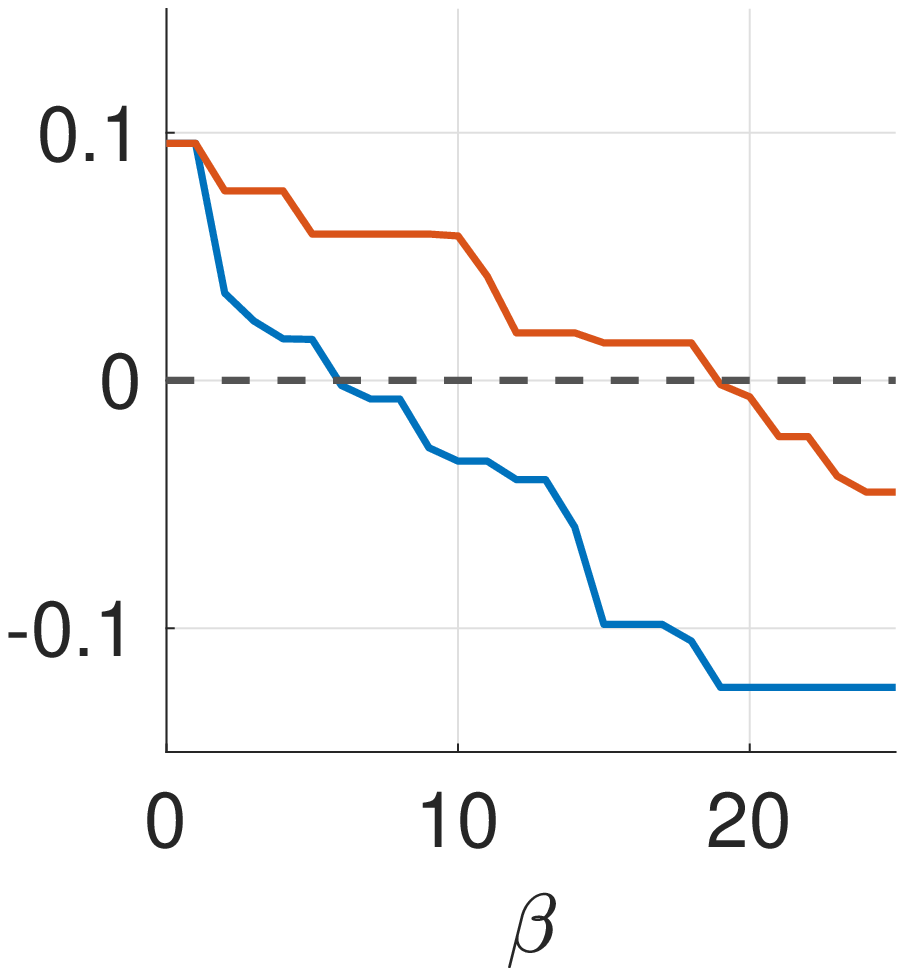}
	\includegraphics[clip = on, trim = 0mm 0mm 0mm 0mm ,width = 0.16\textwidth]{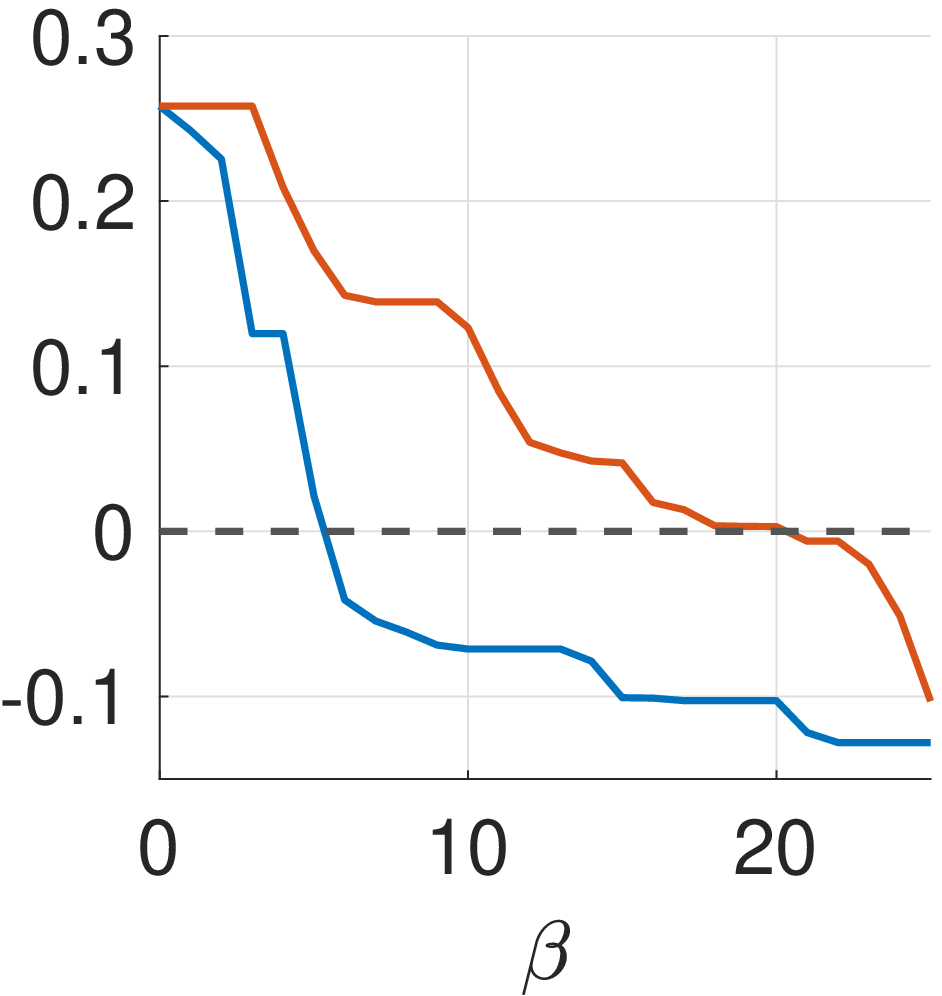}
	\includegraphics[clip = on, trim = 0mm 0mm 0mm 0mm ,width = 0.16\textwidth]{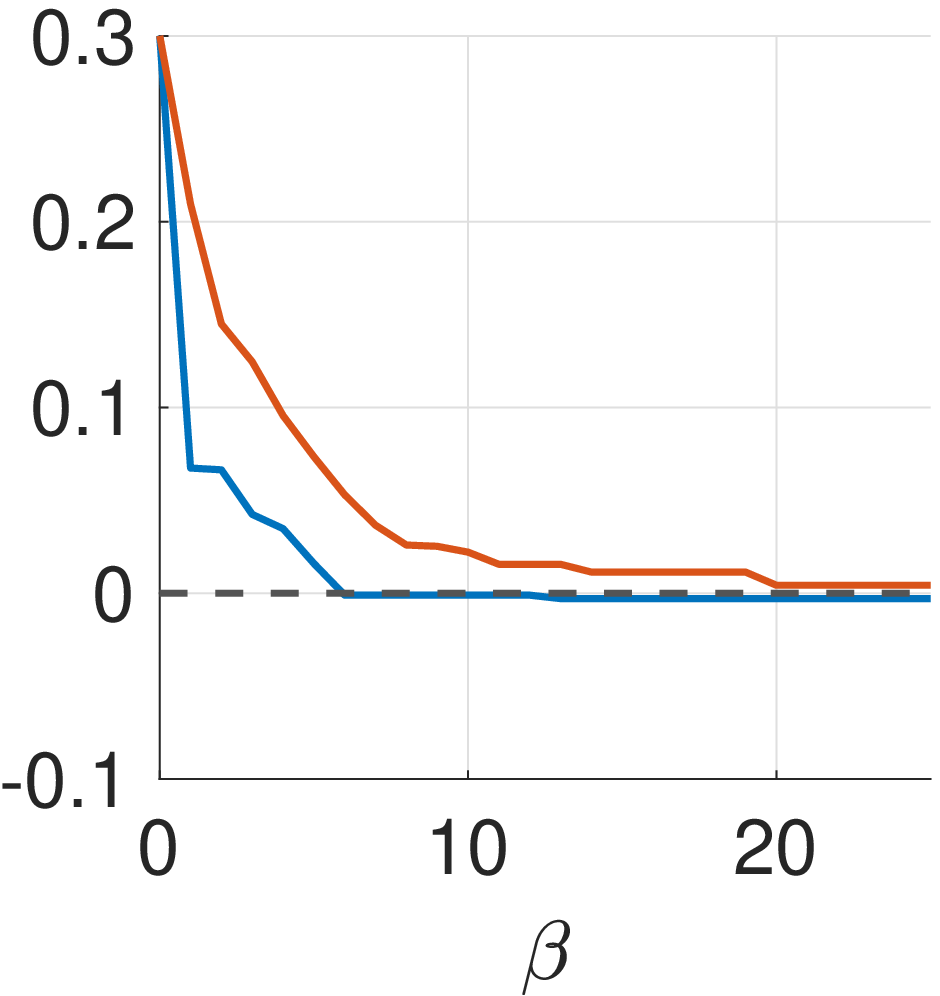}
	\includegraphics[clip = on, trim = 0mm 0mm 0mm 0mm ,width = 0.16\textwidth]{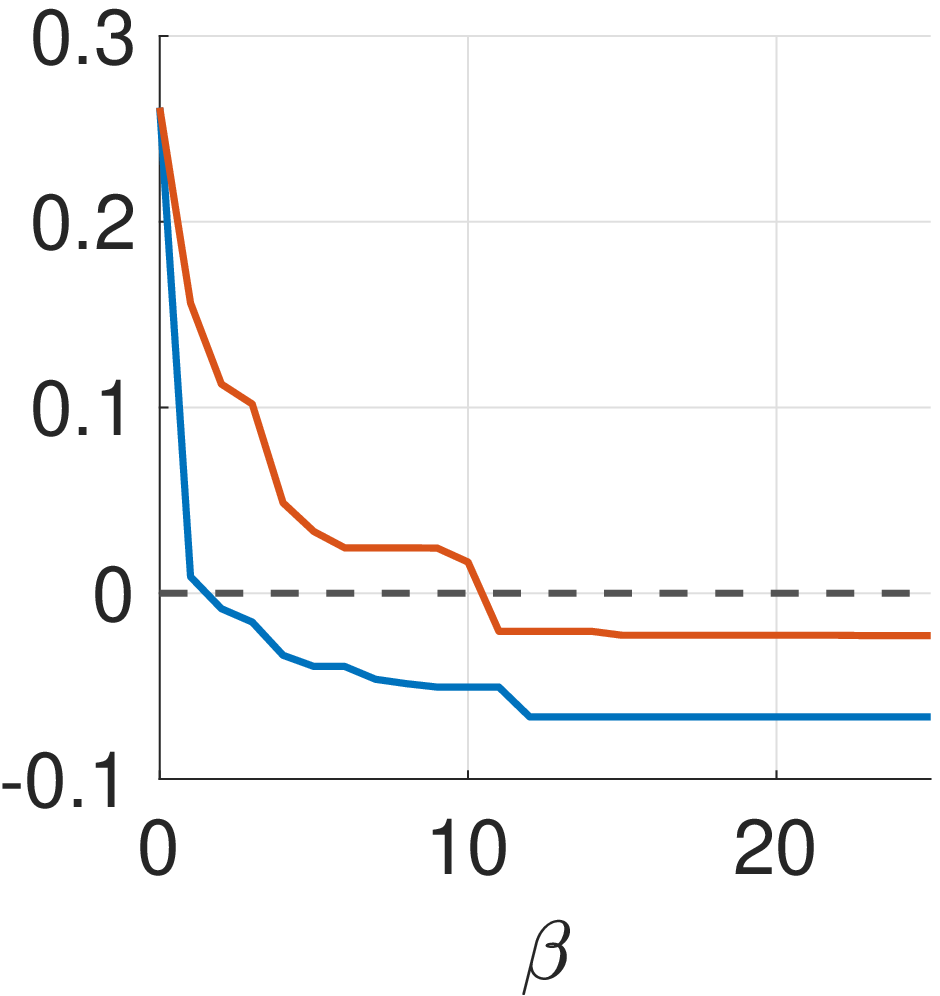}
	\includegraphics[clip = on, trim = 0mm 0mm 0mm 0mm ,width = 0.16\textwidth]{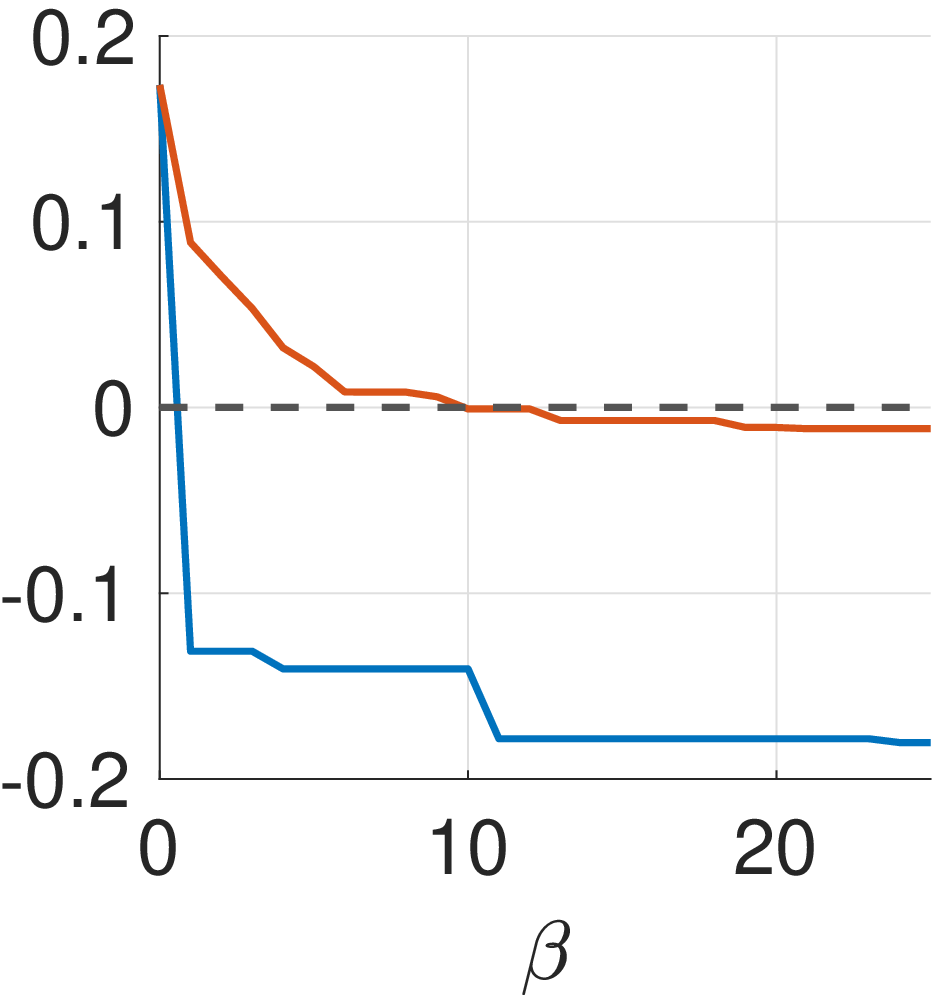} \\
	\includegraphics[clip = on, trim = 10mm 10mm 10mm 0mm ,width = 0.16\textwidth]{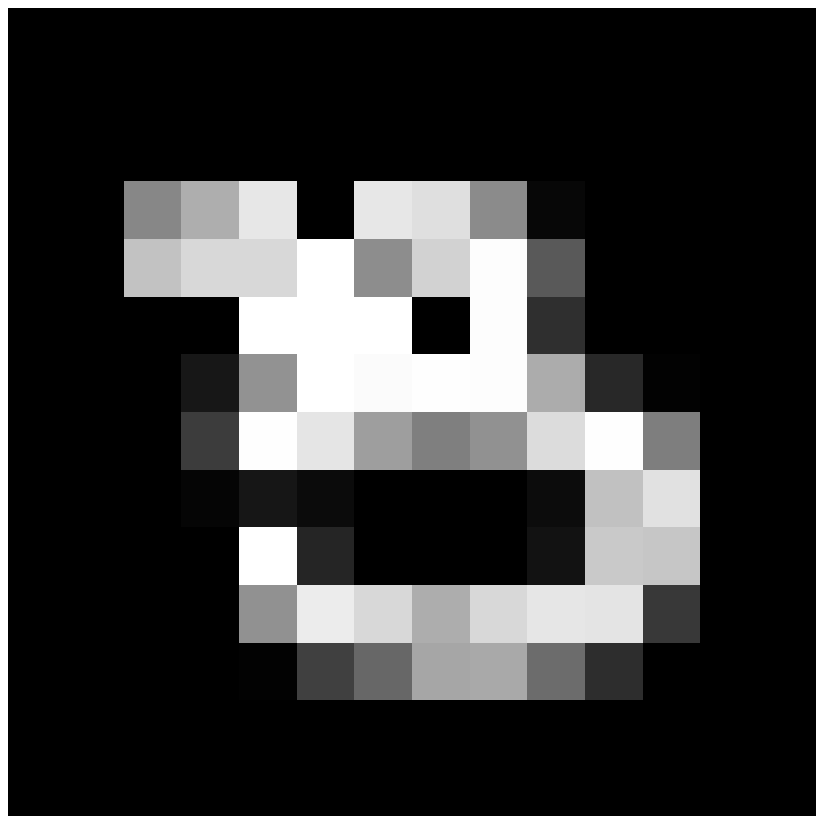}
	\includegraphics[clip = on, trim = 10mm 10mm 10mm 0mm ,width = 0.16\textwidth]{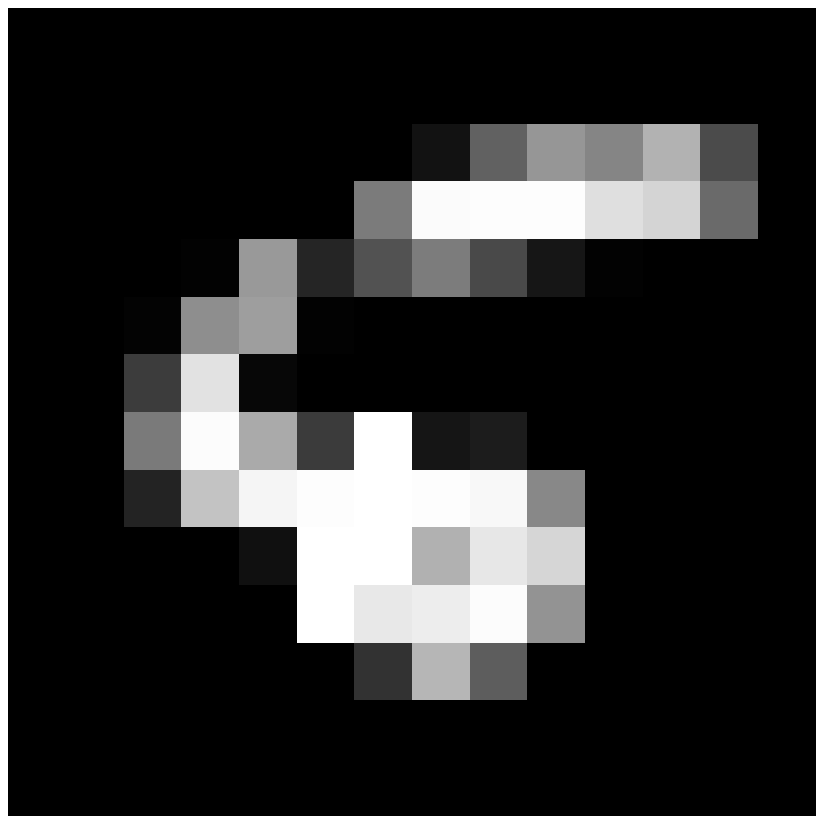}
	\includegraphics[clip = on, trim = 10mm 10mm 10mm 0mm ,width = 0.16\textwidth]{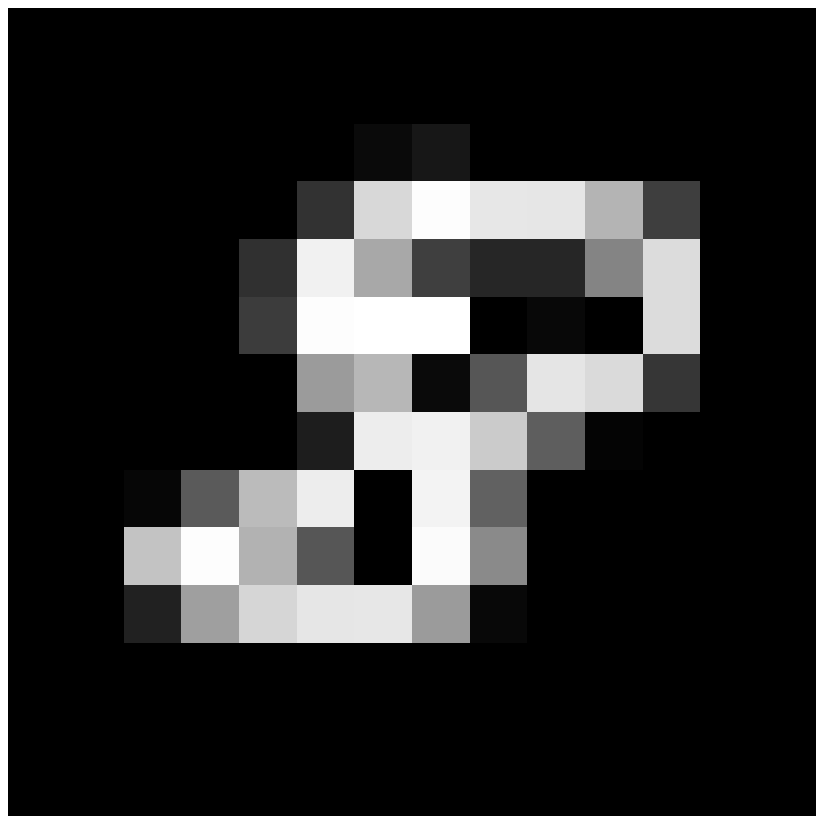}
	\includegraphics[clip = on, trim = 10mm 10mm 10mm 0mm ,width = 0.16\textwidth]{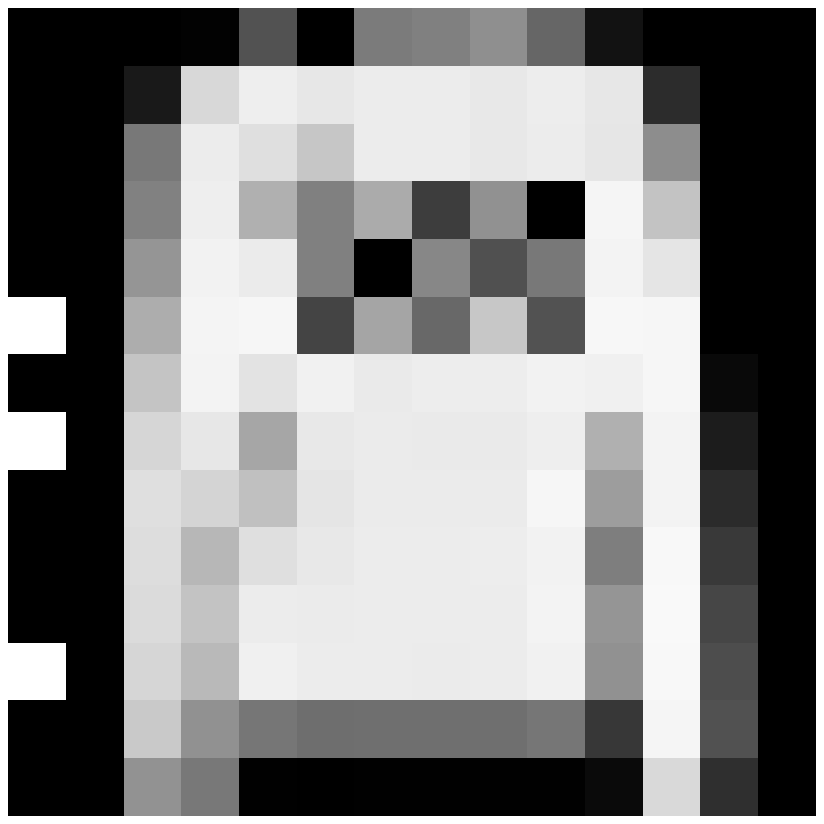}
	\includegraphics[clip = on, trim = 10mm 10mm 10mm 0mm ,width = 0.16\textwidth]{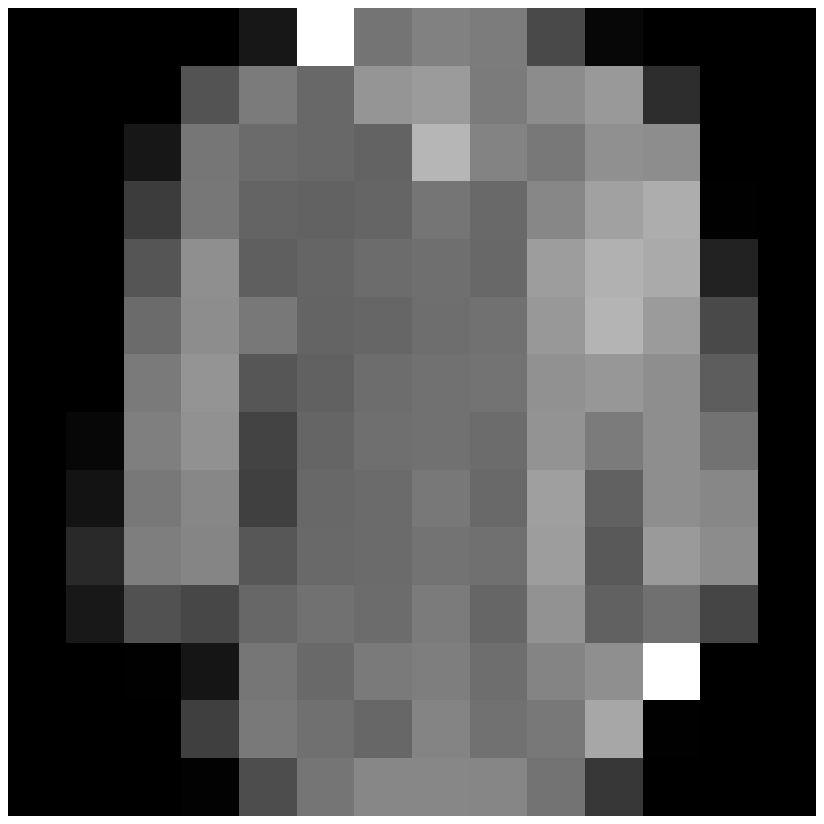}
	\includegraphics[clip = on, trim = 10mm 10mm 10mm 0mm ,width = 0.16\textwidth]{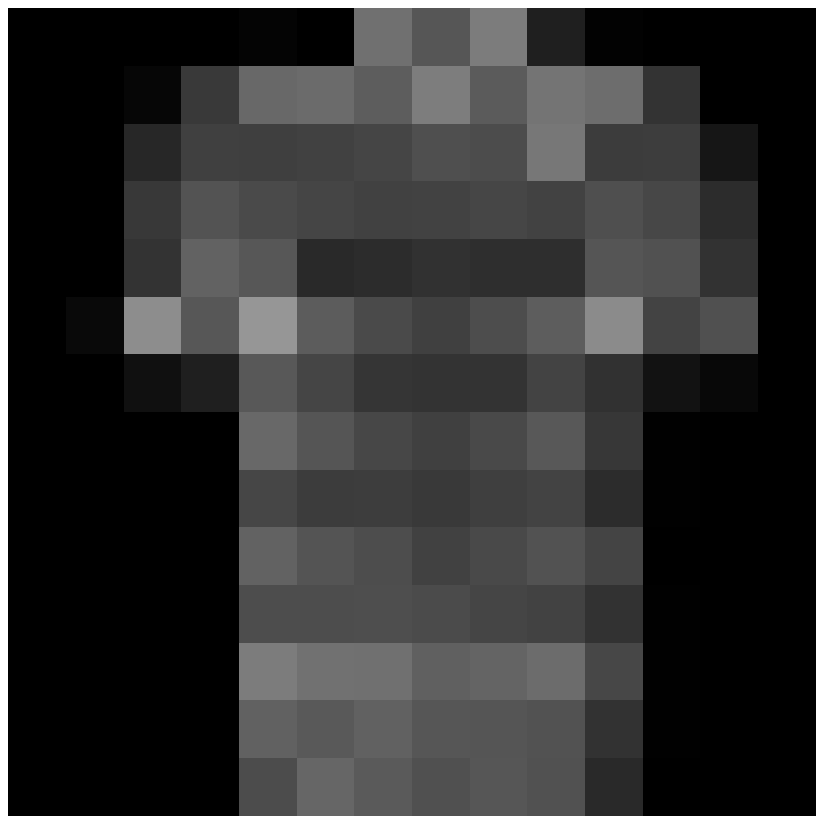}
	\caption{\textbf{First row}: 6 test images randomly selected from MNIST358 and F-MNIST-TSP. \textbf{Second row}: Interpretability metric estimation using our method. \textbf{Third row}: Comparison of adversarial gaps ($y$-axis) obtained for a given budget $\beta$ ($x$-axis) when using our method for interpretability estimation (blue line) and when using LIME (red line). The dashed grey line represents the threshold below which an adversarial is found. \textbf{Fourth row}: Minimal adversarial examples found by utilising our interpretability metric.} 
	\label{fig:Interpretability}
	\vspace*{-0.0cm}
\end{figure}
As shown in \citep{darwiche2020three}, interpretability metrics can be used to synthesise adversarial examples, because
pixels and features that are deemed important for the prediction are also likely to be highly vulnerable to adversarial perturbations.
These results can be used to glimpse further information about interpretability of the predictions by qualitatively examining the obtained adversarial examples.
To this end, given a test point $x^*$ and a point $x$ taken from a small neighbourhood around $x^*$, we define the adversarial gap, $\pi_{\text{gap}}(x)$, as the minimum difference between the confidence in the true class and those of the remaining classes on $x$, so that $\pi_{\text{gap}}(x) < 0$ implies that $x$ is an adversarial example for $x^*$.
We analyse how $\pi_{\text{gap}}$ changes as we change an increasing number of pixels, $\beta$, and compare the results obtained with our method to those of LIME.

We plot the results on six images randomly selected from the MNIST358 and F-MNIST-TSP datasets in Figure \ref{fig:Interpretability}. The selected clean test images are reported in the first row of the figure, and the interpretability values are reported as a heatmap in the second row directly below the corresponding images.
The colour gradient varies from red (positive impact, pixel value increase resulting in increased class probability of shown digit) to blue (negative impact, pixel value increase decreasing the class probability). 
The values of $\pi_{\text{gap}}$ obtained with our method (blue line) are compared with those from LIME (red line) in the third row, and in the fourth row we plot the minimal adversarial examples found with our method.

We observe that, for each image, and for each value of $\beta$, relying on the values estimated by LIME leads to an over-estimation of model robustness, and in some cases (e.g., third and fifth column) more than triple the number of pixel modifications is required to find an adversarial example. 
We note that the adversarial examples that we obtained for MNIST and F-MNIST are qualitatively different.
For the MNIST image, our method modified salient bits of the image.
For digit 3, for example, the interpretability analysis retrieves a contiguous blue patch on the left, which is deemed to be the most important part for the prediction. 
When this is modified in adversarial fashion, the image obtained resembles an 8 in the upper part, and a 3 in the lower part, and is (understandably) classified as an 8 by the GP. 
Similarly, digit 5 is modified so that the lower part resembles an 8, whereas in the image of the 8 a 3 shaped contour is highlighted in the adversarial example.
For the F-MNIST image, instead, our method detects pixels that are important for the GP prediction but have little semantic meaning for a human, that is, where modifying pixels in the border of the image suffices to find adversarial examples.

%% file: sections/conclusion.tex
We presented a method for computing, for any compact region of the input space surrounding a test point, provable guarantees on the adversarial robustness of a GP model for all points in that region, up to any desired precision $\epsilon > 0$.
To achieve this, we have developed a branch-and-bound optimisation scheme that computes upper and lower bounds on the minimum and maximum of the model prediction ranges, and proved that it converges in finitely many steps up to an error tolerance  $\epsilon > 0 $ selected a-priori. 

We have experimentally evaluated our method on four datasets, providing results for adversarial robustness, bounds over the predictive posterior distribution and local/global interpretability analysis.
Empirically, we have observed that, in Bayesian prediction settings and with GPs, the adversarial robustness of the model increases with the accuracy of the posterior distribution approximation, and with better hyper-parameter calibration.
%
This differs from what is generally observed in frequentist approaches to learning, for example, in deep neural networks, where better accuracy was empirically observed to imply  worse adversarial robustness \citep{zhang2019theoretically,su2018robustness}. 
We have also observed that increasing the number of training samples might still be beneficial for adversarial robustness even when using sparse approximations for GP training.

One limitation of the approach presented in this paper is its high, exponential time, computational complexity. This unsurprising since the problem we are solving is non-linear optimisation. To reduce the computational time requirement, we have formulated analytical solutions for the main types of kernels used in practical applications. We have also observed that sparse GPs, as well as improving training time, can significantly reduce the time requirement, as bounding functions need to be computed only with respect to the inducing points. We believe that the methods proposed in this paper are therefore widely applicable in practice.

%% file: sections/appendix.tex
\section{Additional Lemmas and Proofs}\label{sec:lemmas_and_proofs}
In this section we provide statements of additional lemmas referred to in the main text of the paper, as well as proofs of that were omitted for space reasons. 
\begin{lemma}
\label{lemmma:linear_prop}
Let $g_L(t) =  a_L  + b_L t $ and $g^U(t) =  g_U(t) =  a_U  + b_U t $ be an LBF and UBF to a function $g(t)$ $\forall t \in \mathcal{T}$, i.e. $ g_L(t) \leq g(t) \leq  g_U(t) $ $\forall t \in \mathcal{T} \subseteq \mathbb{R}$. 
Consider two real coefficients $\alpha \in \mathbb{R}$ and $\beta \in \mathbb{R}$.
Define
\begin{align}
    &\bar{b}_L = \begin{cases}  \alpha b_L \; \textrm{if} \,   \alpha \geq 0  \\ \alpha b_U \; \textrm{if} \,   \alpha < 0  \end{cases}
    \bar{a}_L = \begin{cases}  \alpha a_L + \beta \; \textrm{if} \,   \alpha \geq 0  \\ \alpha a^U  + \beta \; \textrm{if} \,  \alpha < 0  \end{cases} \label{eq:lin_transf1}\\
    &\bar{b}_U = \begin{cases}  \alpha b_U \; \textrm{if} \,   \alpha \geq 0  \\ \alpha b_L \; \textrm{if} \,   \alpha < 0  \end{cases}
    \bar{a}_U = \begin{cases}  \alpha a_U + \beta \; \textrm{if} \,   \alpha \geq 0  \\ \alpha a_L  + \beta \; \textrm{if} \,   \alpha < 0  \end{cases} \label{eq:lin_transf2}
\end{align}
Then:
\begin{align*}
\bar{g}_L(t) :=  \bar{a}_L  + \bar{b}_L t \leq \alpha g(t) + \beta \leq  \bar{a}_U  + \bar{b}_U t  =: \bar{g}_U(t)
\end{align*}
That is, LBFs can be propagated through linear transformation by redefining the coefficients through Equations \eqref{eq:lin_transf1}--\eqref{eq:lin_transf2}.
\end{lemma}
\begin{proof}
The proof is an immediate consequence of multiplying the inequalities  $ g_L(t) \leq g(t) \leq  g_U(t) $ with the coefficients $\alpha$ and $\beta$ and re-writing the new inequality using the constants defined in Equations \eqref{eq:lin_transf1}--\eqref{eq:lin_transf2}.
\end{proof}

\begin{lemma}\label{lemma:sign_of_sigmoid_derivative}
Consider the sigmoid function $\sigma(x) = {\frac {1}{1+e^{-x}}}. $Let $z > 0$, then we have:
\begin{align*}
    \sigma'(\mu - z) \begin{cases} > \sigma'(\mu + z) \quad \text{if} \quad \mu > 0 \\ < \sigma'(\mu + z) \quad \text{if} \quad \mu < 0 \end{cases}
\end{align*}
\end{lemma}
\begin{proof}
Let $\mu > 0$; the proof when $\mu < 0$ is similar, because $\sigma'$ is an even function.

\textit{Case 1:} $\mu - z \geq 0$. Since $\sigma$ is strictly concave in $[0,+\infty)$, the derivative is a monotonic crescent in the relevant region. 
Thus, $\sigma'(\mu - z) > \sigma'(\mu + z) $.

\textit{Case 2:} $\mu - z < 0$. Since $\sigma'$ is even we have  $\sigma'(\mu - z) = \sigma'(z - \mu)$. Now $z-\mu>0$, similarly to Case 1 we have $\sigma'(z - \mu) < \sigma'(z + \mu)$, which proves the lemma.
\end{proof}

\begin{lemma}
\label{Lemma:supProbab}
Let $X$ and $Y$ be random variables with joint density function $f_{X,Y}$. Consider measurable sets $A$ and $B$. Then, it holds that 
$$ P(X\in A | Y\in B)\geq \inf_{y \in B}  P(X\in A | Y=y). $$
\end{lemma}
\begin{proof}
\begin{align*}
    &P(X\in A | Y \in B) = \frac{P( X\in A , Y \in B  )}{P(Y \in B)} = \frac{\int_{B} \int_{A}  f_{X,Y}(x,y) dx dy  }{\int_B f_Y(y)dy}\\
    &=  \frac{\int_{B} \int_{A}  f_{X|Y}(x|y)f_Y(y) dx dy  }{\int_B f_Y(y)dy} \geq \frac{ \int_{B} f_Y(y)dy \inf_{y \in B}\int_{A} f_{X|Y}(x|y)  dx   }{\int_B f_Y(y)dy}  \\
    &= \inf_{y \in B} P(X \in A | Y=y).
\end{align*}
\end{proof}

\subsection{Proof of Proposition \ref{prop:lbf_ubf_kernel}}
\begin{proof}
We show how to compute $\bar{a}_L$ and $\bar{b}_L$; the same arguments also apply to the computation of $\bar{a}_U$ and $\bar{b}_U$ by simply considering $-\Sigma_{x,\bar{x}}$.

Consider $c_L = -1$ and $c_U = 1 $ coefficients associated to the input point $\bar{x}$.
Let $\varphi_L = U(c_L)$ and $\varphi_U = U(c_U)$, then by Assumption $3$ of bounded kernel decomposition we have that $\varphi(x,\bar{x}) \in [\varphi_L, \varphi_U ]$ for all $x \in T$.
Consider now the function $\psi$ restricted to the interval $[\varphi_L, \varphi_U ]$.
Then there are four cases to consider for $\psi$. 
\paragraph{Case 1}  
If $\psi$ happens to be concave in $[\varphi_L, \varphi_U ]$, then, by definition of concave function, a lower bound is given by the line that links the points $(\varphi^L,\psi (\varphi^L)  )$ and $(\varphi^U,\psi (\varphi^U)  )$, that is, $g_L$ is simply the LBF with coefficients:
\begin{align*}
    \bar{b}_L &= \frac{\psi (\varphi^L) - \psi (\varphi^U)}{\varphi^L - \varphi^U}\\ 
    \bar{a}_L &= \psi (\varphi^L) -\bar{b}_L \varphi^L.
\end{align*}
\paragraph{Case 2}
If $\psi$ happens to be a convex function, then, by definition of convex function and by the differentiability of $\psi$, a valid lower bound is given by the tangent line in the middle point $\varphi^C = (\varphi^L + \varphi^U)/2$ of the interval, that is, $g_L$ is the LBF with coefficients:
\begin{align*}
    \bar{b}_L &= \frac{d\psi(\varphi^C)}{d\varphi}\\ 
    \bar{a}_L &= \psi (\varphi^L) - \bar{b}_L \varphi^L.   
\end{align*}
\paragraph{Case 3}
Assume now that $\psi$ is concave in  $[\varphi^L, \varphi^F ]$, and convex in $ [\varphi^F, \varphi^U ]$ (the arguments are very similar if we assume the first interval is that in which $\psi$ is convex and the second 
concave).
In other words, there is only one flex point $\varphi^F \in  (\varphi^L, \varphi^U )$.
Let $\bar{a}_L'$, $\bar{b}_L'$ be coefficients for linear lower approximation in $[\varphi^L, \varphi^F ]$ and $\bar{a}_L''$, $\bar{b}_L''$ analogous coefficients in $[\varphi^F, \varphi^U ]$ (respectively computed as for Case 1 and Case 2 above), and call $g'$ and $g''$ the corresponding functions.
Define $g_L$ to be the LBF function of coefficients $\bar{a}_L$ and $\bar{b}_L$ that goes through the two points $(\varphi^L , \min(g'(\varphi^L),g''(\varphi^L)  )   )$ and $(\varphi^U, g''(\varphi^U))$.
We then have that $g_L$ is a valid lower bound function for $\psi$ in $[\varphi^L, \varphi^U ]$.
In order to prove this we distinguish between two cases:
\begin{enumerate}
    \item If $\min(g'(\varphi^L),g''(\varphi^L)) = g'(\varphi^L)$, then we have that $g_L(\varphi^L) = g'(\varphi^L)  \leq g''(\varphi^L) $, and $g_L(\varphi^U) = g''(\varphi^U)$.
    Hence, because of linearity, $g_L(\varphi) \leq g''(\varphi)$ in $[\varphi^L,\varphi^U]$, and in particular in $[\varphi^F,\varphi^U]$ as well. 
    This also implies that $g_L(\varphi^F) \leq g''(\varphi^F) \leq g'(\varphi^F)$. 
    On the other hand, $g_L(\varphi^L) = g'(\varphi^L)$, hence $g_L(\varphi) \leq g'(\varphi)$ in $[\varphi^L,\varphi^F]$. 
    Combining these two results and by construction of $g'$ and $g''$ we have that $g_L(\varphi) \leq \psi(\varphi)$ in $[\varphi^L,\varphi^U]$.
    \item If $\min(g'(\varphi^L),g''(\varphi^L)) =  g''(\varphi^L)$, then in this case we have $g_L = g''$, and just have to show that $g(\varphi) \leq g'(\varphi)$ in $[\varphi^L,\varphi^F]$. 
    This immediately follows by noticing that $g''(\varphi^F) \leq g'(\varphi^F)$ and $g''(\varphi^L) \leq g'(\varphi^L)$.
\end{enumerate}
\paragraph{Case 4}
In the general case, as we have a finite number of flex points, we can divide  $[\varphi^L, \varphi^U ]$ in subintervals in which $\psi$ is either convex or concave.
We can then proceed iteratively from the two left-most intervals by repeatedly applying Case 3. 
\end{proof}

\subsection{Proof of Lemma \ref{lemma:convergence_lbf}}
\begin{proof}
We prove the lemma for the LBF. 
An analogous argument can be made for the UBF.

Letting $\epsilon > 0$, we want to find an $\bar{r} > 0$ such that $\text{diam}(R) < \bar{r}$ implies $\max_{x \in R}| g_L^R(x) - \Sigma_{\bar{x},x} | < \epsilon$.
Consider $\varphi_L^R$ and $\varphi_U^R$, lower and upper bound values for $\varphi$ in $R$.
By taking $\bar{r}$ small enough we can assume without loss of generality that $\psi(\varphi)$ has at most one flex point in $[\varphi_L^R,\varphi_U^R]$. 
We then have the following three cases.

\textit{CASE 1:} if $\psi(\varphi)$ is concave  then $g_L^R$ is defined as the line connecting the two extreme points of the interval $[\varphi_L^R,\varphi_U^R]$.
Since $\psi(\varphi)$ is concave, we have that it obtains its minimum in one of these two extrema, so that we have $$\min_{x \in R} g_L^R(x) = \min_{x \in R} \Sigma_{\bar{x},x}.$$
By Assumption 2 of kernel decompositions (see Definition \ref{def:kernel_decomposition}), it follows that $\psi$ is Lipschitz continuous on any compact interval, so that we have that:
$$ \lim_{r\to0} \left\vert \min_{x \in R} \Sigma_{\bar{x},x} - \max_{x \in R} \Sigma_{\bar{x},x} \right\vert = 0$$
where $r = \text{diam}(R)$.
Putting the two results together we have that the difference between $\min_{x \in R} g_L^R(x) $ and $ \max_{x \in R} \Sigma_{\bar{x},x}$ vanishes  whenever that $r$ tends to zero, which proves the statement.

\textit{CASE 2:} if $\psi(\varphi)$ is convex then $g_L^R$ is the Taylor expansion of $\psi(\varphi)$ around the mid-point of the interval, truncated at the first-order.
By continuity of $\varphi$ we then obtain that shrinking $r$ shrinks also the width of the interval $[\varphi_L^R,\varphi_U^R]$, which then, relying on the properties of truncation error of Taylor expansions, proves the lemma statement.

\textit{CASE 3:} in the case in which a flex point exists, $g_L^R$ is defined to be the maximum line that is below the two LBFs respectively defined over the convex and the concave part of the interval. 
Since by Case 1 and Case 2 these converge, we also have that $g_L^R$ converges.
\end{proof}

\subsection{Proof of Lemma \ref{Lemma:Gaussian}}
\begin{proof}
We provide the proof for the minimum; similar arguments also hold for the maximum.

By definition of $\mu^L_T$, $\mu^U_T$, ${\Sigma}^L_T $, ${\Sigma}^U_T$, we have that for every $x \in T$, $\bar{\mu}(x) \in [\mu^L_T, \mu^U_T]$ and $\bar{\Sigma}(x) \in [\Sigma^L_T, \Sigma^U_T]$.
Thus:
\begin{align*}
&\min_{x \in T} \int_{a}^{b} \mathcal{N}(\xi | \bar{\mu}(x),\bar{\Sigma}(x))d\xi
\geq \min_{\substack{  \mu \in [\mu^L_T,\mu^U_T] \\ \Sigma \in [{\Sigma}^L_T ,{\Sigma}^U_T]  }} \int_{a}^{b} \mathcal{N}(\xi | \mu,\Sigma) d \xi =\\
&\frac{1}{2} \min_{\substack{  \mu \in [\mu^L_T,\mu^U_T] \\ \Sigma \in [{\Sigma}^L_T ,{\Sigma}^U_T]  }} \left( \erf \left(\frac{\mu-a}{ \sqrt{2\Sigma}}\right) - \erf \left(\frac{\mu-b}{ \sqrt{2\Sigma}}\right) \right) = 
\frac{1}{2} \min_{\substack{  \mu \in [\mu^L_T,\mu^U_T] \\ \Sigma \in [{\Sigma}^L_T ,{\Sigma}^U_T]  }} \Phi(\mu,\Sigma)
\end{align*}
where we have set $\Phi(\mu,\Sigma) \defnotation \erf \left(\frac{\mu-a}{ \sqrt{2\Sigma}}\right) - \erf \left(\frac{\mu-b}{ \sqrt{2\Sigma}}\right)  $.
By looking at the partial derivatives we have that:
\begin{align*}
    &\frac{\partial \Phi(\mu,\Sigma)}{\partial \mu} = \frac{\sqrt{2}}{\sqrt{ \pi \Sigma }} \left( e^{-\frac{({\mu} - b)^2}{2\Sigma}  } -  e^{-\frac{({\mu} - a)^2}{2\Sigma}  } \right) 
    \geq 0 \Leftrightarrow 
     \mu \leq \frac{a + b}{2} = \mu^c
\end{align*}
and that if $\mu \not\in [a,b]$: 
\begin{align*}
    &\frac{\partial \Phi(\mu,\Sigma)}{\partial \Sigma} =   \frac{1}{\sqrt{2 \pi \Sigma^3 }} \left( ({\mu} - b_i) e^{-\frac{({\mu} - b_i)^2}{2\Sigma^2}  } - ({\mu} - a_i) e^{-\frac{({\mu} - a_i)^2}{2\Sigma^2}  } \right) 
    \geq 0\\ 
    &\Leftrightarrow 
    \Sigma \leq \frac{{(\mu-a)^2-(\mu-b)^2}}{{2 \log \frac{\mu-a}{\mu-b}}} = \Sigma^c(\mu) 
\end{align*}
as otherwise the last inequality has no solutions.
As such, $\mu^c$ and $\Sigma^c$ will correspond to global maximum with respect to $\mu$ and $\Sigma$, respectively.
As $\Phi$ is symmetric w.r.t.\ $\mu^c$ we have that the minimum value w.r.t. to $\mu$ is always obtained for the point furthest away from $\mu^c$, that is, at $\munderbar{\mu}^* = \argmax_{\mu \in [ \mu^L_T, \mu^U_T ]} \vert \mu^c - \mu \vert $.
The minimum value w.r.t. to $\Sigma$ will hence be either for $\Sigma^L_T$ or $\Sigma^U_T$, that is $\munderbar{\Sigma}^* =\argmin_{\Sigma \in \{\Sigma^L_T, \Sigma^U_T\}} \Phi(\munderbar{\mu}^*,\Sigma).$
\end{proof}

\section{Kernel Function Decomposition}\label{sec:decomposition}
In this section we compute explicit kernel decompositions $(\varphi,\psi,U)$ for several kernels of practical relevance in applications.
In particular, we give explicit formulas for the squared-exponential, the rational quadratic, the Mat\'ern (for half-integer values) and the periodic kernels, along with how kernel decomposition can be propagated through addition and multiplication with kernels.
Furthermore, we show how to compute kernel decompositions for generalised spectral kernels, both in the stationary and non-stationary case.

Throughout this section we assume $T = [x^L,x^U]$, for some $x^L, x^U \in \inputspace $. 
For building the bounding function $U$ we use the notation $x^{(i)}$, $i=1,\ldots,\tss$ for the set of input points, and $c_i$, $i=1,\ldots,\tss$, for their associated multiplying coefficients.

\subsection{Squared-Exponential Kernel}\label{sec:sqe_decomposition}
For the squared-exponential kernel, we build a bounded kernel decomposition by setting:
\begin{align*}
    \psi(\varphi) &= \sigma^2 \exp{ \left( -  \varphi  \right) } \\
    \varphi(x',x'') &= \sum_{j=1}^d \theta_j (x'_j -x''_j)^2.
\end{align*}
It is straightforward to notice that Assumptions $1$ and $2$ of Definition \ref{def:kernel_decomposition} are met by this decomposition.
Concerning the definition of $U$, consider a set $\x{1},\ldots,\x{\tss}$ of $\tss$ points in the input space and associated real coefficients $c_1,\ldots,c_\tss$.
For a hyper-rectangle $T = [x^L,x^U]$ we have that:
\begin{align*}
     \sup_{x \in T} \sum_{i=1}^{\tss} c_i \varphi(x,\x{i}) &= \sup_{x \in T} \sum_{i=1}^{\tss} c_i \sum_{j=1}^d \theta_j (x_j -\x{i}_j)^2 = 
     \sup_{x \in T} \sum_{j=1}^d \theta_j \sum_{i=1}^{\tss} c_i (x_j -\x{i}_j)^2  \\
      &=  \sup_{x \in T} \sum_{j=1}^d \left( \theta_j \sum_{i=1}^{\tss}c_i x_j^2   - 2 \theta_j \sum_{i=1}^{\tss}c_i \x{i} x_j + \theta_j  \sum_{i=1}^{\tss}c_i x^{(i)2}  \right) \\
      &= \sum_{j=1}^d \sup_{x_j\in [x^L_j,x^U_j]} \left( \theta_j \sum_{i=1}^{\tss}c_i x_j^2   - 2 \theta_j \sum_{i=1}^{\tss}c_i \x{i} x_j + \theta_j  \sum_{i=1}^{\tss}c_i x^{(i)2}  \right). 
\end{align*}
The right-hand-side of the last equation simply involves the computation of the maximum of a 1-d parabola over an interval of the real line, which can be done exactly and in constant time by simple inspection of the derivative function and by evaluating the function at the extrema of the interval.
Call $\bar{x}_j$ the only critical point of the $j$th parabola, and denote with $h_j(x_j) = \alpha_j x_j^2 + \beta_j x_j + \gamma_j $ the parabola associated with the $j$th coordinate value, with $\alpha_j =  \theta_j \sum_{i=1}^{\tss}c_i$, $\beta_j = - 2 \theta_j \sum_{i=1}^{\tss}c_i \x{i}$ and $\gamma_j = \theta_j  \sum_{i=1}^{\tss}c_i x^{(i)2} $, then we set:
\begin{align}\label{eq:bounding_sqe}
    U(\mathbf{c}) &= \sum_{j=1}^d  U_j(\mathbf{c})
\end{align}
where:
\begin{align*}
U_j(\mathbf{c}) &= \begin{cases} \max{ \{ h_j(x_j^L), h_j(x_j^U) , h_j(\bar{x}_j)   \}   }   \quad  & \text{if} \quad \bar{x}_j \in [x_j^L , x_j^U ]   \\
\max{ \{ h_j(x_j^L), h_j(x_j^U)    \}} \quad  & \text{otherwise}   \end{cases}. \\
\end{align*} 
%
Furthermore it follows that the time complexity for the computation of $U$ in the squared-exponential case is $\mathcal{O}(\tss + d)$.
\subsection{Rational Quadratic Kernel}
An analogous argument to the above holds for the rational quadratic kernel, where we can set:
\begin{align*}
    \psi(\varphi) &= \sigma^2\left(1+\frac{\varphi}{2}\right)^{-\alpha} \\
    \varphi(x',x'') &= \sum_{j=1}^d \theta_j (x'_j -x''_j)^2.
\end{align*}
As the definition of $\varphi$ is exactly the same as for the squared-exponential kernel, the bounding function $U$ can be defined as in Equation \eqref{eq:bounding_sqe}.

\subsection{Mat\'ern Kernel}
For half-integer values, the explicit form of the Mat\'ern Kernel allows us to find an analogous kernel decomposition to the two discussed above:
\begin{align*}
    \psi(\varphi) &= \sigma^2 k_{p} \exp{(-\sqrt{\hat{k}_p \varphi})}\sum_{l = 0}^{p}k_{l,p} \sqrt[{p-l}]{\hat{k}_p \varphi} \\  
    \varphi(x',x'') &= \sum_{j=1}^d \theta_j (x'_j -x''_j)^2.
\end{align*}
\subsection{Periodic Kernel}
For the periodic kernel we define:
\begin{align*}
    \psi(\varphi) &= \sigma^2 \exp(-0.5\varphi) \\
    \varphi(x',x'') &= \sum_{j = 1}^{d}\theta_j \sin(p_j(x'_j-x''_j))^2.
\end{align*}
Assumptions 1 and 2 are trivially satisfied because of the smoothness of $\psi$ and $\varphi$.
For the definition of the bounding function $U$ we have that:
\begin{align*}
    \sup_{x \in T} \sum_{i=1}^{\tss} c_i \varphi(x,\x{i}) =& \sup_{x \in T} \sum_{i=1}^{\tss} c_i \sum_{j = 1}^{d}\theta_j \sin(p_j(x_j-\x{i}_j))^2 \\ 
    \leq &  \sum_{i=1}^{\tss} \sum_{j = 1}^{d}   \sup_{x_j \in [ x^L_j , x^U_j ] }  c_i \theta_j \sin\left( p_j (x_j-\x{i}_j)  \right)^2  .
\end{align*}
The supremum in the final equation can be obtained by simply inspecting the derivative of $c_i \theta_j  \sin\left( p_j (x_j-\x{i}_j)  \right)^2$ and its function value at the extrema of each interval  $ [ x^L_j , x^U_j ]$.
Let $U_{ij}(c_i)$ be the value computed in such a way for each $i$ and $j$, then we define:
\begin{align}\label{eq:priodicU}
    U(\mathbf{c}) &= \sum_{j=1}^d \sum_{i=1}^{\tss}  U_{ij}(c_i).
\end{align}
Furthermore it follows that the time complexity for the computation of $U$ in the squared-exponential case is $\mathcal{O}(\tss d)$.

\subsection{Kernel Addition}\label{app:kernel_addition}
Consider now the case in which the kernel function $\Sigma$ is defined by linear composition of two kernels $\Sigma'$ and $\Sigma''$ such as:
\begin{align}
     \Sigma_{x',x''} = k' \Sigma'_{x',x''} + k'' \Sigma''_{x',x''} \qquad \forall x',x'' \in \inputspace
\end{align}
for some given $k'$ and $k'' \geq 0$.
Then, we have that kernel decomposition for $\Sigma'$ and $\Sigma''$ can be simply propagated through the sum.
To see that, let $(\varphi',\psi',U')$ and $(\varphi'',\psi'',U'')$ be the two kernel decomposition.
Then, by simply summing up the LBFs and UBFs for $\Sigma'$ and $\Sigma''$, Proposition \ref{prop:lbf_ubf_kernel} can be generalised to this case as follows.
\begin{proposition}\label{prop:lbf_ubf_kernel_sum}
    Let $g'_L$, $g'_U$, $g''_L$ and  $g''_U$ be lower and upper bounding function for $\Sigma'_{x,\bar{x}}$ and $\Sigma''_{x,\bar{x}}$, for all $x \in T$, as computed in Proposition \ref{prop:lbf_ubf_kernel}. 
    Then
    \begin{align*}
    g_L(x) &= k' g'_L(x) + k'' g''_L(x)  \\
    g_U(x) &= k' g'_U(x) + k'' g''_U(x)
    \end{align*}
    are respectively lower and upper bounding functions on $\Sigma_{x,\bar{x}}$.
\end{proposition}  
As a consequence of the above proposition, it immediately follows that the infimum of the posterior mean function over the compact set $T$ can be safely lower-bounded for the kernel $\Sigma$ by setting:
\begin{align*}
    \mu^L_T =  k' \mu'^L_T + k'' \mu''^L_T,
\end{align*}
where $\mu'^L_T$ and $\mu''^L_T$ are computed by applying Proposition \ref{prop:mean_bound} to the kernels $\Sigma'$ and $\Sigma''$.
Similarly, Propositions \ref{prop:variance_min} and \ref{prop:variance_max} can be generalised by considering two sets of slack variables, one associated to $\varphi'$ and one to $\varphi''$, and relying directly on the lower- and upper-bounding functions defined in Proposition \ref{prop:lbf_ubf_kernel_sum}. 
    
%
\subsection{Kernel Multiplication}\label{eq:kernel_mult}
When two kernels are combined through multiplication, we have that $\Sigma_{x',x''} = \Sigma'_{x',x''} \Sigma''_{x',x''}  $.
This case can be reduced to the addition by considering the following McCormick's inequalities \citep{mccormick1976computability}: 
\begin{align}
  \Sigma_{x',x''} &=  \Sigma'_{x',x''} \Sigma''_{x',x''} \geq \Sigma'_L \Sigma''_{x',x''}  +  \Sigma'_{x',x''} \Sigma''_L - \Sigma'_L \Sigma''_L  \label{eq:mclower}\\
  \Sigma_{x',x''} &=  \Sigma'_{x',x''} \Sigma''_{x',x''} \leq \Sigma'_U  \Sigma''_{x',x''} +  \Sigma'_{x',x''}  \Sigma''_L - \Sigma'_U  \Sigma''_L, \label{eq:mcupper}
\end{align}
where $\Sigma'_L$, $\Sigma'_U$, $\Sigma''_L$ and $\Sigma''_U$ are lower and upper bound values for $\Sigma'$ and $\Sigma''$ in $T$, respectively.
Then we can proceed by using the kernel summation of Equation \eqref{eq:mclower} when computing lower bounding function on the kernel, and  Equation \eqref{eq:mcupper} when computing the upper bounding function, and by using the techniques discussed in the section just above.
\subsection{Generalised Spectral Kernel}
We show how to find kernel decompositions compatible with our optimisation framework for generalised spectral kernels \citep{samo2015generalized}.
We note that these are dense in the space of kernel functions, so that they can be used to derive any kernel up to an arbitrary small error tolerance. 
\paragraph{Stationary Kernel}
For the stationary case, we have:
\begin{align}\label{eq:spectral_kernel}
    \Sigma_{x',x''} = \sum_{k=1}^K \sigma^2 h( (x'-x'') \odot \theta^{(k)} ) \cos( w_k^T (x'-x'')),       
\end{align}
where $\theta^{(k)} \geq 0$, and $h$ is any given positive definite function; in particular, we choose $h( (x'-x'') \odot \theta^{(k)}) = \exp \left(  - \sum_{j=1}^m \theta^{(k)}_j (x' -x'')^2 \right) $.
We now show how a bounded kernel decomposition $(\varphi,\psi,U)$ can be derived for this kernel.

We first observe that the kernel is obtained by summing over $K$ different kernel components.
According to the results for kernel addition described in Appendix \ref{app:kernel_addition}, it suffices to find a bounded kernel decomposition for each summand of Equation \eqref{eq:spectral_kernel}, i.e., for:
\begin{align*}
   k(x',x'') =  \sigma^2 \exp \left(  - \sum_{j=1}^m \theta_j (x' -x'')^2 \right) \cos(w^T (x'-x'')).
\end{align*}
In turn, by setting $k_1(x',x'') =  \sigma^2 \exp \left(  - \sum_{j=1}^d \theta_j (x' -x'')^2 \right) $ and $k_2(x',x'') = \cos(w^T (x'-x''))$, we have that $k(x',x'') = k_1(x',x'') k_2(x',x'')$, and thus a kernel decomposition can be found by using the formulas for kernel multiplication derived in Appendix \ref{eq:kernel_mult} to  $k_1(x',x'')$ and $k_2(x',x'')$.

Observe that $k_1(x',x'')$ has the same shape as the squared-exponential kernel, for which kernel decomposition was derived in Appendix \ref{sec:sqe_decomposition}.
For $k_2(x',x'')$, we set
\begin{align}
    \varphi(x',x'') = \sum_{j=1}^d w_j (x'_j-x''_j) \label{eq:varphi_for_cos} \\
    \psi(\varphi) = \cos(\varphi).
\end{align}
We note that Assumptions 1 and 2 of Definition \ref{def:kernel_decomposition} are satisfied by this decomposition.
For the definition of a bounding function $U$, i.e., Assumption 3, we have:
\begin{align*}
     \sup_{x \in T} \sum_{i=1}^{\tss} c_i \varphi(x,\x{i}) &= \sup_{x \in T} \sum_{i=1}^{\tss} c_i  \sum_{j=1}^d w_j (x_j-x^{(i)}_j) = \sup_{x \in T} \sum_{j=1}^d \sum_{i=1}^{\tss} c_i w_j (x_j-x^{(i)}_j) = \\
     & \sup_{x \in T}   \sum_{j=1}^d \left( \sum_{i=1}^{\tss} c_i \right) w_j x_j - \sum_{j=1}^d \sum_{i=1}^{\tss} w_j x^{(i)}_j =  \sum_{j=1}^d  \sup_{x \in T} \bar{w}_j x_j - \beta,
\end{align*}
where $\bar{w}_j  = \left( \sum_{i=1}^{\tss} c_i \right) w_j$ and $\beta = \sum_{j=1}^d \sum_{i=1}^{\tss} w_j x^{(i)}_j $.
As the above is a linear form, we have that the supremum of $\bar{w}_j x_j$ occurs in  the point $x_j^*= x_j^U$ if $\bar{w}_j \geq 0$ and in $x_j^*= x_j^L$ otherwise.
Thus, we have that $U_{k_2}(\mathbf{c}) = \sum_{j=1}^d \bar{w}_j x_j^* - \beta$ is a valid upper-bound function for the sub-kernel $k_2$. 

\paragraph{Non-Stationary Kernel}
In the non-stationary case, we have:
\begin{align*}
    \Sigma_{x',x''} = \sum_{k=1}^K \sigma^2_k \bar{k} \left( x' \odot \theta^{(k)} , x'' \odot \theta^{(k)}  \right) \Psi_k(x')^T \Psi_k(x''),
\end{align*}
where $\theta^{(k)} \geq 0 $, $\bar{k}$ is a positive semi-definite, continuous and integrable function and $$\Psi_k(x) = \left[\cos \left( x^T w^{(k)}_1  \right) + \cos \left( x^T w^{(k)}_2  \right) , \sin \left( x^T w^{(k)}_1  \right) + \sin \left( x^T w^{(k)}_2  \right) \right].$$
In particular we choose:
\begin{align*}
    \bar{k} \left( x'\odot \theta , x''\odot \theta\right) = k(x'\odot \theta)k(x''\odot \theta) =  e^{x' \odot \theta }e^{x''\odot \theta} = e^{\sum_j \theta_j( x'_j + x''_j)}.
\end{align*}
Proceeding similarly to the case of stationary kernels, we can analyse each summand and factor in isolation.
The final decomposition can then be obtained by using the addition and multiplication formulas for kernel decompositions derived in Appendix \ref{app:kernel_addition} and \ref{eq:kernel_mult}. 

For $\bar{k} \left( x' , x''\right)$, we select:
\begin{align*}
    \varphi(x',x'') &= \sum_j \theta_j( x'_j + x''_j) \\
    \psi(\varphi) &= e^\varphi.
\end{align*}
Since $\varphi$ has exactly the same shape as for that in Equation \eqref{eq:varphi_for_cos}, a similar argument can be used for finding the upper-bound function.

It remains only to find a decomposition for $\Psi_k(x')^T \Psi_k(x'')$. 
In particular, we have:
\begin{align*}
    \Psi(x')^T \Psi(x'') &= \left[    \cos ( x'^T w_1  )  + \cos ( x'^T w_2  )    \right] \left[    \cos ( x''^T w_1  )  + \cos ( x''^T w_2  )    \right]  \\
    &+  \left[    \sin ( x'^T w_1  )  + \sin ( x'^T w_2  )    \right] \left[    \sin ( x''^T w_1  )  + \sin ( x''^T w_2  )    \right] \\     
    & = \cos ( x'^T w_1  ) \cos ( x''^T w_1  ) + \cos ( x'^T w_1  ) \cos ( x''^T w_2  ) \\
    & + \cos ( x'^T w_2  ) \cos ( x''^T w_1  ) + \cos ( x'^T w_2  ) \cos ( x''^T w_2  ) \\
     & + \sin ( x'^T w_1  ) \sin ( x''^T w_1  ) + \sin ( x'^T w_1  ) \sin ( x''^T w_2  ) \\
    & + \sin ( x'^T w_2  ) \sin ( x''^T w_1  ) + \sin ( x'^T w_2  ) \sin ( x''^T w_2  ) 
\end{align*}
Again, we can focus on the single factor from the equation above, and rely on the addition and multiplication formulas to obtain the overall result.
We consider the first factor, $\cos ( x'^T w_1  ) \cos ( x''^T w_1  )$, and select:
\begin{align*}
    \varphi(x',x'') = \cos ( x'^T w_1  ) \cos ( x''^T w_1  ) \\
    \psi(\varphi) = \varphi.
\end{align*}
For the computation of the upper-bound function, we have the following:
\begin{align*}
    \sup_{x} \sum_i c_i \cos ( x^{(i),T} w_1  ) \cos ( x^T w_1  ) \leq \sum_i \sup_{x} c_i  \cos( x^{(i),T} w_1  ) \cos( x^T w_1  ) = \sum_i \sup_{x} \gamma_i \cos( x^T w_1  ),
\end{align*}
where we define $\gamma_i = c_i  \cos( x^{(i),T} w_1  )$. It is thus straightforward to find the maximum of the right-hand-side equation by inspecting the derivatives of the cosine function.